\documentclass[preprint]{imsart}

\usepackage{graphicx}
\usepackage{color}
\usepackage{amscd}
\usepackage{framed}
\usepackage{bbm}
\usepackage{amsmath,amsthm,amssymb}
\usepackage{fancyhdr,a4wide}

\usepackage{comment}
\newcommand{\mc}{\mathcal}
\newcommand{\mbb}{\mathbb}
\newcommand{\mr}{\mathrm}
\newcommand{\argmin}{\mathop{\rm argmin}\limits}

\newcommand{\vecb}{\mathbf{b}}

\newcommand{\vecg}{\mathbf{g}}

\newcommand{\vecu}{\mathbf{u}}
\newcommand{\vecv}{\mathbf{v}}
\newcommand{\vecw}{\mathbf{w}}
\newcommand{\vecx}{\mathbf{x}}

\newcommand{\vecz}{\mathbf{z}}

\newcommand{\vecbeta}{\boldsymbol \beta}
\newcommand{\vectheta}{\boldsymbol \theta}

\newcommand{\vecmu}{\boldsymbol \mu}
\usepackage{bm}

\usepackage{algorithm}
\usepackage{algorithmic}

\numberwithin{equation}{section}
\newtheorem{theorem}{Theorem}[section]
\newtheorem{proposition}{Proposition}[section]
\newtheorem{remark}{Remark}[section]
\newtheorem{lemma}{Lemma}[section]
\newtheorem{corollary}{Corollary}[section]
\newtheorem{definition}{Definition}[section]
\newtheorem{assumption}{Assumption}[section]

\begin{document}
\begin{frontmatter}
	\title{Adversarial Robust Low Rank Matrix Estimation: Compressed Sensing and Matrix Completion}
	\runtitle{Adversarial robustness and sparsity}
		
	\begin{aug}			
		\author{\fnms{Takeyuki} \snm{Sasai}\ead[label=e1]{sasai@ism.ac.jp}}
		
		\address{Department of Statistical Science, The Graduate University for Advanced Studies, SOKENDAI, Tokyo, Japan. 
			\printead{e1}}
		
		\author{\fnms{Hironori} \snm{Fujisawa}
			\ead[label=e3]{fujisawa@ism.ac.jp}
		}	
		\address{The Institute of Statistical Mathematics, Tokyo, Japan. \\
Department of Statistical Science, The Graduate University for Advanced Studies, SOKENDAI, Tokyo, Japan. \\
Center for Advanced Integrated Intelligence Research, RIKEN, Tokyo, Japan. \\
			\printead{e3}
		}
		
		\thankstext{t1}{This work was supported in part by JSPS KAKENHI Grant Number 17K00065.}
		\runauthor{Sasai and Fujisawa}
	\end{aug}

	\begin{abstract}
		We consider robust low rank matrix estimation as a trace regression
		when outputs are contaminated by adversaries. The adversaries are
		allowed to add arbitrary values to arbitrary outputs. Such values can
		depend on any samples. We deal with matrix compressed sensing,
		including lasso as a partial problem, and matrix completion, and then
		we obtain sharp estimation error bounds. To obtain the
		error bounds for different models such as matrix compressed sensing
		and matrix completion, we propose a simple unified approach based on a
		combination of the Huber loss function and the nuclear norm
		penalization. Some error bounds obtained in the present paper are sharper than the past ones.
	\end{abstract}
	
	\begin{keyword}[class=MSC]
		\kwd{62G35}
		\kwd{62G05}
	\end{keyword}
	
	\begin{keyword}
		\kwd{Linear regression}
		\kwd{Robustness}
		\kwd{Convergence rate}
		\kwd{Huber loss}
	\end{keyword}
	\tableofcontents
\end{frontmatter}

\section{Introduction}
\label{intro}
Sparse estimation is a well-studied topic in high-dimensional statistics. 
For sparse estimation of the linear regression coefficient, the $\ell_1$ penalization and its variants have been introduced by \cite{Tib1996Regression}, \cite{FanLi2001Variable}, \cite{ZouHas2005Regularization}, \cite{YuaLin2006Model}, \cite{Zha2010Nearly}, \cite{SuCan2016Slope}, and \cite{BelLecTsy2018Slope} .
In some studies, such as \cite{KolLouTsy2011Nuclear}, \cite{NegWai2011Estimation}, \cite{RohTsy2011Estimation}, \cite{NegWai2012Restricted}, \cite{CaiZha2013Sparse}, \cite{Klo2014Noisy}, \cite{KloLouTsy2017Robust}, and \cite{FanWanZhu2021Shrinkage}, sparse estimation of the linear regression coefficient was extended to low rank matrix estimation, mostly utilizing the nuclear norm penalization.

In this study, we consider sparse estimations of regression coefficient under the existence of malicious outlier. \cite{NguTra2012Robust} considered the case that a part of outputs is adversarially contaminated. They dealt with the following model: 
\begin{align}
\label{intro:model}
y_i = \langle \vecx_i, \vecbeta^*\rangle + \xi_i + \sqrt{n}\theta^*_i,\quad i=1,\cdots,n,
\end{align}
where $\vecbeta^* \in \mbb{R}^d$ is the true coefficient vector, $\{\vecx_i\}_{i=1}^n$ is a sequence of covariate vectors, $\langle \cdot,\cdot \rangle$ is the inner product and  $\{\xi_i\}_{i=1}^n$ is a sequence of random noises.
In addition, $\vectheta^*  = \left(\theta_1^*,\cdots,\theta_n^*\right)^\top$ is a vector of adversarial noises ($\sqrt{n}$ is used for normalization). An adversary is allowed to set any value to any position in $\theta^*$.
Here we give a more explanation on the role of $\vectheta^*$. 
Let $I_I$ and $I_O$ be the index sets for uncontaminated and contaminated outputs, respectively, in other words, we have $\theta^*_i =0$ for $i \in I_I$ and $\theta^*_i \neq 0$ for $i \in I_O$, respectively.  Let $o$ be the number of elements of $I_O$. We allow the adversary can choose $I_O$ arbitrarily on the knowledge of $\{\vecx_i\}_{i=1}^n$ and $\{\xi_i\}_{i=1}^n$ only with the constraint that $o/n$, which is a ratio of the contaminated samples by the adversary, is sufficiently small.
We should note that inliers can lose their independence and outliers can be correlated to inliers because the values of $\theta^*_i$ for $i \in I_O$ are not constrained and $I_O$ can be chosen freely.

For a vector $\vecv$, let $\|\vecv\|_2$ and $\|\vecv\|_1$ be the $\ell_2$ and $\ell_1$ norms, respectively. 
\cite{NguTra2012Robust} introduced the following estimator:
\begin{align}
	\label{model:obj2-intro}
	&(\hat{\vecbeta},\hat{\vectheta}) \in 
	\argmin_{(\vecbeta,\vectheta) \in (\mbb{R}^d,\,\mbb{R}^n)} \mr{obj} (\vecbeta,\vectheta),\\
	&\quad \mr{obj} (\vecbeta,\vectheta)=  \sum_{i=1}^n \left(y_i- \langle  \vecx_i, \vecbeta\rangle -\sqrt{n} \theta_i\right)^2 + \lambda_*\|\vecbeta\|_1+ \lambda_o\|\vectheta\|_1.\nonumber 
\end{align}
Then, \cite{NguTra2012Robust} got a high-probability error bound for $\|\hat{\vecbeta}-\vecbeta^*\|_2$, which is
\begin{align}
	\label{NT}
	\mbb{P}\left\{\|\hat{\vecbeta}-\vecbeta^*\|_2 \leq C_{\vecx, \xi,\delta}\left(\sqrt{\frac{s \log d}{n}}+\sqrt{\frac{o}{n}\log n} \right)\right\} \geq 1-\delta,
\end{align}
where $C_{\vecx, \xi,\delta}$ is some constant depending on $\delta$ and properties of $\{\vecx_i\}_{i=1}^n$ and $\{\xi_i\}_{i=1}^n$ when $\{\vecx_i\}_{i=1}^n$ and $\{\xi_i\}_{i=1}^n$ are drawn from Gaussian distribution.
It is known that without adversarial noises, by using the $\ell_1$ penalization, we can get an error bound such as $\sqrt{\frac{s \log d}{n}}$ up to constant factor \cite{RasWaiYu2010Restricted}.
We note that even when adversarial noises contaminate outputs, the error bound is only loosened by $\sqrt{\frac{o}{n}\log n}$ up to constant factor by using estimator \eqref{model:obj2-intro}.
Model \eqref{intro:model} and estimator \eqref{model:obj2-intro} were also studied in \cite{DalTho2019Outlier}. \cite{DalTho2019Outlier} introduced new concentration inequalities and derived a sharper error bound than \eqref{NT}, which is 
\begin{align}
	\label{DT}
	\mbb{P}\left\{\|\hat{\vecbeta}-\vecbeta^*\|_2 \leq C'_{\vecx, \xi,\delta}\left(\sqrt{\frac{s \log d}{n}}+\frac{o}{n} \sqrt{\log n} \sqrt{\log \frac{n}{o}} \right)\right\} \geq 1-\delta,
\end{align}
where $C_{\vecx, \xi,\delta}'$  is some constant depending on $\delta$ and  properties of $\{\vecx_i\}_{i=1}^n$ and $\{\xi_i\}_{i=1}^n$.

Recently, \cite{Tho2020Outlier} proposed a new estimator, which is a variant of \eqref{model:obj2-intro}.
The estimator is a combination of \eqref{model:obj2-intro} and SLOPE \cite{SuCan2016Slope,BodVanSabSuCan2015Slope,BelLecTsy2018Slope}:
\begin{align}
	\label{TsubEst}
	&(\hat{\vecbeta},\hat{\vectheta}) \in 
	\argmin_{(\vecbeta,\vectheta) \in (\mbb{R}^d,\,\mbb{R}^n)} \mr{obj} (\vecbeta,\vectheta),\\
	&\quad \mr{obj} (\vecbeta,\vectheta)=  \sum_{i=1}^n \left(y_i- \langle  \vecx_i, \vecbeta\rangle -\sqrt{n} \theta_i\right)^2 + \|\vecbeta\|_{\flat}+ \|\vectheta\|_{\sharp},\nonumber 
\end{align}
where $\|\cdot\|_{\flat}$ and $\|\cdot\|_{\sharp}$ are the SLOPE norms.
The error bound of \cite{Tho2020Outlier} is 
\begin{align}
	\label{Tsub}
	\mbb{P}\left\{\|\hat{\vecbeta}-\vecbeta^*\|_2 \leq C_{\vecx, \xi}\left(\sqrt{\frac{s \log (d/s)}{n}}+\frac{o}{n}\log \frac{n}{o} +\frac{1+\sqrt{\log (1/\delta)}}{\sqrt{n}}\right)\right\} \geq 1-\delta,
\end{align}
where $C_{\vecx, \xi}$  is some constant depending on properties of $\{\vecx_i\}_{i=1}^n$ and $\{\xi_i\}_{i=1}^n$,  \cite{Tho2020Outlier} also weaken the assumption to make it applicable to the case that  $\{\vecx_i\}_{i=1}^n$ is drawn from an $L$-subGaussian distribution and $\{\xi_i\}_{i=1}^n$ is drawn from a subGaussian distribution.
We note that the error bound in \eqref{Tsub} is shaper than in \eqref{DT} and the constant of error bound in \eqref{Tsub} does not depend on $\delta$.

On the other hand, after optimizing about $\vectheta$, from \cite{SheChe2017Robust}, \eqref{model:obj2-intro} can be re-written as
\begin{align}
\label{obj2-intro-h}
&\hat{\vecbeta} \in {\argmin}_{\vecbeta} \mr{obj}_H(\vecbeta), \\
&\mr{obj}_H(\vecbeta)=\lambda_o^2\sum_{i=1}^n H\left(\frac{y_i-\langle \vecx_i,\vecbeta\rangle }{\lambda_o\sqrt{n}}\right)+\lambda_*\|\vecbeta\|_1,\nonumber
\end{align}
where $H(t)$ is the Huber loss function
\begin{align}
H(t) = \begin{cases}
|t| -1/2 & (|t| > 1) \\
t^2/2  & (|t| \leq 1)
\end{cases}.
\end{align}
In the present paper, we analysis  \eqref{obj2-intro-h} rather than \eqref{model:obj2-intro} and 
derive a sharper error bound 
\begin{align}
	\label{Our}
	\mbb{P}\left\{\|\hat{\vecbeta}-\vecbeta^*\|_2 \leq C'_{\vecx, \xi}\left(\sqrt{\frac{s \log (d/s)}{n}}+\frac{o}{n} \sqrt{\log \frac{n}{o}} +\frac{1+\sqrt{\log (1/\delta)}}{\sqrt{n}}\right)\right\} \geq 1-\delta,
\end{align}
where $C_{\vecx, \xi}'$  is some constant depending on properties of $\{\vecx_i\}_{i=1}^n$ and $\{\xi_i\}_{i=1}^n$. We also weaken the assumptions on $\{\xi_i\}_{i=1}^n$ from a subGaussian distribution to a heavy-tailed distribution with a finite absolute moment.
We note that the error bound in \eqref{Our} is sharper than that of \eqref{Tsub}, however, our 
estimator requires rough knowledge about the sparsity $s$ and the number of contaminated outputs $o$ in the tuning parameters, as mentioned in Theorem \ref{t:lasso:main}.

Sparse estimation can be considered not only on sparse vectors, but also on low rank matrices.
In the present paper, we consider the following model 
\begin{align}
	\label{intro:model:M}
	y_i = \langle X_i, B^*\rangle + \xi_i + \sqrt{n}\theta^*_i,\quad i=1,\cdots,n,
\end{align}
where $B^* \in \mbb{R}^{d_1 \times d_2}$ is the true (low rank) coefficient matrix, $\{X_i\}_{i=1}^n$ is a sequence of covariate matrices.
For a matrix $M$, let $\|M\|_*$ be the nuclear norm. 
To estimate $B^*$, \cite{Tho2020Outlier} considered the following more general estimator than \eqref{TsubEst}:
\begin{align}
	\label{model:obj2-intro-m}
	&(\hat{B},\hat{\vectheta}) \in 
	\argmin_{(B,\vectheta) \in (\mbb{R}^d,\,\mbb{R}^n)} \mr{obj} (B,\vectheta),\\
	&\quad \mr{obj} (B,\vectheta)=  \sum_{i=1}^n \left(y_i- \langle  X_i, B\rangle -\sqrt{n} \theta_i\right)^2 + \lambda_*\|B\|_*+ \|\vectheta\|_{\sharp}.\nonumber
\end{align}
\cite{Tho2020Outlier} dealt with robust matrix compressed sensing, robust matrix completion, trace regression with matrix decomposition.
On the other hand, to estimate the low rank matrix, we consider the following model:
\begin{align}
	\label{obj2-intro-m}
	&\hat{B} \in {\argmin}_{B} \mr{obj}_H(B), \\
	&\mr{obj}_H(\vecbeta)=\lambda_o^2\sum_{i=1}^n H\left(\frac{y_i-\langle X_i,B\rangle }{\lambda_o\sqrt{n}}\right)+\lambda_*\|B\|_*,\nonumber
\end{align}
which is an extension of \eqref{obj2-intro-h}.
In the present paper, we also derive a sharper error bound than \cite{Tho2020Outlier} about robust matrix compressed sensing and robust matrix completion under a weaker condition.

The remainder of this paper is organized as follows.
In Section \ref{sec:res}, we exhibit our results. 
In Section \ref{sec:re}, we explain related works. 
In Section \ref{sec:key}, we digest our key propositions and lemmas. 
All of the proofs are postponed to the Appendix

\section{Results}
\label{sec:res}
\subsection{Adversarial matrix compressed sensing}
\label{subsec:amcs}
Before presenting our results, we introduce the matrix restricted eigenvalue (MRE) condition for the covariance matrix of random matrix, which is defined in Definition \ref{def:MRE}. The MRE condition is an extension of restricted eigenvalue condition for the covariance matrix of random vector introduced in \cite{DalTho2019Outlier}.
This condition enables us to deal with the case where the covariance matrix of is singular.

Before defining the MRE condition, we prepare some notations. 
Let $\mr{Proj}_V$ be the orthogonal projection into a linear vector subspace 
$V$ of Euclidean space and  $V^{\bot}$ be the orthogonal complement space of $V$.
For a matrix $E$, let $l_i(E)$ and $r_i(E)$ be the left and right orthonormal singular vectors 
of $E$, respectively. Let $V_l(E)$ and $V_r(E)$ be  the linear spans of $\left\{ l_i(E)\right\}$ and  $\left\{ r_i(E)\right\}$, respectively. For a matrix $M \in \mbb{R}^{d_1 \times d_2}$, we define
\begin{align}
	\mr{P}^\bot_{E}(M) = \mr{Proj}_{V_l^\bot(E)}M\mr{Proj}_{V_r^\bot(E)}
\end{align}
and
\begin{align}
	 \mr{P}_{E}(M) = M- \mr{P}^\bot_{E}(M).
\end{align}
For a matrix $M$, let $\|M\|_{\mr{F}}$ be the Frobenius norm and $\mr{T}_{\Sigma} \in \mbb{R}^{d_1 \times d_2}$ is an operator such that $\mr{T}_\Sigma(M) = \mr{vec}^{-1} (\Sigma^\frac{1}{2}\mr{vec}(M))$, where $\mr{vec}(\cdot):\,\mbb{R}^{d_1\times d_2} \to \mbb{R}^{d_1d_2}$ is the operator that unfolds a matrix to a vector.
\begin{definition}[MRE Condition]
	\label{def:MRE}
	The matrix $\Sigma$ is said to satisfy the matrix restricted eigenvalue condition $\mr{MRE}(r,c_0,\kappa)$ with a positive integer $r$ and positive values $c_0$ and $\kappa$, if
	\begin{align}
		\label{con:re-v-pre}
		\kappa \|\mr{P}_{E}(M)\|_\mr{F} \leq \|\mr{T}_\Sigma( M)\|_\mr{F}
	\end{align}
	for any matrix $M\in \mbb{R}^{d_1 \times d_2}$ and any matrix $E \in \mbb{R}^{d_1 \times d_2}$ with $\mr{rank}(E) \leq r$ such that
	\begin{align}
		\label{con:re-v}
		\|\mr{P}_{E}^\bot (M)\|_*\leq c_0\|\mr{P}_{E} (M)\|_*.
	\end{align}
\end{definition}

Next, we introduce an $L$-subGaussian random vector, which appeared in \cite{Men2016Upper,MenZhi2020Robust,Tho2020Outlier} and others.
\begin{definition}[$L$-subGaussian random vector]
	A random vector $\vecx \in \mbb{R}^d$ with the mean $\mbb{E}\vecx = \vecmu$ is said to be an $L$-subGaussian if for every $\vecv \in \mbb{R}^d$ and every $p\geq 2$,
	\begin{align}
		\|\langle \vecx-\vecmu,\vecv\rangle\|_{\psi_2}\leq L\left(\mbb{E}|\langle \vecx-\vecmu,\vecv\rangle|^2\right)^\frac{1}{2},
	\end{align}
	where the norm $\|\cdot\|_{\psi_2}$ is defined in Definition \ref{d:orlicz}.
\end{definition}
\begin{remark}
	See Remark 1.3 of \cite{MenZhi2020Robust} for details on the difference between a subGaussian random variable and an L-subGaussian random variable. The L-subGaussian property enables us to use the Generic Chaining  \cite{Tal2014Upper}.
\end{remark}
\begin{definition}[$\psi_2$-norm]
	\label{d:orlicz}
	Let $f$ be defined on the some probability space. Set
	\begin{align}
		\|f\|_{\psi_2}:=	\inf\left[ \eta>0\,:\, \mbb{E}\exp(f/\eta)^2\leq 2\right] < \infty.
	\end{align}	
\end{definition}

For adversarial matrix compressed sensing, we use the following assumption.
\begin{assumption}
	\label{a:mcs}
	Assume that 
	\begin{itemize}
		\item[(i)] $\{\mr{vec}(Z_i)\}_{i=1}^n$  is a sequence with independent random vectors drawn from $L$-subGaussian distributions, where $\Sigma^\frac{1}{2} \mr{vec}(Z_i)=\mr{vec}(X_i)$ and $\Sigma:=\mbb{E} \mr{vec}(X_i) \mr{vec}(X_i)^\top$ and $\Sigma$ satisfies $MRE(r,c_0,\kappa)$. 
		\item[(ii)] $\{\xi_i\}_{i=1}^n$ is a sequence with independent random variables from a distribution whose absolute moments is bounded by $\sigma$,
		\item[(iii)] for $i=1,\cdots,n$, $\mbb{E}h\left(\frac{\xi_i}{\lambda_o\sqrt{n}}\right) \times X_i$ is the zero matrix.
	\end{itemize}
\end{assumption}

Under Assumption \ref{a:mcs}, we have the following theorem.
\begin{theorem}
	\label{t:cs:main}
	Suppose that Assumption \ref{a:mcs} holds. Consider the optimization problem \eqref{obj2-intro-m}. 
	Suppose that  $\lambda_o\sqrt{n} \geq 72 L^4 \sigma$ and $\lambda_* = c_{mcs} \times\lambda_o\sqrt{n}\times L \times r_{\lambda_*}$, where
	\begin{align}
r_{\lambda_*}= \rho\sqrt{\frac{d_1+d_2}{n}}+\frac{1}{c_\kappa \sqrt{r}}\frac{1+\sqrt{\log (1/\delta)}}{\sqrt{n}}+\frac{1}{c_\kappa \sqrt{r}}\frac{o}{n}\sqrt{\log \frac{n}{o}}
	\end{align}
	and where $c_{mcs}$ is some numerical constant and $c_\kappa = \frac{c_0+1}{\kappa}$ and $\rho^2$ is the maximum diagnonal element of $\Sigma$. Let 
	\begin{align}
		\label{ine:ebmcs}
		r_{mcs}  =c_{mcs}' \times\lambda_o\sqrt{n}\times L\times c_\kappa \sqrt{r} \times r_{\lambda_*}, 
	\end{align}
	where $c_{mcs}'$ is some numerical constant, and suppose $r_{mcs}\leq \frac{1}{4\sqrt{3} L^2}$.
	Then, the optimal solution $\hat{B}$ satisfies $\|\mr{T}_\Sigma(\hat{B} -B^*)\|_\mr{F} < r_{mcs}$
	with probability at least $1-3\delta$ for $0<\delta<1/7$.
\end{theorem}
\begin{remark}
	When $o=0$, in other words, there are no adversarial noises, the error bound \eqref{ine:ebmcs} gets
	\begin{align}
		O\left(c_\kappa\rho\sqrt{r\frac{d_1+d_2}{n}}+\frac{1+\sqrt{\log (1/\delta)}}{\sqrt{n}}\right).
	\end{align}
	\cite{NegWai2011Estimation} considered no adversarial case with the random noises drawn from a Gaussian distribution and obtained the error bound 
	\begin{align}
		C_{\delta}\left(c_\kappa\rho\sqrt{r\frac{d_1+d_2}{n}}\right),
	\end{align}
	where $C_{\delta}$ is some constant which depends on $\delta$.
	In our error bound, the term involving $\delta$ and the term involving $r$ is  separated and the assumption of random noise is weakened from a Gaussian distribution to a distribution with a finite absolute moment. 
\end{remark}

\begin{remark}
	As we mentioned in Introduction, \cite{Tho2020Outlier} also considered the estimation of $B^*$ from model \eqref{intro:model:M}. The main term of the error bound of \cite{Tho2020Outlier} is
	\begin{align}
		\label{ine:ebmcs2}
		  O\left(L\sigma' \rho \sqrt{r\frac{d_1+d_2}{n}} +L\sigma' \frac{1+\sqrt{\log(1/\delta)}}{\sqrt{n}}+L\sigma'^{2} \frac{o}{n}\log \frac{n}{o}\right),
	\end{align}
	where $\sigma'^2$ is the second moment of $\{\xi_i\}_{i=1}^n$.
	Our error bound is sharper from the perspective of convergence rate and we weaken the assumption on $\{\xi_i\}_{i=1}^n$ because \cite{Tho2020Outlier} requires that $\{\xi_i\}_{i=1}^n$ is drawn from a subGaussian distribution.
	On the other hand, our method requires the information $o$ for the tuning parameter $\lambda_*$, however, the estimator of  \cite{Tho2020Outlier} does not. This factor is important for the practical use.
	In addition, the dependence of constants of the result of \cite{Tho2020Outlier} is  better than ours. For example, when $\lambda_o$ is set to $72L^4 \sigma$, our error bounds depend on $L^5$, however, the one of \cite{Tho2020Outlier} depends on $L$. This point would also be important for the practical use.
\end{remark}

\subsection{Adversarial estimation of sparse linear regression coefficient (adversarial lasso)}
\label{subsec:aelrc}

Before presenting our results, we introduce the restricted eigenvalue condition for the covariance matrix of random vector (the RE condition), which was introduced by \cite{DalTho2019Outlier}.
As we refer to in section \ref{subsec:amcs} and the MRE condition is an extension of the RE condition, the MRE condition contains RE condition.
For completeness, we introduce the RE condition here.

\begin{definition}[RE Condition]
	\label{def:RE}
	The matrix $\Sigma$ is said to satisfy the restricted eigenvalue condition $\mr{RE}(s,c_0,\kappa)$ with a positive integer $s$ and positive values $c_0$ and $\kappa$, if
	\begin{align}
		\label{con:re-v-pre-vec}
		\kappa \|\vecv_J\|_2 \leq \|\Sigma^\frac{1}{2}\vecv\|_2
	\end{align}
	for any vector $\vecv \in \mbb{R}^{d}$ and for any set $J \subset \{1,\cdots,d\}$ with $\mr{Card}(J) \leq s$ such that
	\begin{align}
		\label{con:re-v-vec}
		\|\vecv_{J^c}\|_1\leq c_0\|\vecv_J\|_1.
	\end{align}
\end{definition}

For adversarial lasso, we use the following assumption.
\begin{assumption}
	\label{a:lasso}
	Assume that 
	\begin{itemize}
		\item[(i)] $\{\vecz_i\}_{i=1}^n$  is a sequence with independent random vectors drawn from an $L$-subGaussian distribution, where $\Sigma^\frac{1}{2} \vecz_i = \vecx_i$ and $\Sigma =\mbb{E} \vecx_i \vecx_i^\top$ and $\Sigma$ satisfies $\mr{RE}(r,c_0,\kappa)$.
		\item[(ii)] $\{\xi_i\}_{i=1}^n$ is a sequence with independent random variables from a distribution whose  absolute moment is bounded by $\sigma$,
		\item[(iii)] for $i=1,\cdots,n$, $\mbb{E}h\left(\frac{\xi_i}{\lambda_o\sqrt{n}}\right) \times \vecx_i$ is the zero vector.
	\end{itemize}
\end{assumption}

For a vector $\vecv$, let $\|\vecv\|_0$ be  the number of non-zero entries of $\vecv$.
Under Assumption \ref{a:lasso}, we have the following theorem.
\begin{theorem}
	\label{t:lasso:main}
	Suppose that Assumption \ref{a:lasso} holds. 
	Consider the optimization problem \eqref{obj2-intro-h}. 
	Assume that $\lambda_o\sqrt{n} \geq 72 L^4 \sigma$ and  $\lambda_* = c_{lasso} \times\lambda_o\sqrt{n}\times L\times r_{\lambda_*}$, where 
	\begin{align}
	r_{\lambda_*}= \rho\sqrt{\frac{\log (d/s)}{n}}+\frac{1}{c_\kappa \sqrt{s}}\frac{1+\sqrt{\log (1/\delta)}}{\sqrt{n}}+\frac{1}{c_\kappa \sqrt{s}}\frac{o}{n}\sqrt{\log \frac{n}{o}},
	\end{align}
	and  $c_{lasso}$ is some numerical constant, $c_\kappa = \frac{c_0+1}{\kappa}$, $s = \|\vecbeta^*\|_0$ and $\rho^2$ is the maximum diagnonal element of $\Sigma$. Let 
	\begin{align}
		\label{ine:eblasso}
		r_{lasso}  =c_{lasso}' \times\lambda_o\sqrt{n}\times L\times c_\kappa \sqrt{r}\times r_{\lambda_*}, 
	\end{align}
	where $c_{lasso}'$ is some numerical constant, and suppose $r_{lasso}\leq \frac{1}{4\sqrt{3} L^2}$.
	Then, the optimal solution $\hat{\vecbeta}$ satisfies $\|\Sigma^\frac{1}{2}(\hat{\vecbeta} -\vecbeta^*)\|_2 < r_{lasso}$
	with probability at least $1-3\delta$ for $0<\delta<1/7$.
\end{theorem}

\begin{remark}
	When $o=0$, in other words, there are no adversarial noises, the error bound \eqref{ine:eblasso} gets
	\begin{align}
		O\left(c_\kappa\rho\sqrt{s\frac{\log (d/s)}{n}}+\frac{1+\sqrt{\log (1/\delta)}}{\sqrt{n}}\right).
	\end{align}
	 \cite{RasWaiYu2010Restricted} considered no adversarial case with the random noises drawn from a Gaussian distribution and obtained the error bound
	\begin{align}
		C_{\delta}\left(c_\kappa\rho\sqrt{s\frac{\log (d/s)}{n}}\right),
	\end{align}
	where $C_{\delta}$ is some constant which depends on $\delta$.
	In our error bound, the term involving $\delta$ is separated from the term involving $s$.
\end{remark}
\begin{remark}
	As we mentioned in Introduction, \cite{Tho2020Outlier} also considered the estimation of $\vecbeta^*$ from model \eqref{intro:model}. The main term of the error bound of \cite{Tho2020Outlier} is
	\begin{align}
		\label{ine:eblasso2}
		  O\left(L\sigma' \rho \sqrt{s\frac{\log (d/s)}{n}} +L\sigma' \frac{1+\sqrt{\log(1/\delta)}}{\sqrt{n}}+L\sigma'^{2} \frac{o}{n}\log \frac{n}{o}\right),
	\end{align}
	where $\sigma'^2$ is the second moment of $\{\xi_i\}_{i=1}^n$.
	As in the case of matrix compressed sensing, our error bound is sharper from the perspective of the convergence rate and we weaken the assumption on $\{\xi_i\}_{i=1}^n$ because \cite{Tho2020Outlier} requires that $\{\xi_i\}_{i=1}^n$ is drawn from a subGaussian distribution.
	On the other hand, our method requires the information $o$ for the tuning parameter $\lambda_*$, however, the estimator of  \cite{Tho2020Outlier} does not. This factor is important for practical use.
	In addition, the dependence of constants of the result of \cite{Tho2020Outlier} is better than ours. For example, when $\lambda_o$ is set to $72L^4 \sigma$, our error bounds depend on $L^5$, however, the one of \cite{Tho2020Outlier} depends on $L$. This point would also be important for practical use.
\end{remark}

\subsection{Adversarial matrix completion}
\label{sc:amc}
Before presenting the observation model, we introduce a sequence of `mask' matrix, $\{E_i\}_{i=1}^n$. Assume that $\{E_i\}_{i=1}^n$ lies in
\begin{align}
\left\{e_k(d_1) e_l(d_2) ^\top \,| \,1\leq k\leq d_1,\ 1\leq l\leq d_2\right\},
\end{align}
where $e_k(d_1)$ is the $d_1$-dimensional $k$-th unit vector and $e_l(d_2)$ is the $d_2$-dimensional $l$-th unit vector, i.e.,
\begin{align}
e_k(d_1) &= (\underbrace{0,\cdots0}_{k-1},1, \underbrace{0,\cdots,0}_{d_1-k})^\top \in \mbb{R}^{d_1},\nonumber \\
e_l(d_2) &= (\underbrace{0,\cdots0}_{l-1},1, \underbrace{0,\cdots,0}_{d_2-l})^\top \in \mbb{R}^{d_2}.
\end{align}
Let $E_{i_{kl}}$ be the $(k,l)$-component of $E_i$. 
Let $d_{mc} = \sqrt{d_1d_2}$. For matrix completion, we consider the following observation model:
\begin{align}
	\label{intro:model:MC}
	y_i = \langle X_i, B^*\rangle + \xi_i + \sqrt{n}\theta^*_i,\quad i=1,\cdots,n,
\end{align}
where $X_i = d_{mc}\varepsilon_i E_i$, where $\varepsilon_i \in \left\{-1,+1\right\}$  is a  random sign and $\left\{E_{i}\right\}_{i=1}^n$ is a sequence with independent random variables such that 
\begin{align}
\mbb{P}[E_{i_{kl}} =1] = 1/d_{mc}^2.
\end{align}

For a matrix $M$, let $\|M\|_\infty$  be the  element-wise $\ell^\infty$-norm.
We introduce the spikiness condition used in \cite{NegWai2012Restricted,Tho2020Outlier, FanWanZhu2021Shrinkage} and so on. 
For a nonzero matrix $M$, let 
\begin{align}
	\alpha (M) := d_{mc}\frac{\|M\|_\infty}{\|M\|_\mr{F}}.
\end{align}
The spikiness condition is given by the following definition.
\begin{definition}
	\label{a:spi}
The spikiness condition is given by $\|B^*\|_\mr{F} \leq 1$, or
\begin{align}
	\|B^*\|_\infty \leq \alpha(B^*)/d_{mc}.
\end{align} 
\end{definition}
\begin{remark}
	Roughly speaking, the spikiness condition requires the true matrix $B^*$ does not have overly large elements. 
\end{remark}
Using the spikiness condition,
we consider the optimal solution given by
\begin{align}
	\label{obj-mc}
	\hat{B} \in 
	\argmin_{\|B\|_\infty   \leq \alpha^*/d_{mc}} \mr{obj}_H (B)
\end{align}
where we abbreviate $\alpha(B^*)$ to $\alpha^*$.

We derive two results about adversarial matrix completion. The assumptions of $\{\xi_i\}_{i=1}^n$ used for each result are different. 
The first assumption is the following.
\begin{assumption}[Assumption for adversarial matrix completion when random noises are drawn from a heavy-tailed distribution]
	\label{a:mc1}
	We assume that 
	\begin{itemize}
	\item[(i)] Let $X_i = d_{mc}\varepsilon_i E_i$, where $\varepsilon_i \in \left\{-1,+1\right\}$  is a  random sign and $\left\{E_{i}\right\}_{i=1}^n$ is a sequence with independent random variables such that 
	\begin{align}
	\mbb{P}[E_{i_{kl}} =1] = 1/d_{mc}^2.
	\end{align}
	\item[(ii)] $\{\xi_i\}_{i=1}^n$ is a sequence with independent random variables from a distribution whose $\alpha$-th absolute moment is bounded by $\sigma^\alpha_{\xi,\alpha}$, with $\alpha\geq 2$.
	\item[(iii)] for $i=1,\cdots,n$, $\mbb{E}h\left(\frac{\xi_i}{\lambda_o\sqrt{n}}\right) \times X_i$ is the zero matrix.
	\end{itemize}
\end{assumption}

Under Assumption \ref{a:mc1}, we have the following Theorem.
\begin{theorem}
	\label{t:mc:main1}
	Suppose that Assumption \ref{a:mc1} holds.
	Consider the optimization problem \eqref{obj-mc}.
	Suppose 
	\begin{align}
		\lambda_o  \sqrt{n} \geq  2\sigma_{\xi,\alpha}\min\left\{\left(\frac{n}{o}\right)^\frac{1}{\alpha+1},\left(\frac{n}{rd_{mc}\log d_{mc}}\right)^\frac{1}{\alpha}\right\}
	\end{align}
	and $\lambda_* = c_{mc1}\times  \frac{1}{ \sqrt{r}}\times r_{\lambda_*}$, where
	\begin{align}
	r_{\lambda_*}= \sigma_\xi \sqrt{\frac{r d_{mc}\{\log d_{mc} + \log (1/\delta)\}}{n}}+\lambda_o \sqrt{r}\frac{d_{mc} (\log d_{mc} + \log (1/\delta))}{\sqrt{n}} +\sqrt{\lambda_o \sqrt{n}\frac{o}{n}},
		\end{align}
		and $c_{mc1}$ is some numerical constant. Let 
		\begin{align}
			\label{mc1-result1}
			&r_{mc1}  =c_{mc1}'\times\alpha^*\times \left(r_{\lambda_*}+\alpha^*\sqrt{\frac{r d_{mc} (\log d_{mc} + \log (1/\delta))}{n}} +\sqrt{r}\frac{d_{mc} \log d_{mc}}{n}+ \alpha^{*} \left(\frac{o}{n}\right)^\frac{\alpha}{2(1+\alpha)}\right), 
		\end{align}
		where $c_{mc1}'$ is a numerical constant. 
		Then, the optimal solution $\hat{B}$ satisfies $\|\hat{B} -B^*\|_\mr{F} < r_{mc1} $
		with probability at least $1-2\delta$.
\end{theorem}

\begin{remark}
	When $o=0$, in other words, there are no adversarial noises, the error bound \eqref{mc1-result1} gets
	\begin{align}
		O\left( \sqrt{\frac{r d_{mc} \{\log d_{mc} + \log (1/\delta)\}}{n}}\right),
	\end{align}
	and this coincides with the bound in \cite{NegWai2012Restricted} from the perspective of convergence rate, which considered the case of no adversarial noise and Gaussian random noises.
	Our result is an extension of \cite{NegWai2012Restricted} because we consider adversarial noises and distribution with a finite $\alpha$-th absolute moment with $\alpha \geq 2$, which includes a Gaussian distribution. 
\end{remark}

\begin{remark}
	We make a detailed consideration of the following term in \eqref{mc1-result1} 
	\begin{align}
		\label{ine:mc:term1}
		\lambda_o \sqrt{r}\frac{d_{mc} (\log d_{mc} + \log (1/\delta))}{\sqrt{n}} ,
	\end{align}
	when $\lambda_o\sqrt{n}$ is equal to the lower bound $2\sigma_{\xi,\alpha}\min\left\{\left(\frac{n}{o}\right)^\frac{1}{\alpha+1},\left(\frac{n}{rd_{mc} \log d_{mc}}\right)^\frac{1}{\alpha}\right\}$.

	For example, when $\alpha=2$, which is the lower bound, we have
	\begin{align}
		2 \sigma_{\xi,2}\sqrt{\frac{rd_{mc} (\log d_{mc} + \log (1/\delta))}{n}}\times \frac{\sqrt{d_{mc} \{\log d_{mc} + \log (1/\delta)\}}}{n^\frac{1}{6} o^\frac{1}{3}}
	\end{align}
	when $\min\left\{\left(\frac{n}{o}\right)^\frac{1}{3},\left(\frac{n}{rd_{mc} \log d_{mc}}\right)^\frac{1}{2}\right\} = \left(\frac{n}{o}\right)^\frac{1}{3}$.
	On the other hand, when $\alpha \to \infty$, \eqref{ine:mc:term1} limits to
	larger
	\begin{align}
		2\sigma_{\xi,\alpha}\sqrt{ \frac{r d_{mc} (\log d_{mc} + \log (1/\delta))}{n}}.
	\end{align}
	Consequently, when $\alpha$ can be assumed to be large, the result would be acceptable in the view of the convergence rate.
\end{remark}
\begin{remark}
	We make a detailed consideration of the  term $ \alpha^* \sqrt{\lambda_o \sqrt{n}\frac{o}{n}}$ in \eqref{mc1-result1}. Assume $\alpha=2$ and  $\lambda_o \sqrt{n} = 2\sigma_{\xi,\alpha}\left(\frac{n}{o}\right)^\frac{1}{\alpha+1}$. We have
	\begin{align}
		\alpha^* \sqrt{2\sigma_{\xi,\alpha}}\left(\frac{o}{n}\right)^\frac{1}{3}.
	\end{align}
	On the other hand, when $\alpha \to \infty$ and we choose $\lambda_o \sqrt{n} =2\sigma_{\xi,\alpha}\left(\frac{n}{o}\right)^\frac{1}{\alpha+1}$, the term $\alpha^* \sqrt{\lambda_o \sqrt{n}\frac{o}{n}}$ approaches to
	\begin{align}
		\alpha^* \sqrt{2\sigma_{\xi,\alpha}}\left(\frac{o}{n}\right)^\frac{1}{2}.
	\end{align}
\end{remark}

Before presenting the second assumption, we introduce a subWeibull distribution \cite{KucCha2018Moving, VlaGirNguArb2020sub}.
\begin{definition}[subWeibull random variable of order $\alpha$]
	Define $\sigma_{x,\psi_\alpha}$ for the random variable $x$ with $\mbb{E}x=0$ as
	\begin{align*}
		\sigma_{x,\psi_\alpha}:=	\inf\left[ \eta>0\,:\, \mbb{E}\psi_\alpha(x)\leq 1\right] < \infty,
	\end{align*}	
	where $\psi_\alpha(x) = \exp \left(\frac{|x|^\alpha}{\eta^\alpha}\right)-1$.
	The random variable $x$ with $\mbb{E}x=0$ is said to be a subWeibull random variable of order $\alpha$ if 
	\begin{align}
		\mbb{P}(|x|\geq t)\leq 2\exp \left(-\frac{t^\alpha}{\sigma_{x,\psi_\alpha}^\alpha}\right).
	\end{align}
\end{definition}
\begin{remark}
	We see that a subWeibull random variable of order $2$ is a subGaussian random variable and a subWeibull random variable of order $1$ is a subexponential random variable. 
\end{remark}
The second assumption is the following.
\begin{assumption}[Assumption for adversarial matrix completion when random noise drawn from a subWeibull distribution]
	\label{a:mc2}
	We assume  (i) and (iii) in Assumption \ref{a:mc1} and
	\begin{itemize}
	\item[(ii)] $\{\xi_i\}_{i=1}^n$ is a sequence with independent random variables from a subWeibull random variables of order $\alpha$ with $\alpha \leq 2$. In addition, we denote the second moment of $\xi_i$ as $\sigma^2_{\xi}$ for $i=1,\cdots,n$.
	\end{itemize}
\end{assumption}

Under Assumption \ref{a:mc2}, we have the following Theorem
\begin{theorem}
	\label{t:mc:main2}
	Suppose that Assumption \ref{a:mc2} holds.
	Consider the optimization problem \eqref{obj-mc}.
	Suppose
	\begin{align}
		\lambda_o  \sqrt{n} \geq 2 \sigma_{\xi ,\psi_\alpha} \min\left\{\log^\frac{1}{\alpha} \frac{n}{o}, \log^\frac{1}{\alpha} \frac{n}{rd_{mc}\log d_{mc}}\right\}
	\end{align} 
	and $\lambda_* = c_{mc2}\times  \frac{1}{ \sqrt{r}}\times r_{\lambda_*}$, where
	 \begin{align}
		r_{\lambda_*}=  \sigma_\xi\sqrt{\frac{r d_{mc} \{\log d_{mc} + \log (1/\delta)\}}{n}}+\lambda_o\frac{\sqrt{r} d_{mc} (\log d_{mc} + \log (1/\delta))}{\sqrt{n}}+\sqrt{\lambda_o\sqrt{n}\frac{o}{n} },
	 \end{align}
	 where $c_{mc2}$ is some numerical constant. 
	 Let
	 \begin{align}
		\label{mc2-result1}
		 &r_{mc2}  =c_{mc2}'\times\alpha^* \times \left(r_{\lambda_*}+ \alpha^* \sqrt{\frac{r d_{mc} (\log d_{mc} + \log (1/\delta))}{n}}+\sqrt{r}\frac{d_{mc}\log d_{mc}}{n}+\alpha^{*}\sqrt{\frac{o}{n}}\right), 
	 \end{align}
	 where $c_{mc2}'$ is some numerical constant.
	 Then, the optimal solution $\hat{B}$ satisfies $\|\hat{B} -B^*\|_\mr{F} < r_0$
	 with probability at least $1-2\delta$.
\end{theorem}
\begin{remark}
	We make detailed consideration of the term condition in  \eqref{mc2-result1}.
	\begin{align}
		\lambda_o\frac{\sqrt{r} d_{mc} (\log d_{mc} + \log (1/\delta))}{\sqrt{n}}
	\end{align}
	when  $\lambda_o \sqrt{n} = 2 \sigma_{\xi ,\psi_\alpha} \min\left\{\log^\frac{1}{\alpha} \frac{n}{o}, \log^\frac{1}{\alpha} \frac{n}{rd_{mc}\log d_{mc}}\right\}$.

	For example, when the random noises drawn from subGaussian distribution ($\alpha=2$), we have
	\begin{align}
		\sqrt{\frac{rd_{mc} (\log d_{mc} + \log (1/\delta))}{n}} \times \sqrt{\frac{d_{mc} (\log d_{mc} + \log (1/\delta))}{n}}  \sqrt{\log \frac{n}{o}}
	\end{align}
	or 
	\begin{align}
		\sqrt{\frac{rd_{mc} (\log d_{mc} + \log (1/\delta))}{n}} \times \sqrt{\frac{(\log d_{mc} + \log (1/\delta))}{n}}  \sqrt{\log \frac{n}{rd_{mc} \log d_{mc}}}
	\end{align}
	and these term become small when $\sqrt{\frac{d_{mc} (\log d_{mc} + \log (1/\delta))}{n}}  \sqrt{\log \frac{n}{o}}\leq 1$ and this condition would be not so strong.
\end{remark}
\begin{remark}
	\label{r:mcsubW}
	We make detailed consideration of the following  term $\alpha^*\sqrt{\lambda_o\sqrt{n}\frac{o}{n}}$ in \eqref{mc2-result1}. Assume  $\alpha=2$ and  $\lambda_o\sqrt{n} = 2 \sigma_{\xi ,\psi_2} \sqrt{\log \frac{n}{o}}$. We have
	\begin{align}
		\label{outs:mc2}
		2\alpha^* \sigma_{\xi,\psi_2} \sqrt{\frac{o}{n} \sqrt{\log \frac{n}{o}}}.
	\end{align}
\end{remark}

\begin{remark}
	\cite{Tho2020Outlier} also considered the estimation of $B^*$ from model \eqref{intro:model:M} under the condition of matrix completion. The main term of error bound of \cite{Tho2020Outlier} is
	\begin{align}
		\label{ine:mc2}
	&	\sqrt{\|\hat{B}-B^*\|_2^2+\|\hat{\vectheta}-\vectheta^*\|_2^2} =  O\left((a^*+\sigma_{\xi,\psi_2}) \sqrt{r\frac{d_{mc}(\log d_{mc}+\log(1/\delta))}{n}} +(a^*+\sigma_\xi) \sqrt{\frac{o}{n}\log \frac{n}{o}}\right).
	\end{align}
	We note that \cite{Tho2020Outlier} considered non-uniform sampling (the mask matrix is slightly different), however, \eqref{ine:mc2} is the result of assuming uniform sampling from (i) in Assumption \ref{a:mc2}.
	Our error bound is sharper from the perspective of  the convergence rate when $\lambda_o$ is chosen to be equal $2 \sigma_{\xi ,\psi_2} \sqrt{\log \frac{n}{o}}$.
	We also weaken the assumption on $\{\xi_i\}_{i=1}^n$ because \cite{Tho2020Outlier} requires that $\{\xi_i\}_{i=1}^n$ is drawn from a subGaussian distribution.
	As in the case of matrix compressed sensing, on the other hand, our method requires the information $o$ for the tuning parameters, however, the estimator of  \cite{Tho2020Outlier} does not. The dependence of constants of the result of \cite{Tho2020Outlier} is different from ours. These points would be important for practical use.
\end{remark}

\section{Related works}
\label{sec:re}

\subsection{Recent development of robust estimation}
One of the recent lines of research on robust estimation in the presence of outliers was triggered by the epoch-making study of \cite{CheGaoRen2018robust}, which considered robust estimation of mean and various types of covariance matrices under Huber's contamination.
Since then, \cite{CheGaoRen2018robust} has been followed by \cite{DiaKanKanLiMoiSte2019Robust}, \cite{LaiRaoVem2016Agnostic} and many other papers.
These papers dealt with  robust mean estimation or covariance matrix estimation \cite{DiaKanKanLiMoiSte2017Being,KotSteSte2018Robust,DiaKamKanLiMoiSte2018Robustly,CheDiaGe2019High,CheYesFlaBar2019Fast,CheDiaGeWoo2019Faster,DonHopLi2019Quantum,LeiLuhVenZha2020Fast,Hop2020Mean,PraBalRav2020Robust,DepLec2019Robust,LugMen2021Robust,DalMin2020All}, robust PCA \cite{BalDuLiSin2017Computationally, DiaKanKarPriSte2019Outlier}, robust regression \cite{DiaKonSte2019Efficient,BakPra2021Robust,LiuSheLiCar2020High,Gao2020Robust} and so on. These studies are mainly interested in deriving sharp error bounds,  deriving learning limits for each problem and reduction of computational complexity.

\subsection{Adversarial matrix compressed sensing and adversarial lasso}
\label{sec:amcsal}
In Section \ref{sec:amcsal}, we refer, topic by topic, the relevant papers that are not referred to in Sections 1 and 2.

\subsubsection{Sparse coefficient  estimation under existence of outliers}
\cite{CheCarMan2013Robust} and \cite{LiuSheLiCar2020High} considered linear regression under contamination on  both $\{y_i\}_{i=1}^n$ and  $\{x_i\}_{i=1}^n$  by adversarial noise.
\cite{CheCarMan2013Robust} provided pioneering research, however, their error bound is not sharp. 
\cite{LiuSheLiCar2020High} derived a relatively sharper error bound compared to \cite{CheCarMan2013Robust}; however, their method requires the information about the $\ell_2$ norm of the length true coefficient vector $\vecbeta^*$, which is not required for out method.
\cite{Gao2020Robust} considered various regression problems with  $\{y_i\}_{i=1}^n$  and  $\left\{\vecx_i\right\}_{i=1}^n$  sampled from normal distributions with Huber's contamination. \cite{Gao2020Robust} derived the learning limit and error bound which coincides to the limit up to a constant factor.
\cite{Gao2020Robust} applied `Tukey's half-space depth' and consumed exponential computational complexity.

\subsubsection{Sparse estimation under heavy-tailed distribution}
\cite{SunZhoFan2020Adaptive} dealt with linear regression problem when the random noises are drawn from a heavy-tailed distribution.
\cite{FanWanZhu2021Shrinkage} considered regression problems including matrix compressed sensing, lasso, matrix completion and reduced-rank regression with heavy-tailed  covariates and random noises, proposing a new `shrinkage' estimator.
\cite{Chi2020Erm} considered linear regression problem with heavy-tailed random noises contaminated by some kinds of outliers.

\subsubsection{SubGaussian estimator}
As we refer to in Section \ref{intro}, $\delta$, probability of success, appears in our error bounds like $\sqrt{\log (1/\delta)/n}$.
Estimators who have error bounds like this are called `subGaussian estimators'. For example, 
\cite{DevLerLugOli2016Sub}, \cite{Min2018Sub}, \cite{LugMen2019Sub}, \cite{Chi2020Erm}, \cite{CheYesFlaBar2019Fast}, \cite{LeiLuhVenZha2020Fast}, \cite{Gao2020Robust}, \cite{LugMen2021Robust}
constructed sub-Gaussian estimators which have some kinds of robustness for outliers or heavy-tailed distributions.

\subsubsection{Dependency between $\{\vecx_i\}_{i=1}^n$ and $\{\xi_i\}_{i=1}^n$}
Our methods and methods of \cite{Tho2020Outlier} do not require the independence of $\{\vecx_i\}_{i=1}^n$ and $\{\xi_i\}_{i=1}^n$. Instead of the independence, our method requires $\mbb{E}h\left(\frac{\xi_i}{\lambda_o\sqrt{n}}\right) \vecx_i = 0$, and methods of \cite{Tho2020Outlier} requires $\mbb{E}\xi_i\vecx_i = 0$. 
Previous works dealt with robust linear regression \cite{DiaKonSte2019Efficient,BakPra2021Robust,LiuSheLiCar2020High,Gao2020Robust}  under the independence of covariates and random noises.

\subsection{Adversarial matrix completion}
Matrix completion was considered in \cite{NegWai2012Restricted} and \cite{KolLouTsy2011Nuclear}. Since then, and numerous works (\cite{RohTsy2011Estimation},  \cite{Klo2014Noisy}, \cite{CaiZhou2016Matrix} and so on) have been conducted.
The adversarial matrix completion was considered by \cite{Tho2020Outlier} with subGaussian noises. The error bound about adversarial noise term of the result of \cite{Tho2020Outlier} is of the order $\sqrt{\frac{o}{n} \log \frac{n}{o}}$, however, our result is slightly sharper as we stated in Remark \ref{r:mcsubW}. However, as in the case of adversarial lasso, the method of \cite{Tho2020Outlier} does not require the information $o$ for the tuning parameter $\lambda_*$, although our method does.

Our method  also works  when the random noises are sampled from a heavy-tailed distribution,  and we find out that our error bound about the error term varies according to the thickness of the error distribution. In the case of adversarial matrix compressed sensing or adversarial lasso, this phenomenon does not be confirmed.
\cite{ElsVan2018Robust, FanWanZhu2021Shrinkage} considered non-adversarial heavy-tailed noises.
Their error bound is $O(\sqrt{rd_{mc} \{\log d_{mc}+\log (1/\delta)\}/n})$ and requires no additional condition such that $\lambda_o\sqrt{d_{mc} \{\log d_{mc} + \log (1/\delta)\}}\leq 1$, which is required in our method. Ensuring robustness against both heavy-tailed noise and adversarial noise with  weaker additional conditions is left as a future study.

On the other hand, although we consider the matrix completion as a partial problem of the trace regression in the present paper, there is another formulation of matrix completion, which is based on the Bernoulli model (\cite{CanRec2009Exact}, \cite{Cha2015Matrix}, \cite{Klo2015matrix}, \cite{CheXuCarSan2015Matrix}, \cite{JaiNet2015Fast}, \cite{CheGupJai2017Nearly}) and so on.
Matrix completion based on the Bernoulli model is out of scopes of the present paper, however, we give a brief introduction.
In matrix completion based on the Bernoulli model, each entry of $B^*+N$, where $N$ is a matrix of random noise, is observed independently of the other entries with probability $n/(d_1d_2)$.
In other words, assume $S_{ij}$ are i.i.d. Bernoulli random variables of parameter $p$ which is independent of $N$, and we can observe $Y$ with
\begin{align}
	Y_{ij} = S_{ij}(B^*_{ij} + N_{ij})\,\quad 1\leq i \leq d_1,\,1\leq j\leq d_2.
\end{align}
As pointed out by \cite{CarKloLofNik2018Adaptive}, the major difference between matrix completion as a trace regression and matrix completion based on the Bernoulli model is that matrix completion as a trace regression admit multiple sampling of each entry, however,  in the case of matrix completion based on the Bernoulli model, each entry is sampled at most once.

\subsection{Other topics}
There are other topics that the present paper does not deal with; e.g.  reduced-rank regression (multi-task regression) \cite{KimXin2012Tree,VelRei2013Multivariate,SheChe2017Robust}.
Both in theory and in applications, it is important to deepen the robustness aspect of these topics. It will be a future subject.

\section{Main theorem, key propositions, lemmas and corollaries}
\label{sec:key}
We introduce the main theorem, key propositions, lemmas and corollaries.
The adversarial lasso is a special case of adversarial matrix compressed sensing because, like the argument of Section 2.2 of \cite{FanWanZhu2021Shrinkage}, linear regression is a special case of trace regression, where $B^*$ and $X_i,\,i=1,\cdots,n$ are diagonal matrices. We note that for diagonal matrix $M \in \mbb{R}^{d \times d}$, we see that  $\|X_i\|_\mr{op} = \|\mr{diag}(X_i)\|_\infty$ and $\|X_i\|_* = \|\mr{diag}(X_i)\|_1$.
Therefore, in Section \ref{sec:key}, we discuss the main theorem, adversarial matrix compressed sensing and adversarial matrix completion.
The value of the numerical constant $C$ shall be allowed to change from line to line. 

\subsection{Main theorem}
We introduce our main theorem in an informal style.
The precise statements is seen in Theorem \ref{t:det:main}.
\begin{theorem}
	\label{maintheorem}
	Consider the optimization problem \eqref{obj2-intro-m}.
	Suppose that $\{\xi_i,X_i\}_{i=1}^n$ and $\lambda_*,\lambda_o$ satisfies the following inequalities:
	\begin{align}
		\label{ine:det:main1}
		&\left| \frac{\lambda_o}{\sqrt{n}}\sum_{i=1}^n h\left(\frac{\xi_i}{\lambda_o \sqrt{n}} \right) \langle  X_i, M \rangle \right| \leq a_{\mr{F}} r_{a,\mr{F}} \left\| \mr{T}_\Sigma (M)\right\|_\mr{F} + a_*r_{a,*}\|M\|_*,\\
		\label{ine:det:main2}
		&\left| \sum_{i=1}^n \frac{\lambda_o}{\sqrt{n}}u_i \langle X_i, M\rangle\right|
		\leq b_{\mr{F}} r_{b,\mr{F}}\left\| \mr{T}_\Sigma (M)\right\|_\mr{F} + b_*r_{b,*}\|M\|_*,\\
		\label{ine:det:main3}
		&c_1 \| \mr{T}_\Sigma (M)\|_\mr{F}^2-c_2r_{c,\mr{F}}\|\mr{T}_\Sigma (M)\|_{\mr{F}} -c_3 r_{c}\leq \lambda_o^2\sum_{i=1}^n \left\{	-h \left(\frac{\xi_i-\langle X_i,M\rangle}{\lambda_o\sqrt{n}} \right)+h  \left(\frac{\xi_i}{\lambda_o\sqrt{n}}  \right)\right\} \frac{\langle X_i,M\rangle}{\lambda_o \sqrt{n}}
	\end{align}
	where $c_1>0, r_{a,\mr{F}},r_{a,*},r_{b,\mr{F}},r_{b,*},a_\mr{F},a_*,b_\mr{F},b_*,c_2,c_2, r_{\mr{F},c},r_c \geq 0$ are some numbers and  $\vecu = \left(u_1,\cdots,u_n\right)^\top$ is some $o$-sparse vector  such that $\|\vecu\|_\infty\leq c$ with some numerical constant  $c$ and $M \in \mbb{R}^{d_1 \times d_2}$ is some matrix such that $\|M\|_{\mr{F}} = r_0$ for some number $r_0$ satisfying \eqref{ine:det:main5}.
	Suppose that $\lambda_*$ satisfy 
	\begin{align}
		\label{ine:det:main4}
		\lambda_*-C_s>0,
	\end{align} 
	where
	\begin{align}
		C_s =\frac{a_\mr{F}r_{a,\mr{F}}+\sqrt{2}b_\mr{F}r_{b,\mr{F}}}{c_\kappa \sqrt{r}} + (a_*r_{a,*}+\sqrt{2}b_*r_{b,*}).
	\end{align}
	Then, for the number $r_0$ such that 
	\begin{align}
			\label{ine:det:main5}
		\frac{c_2r_{c,\mr{F}} +C_{\lambda_*} +\sqrt{c_1c_3r_c}}{c_1}  < r_0,
	\end{align}
	 where
	\begin{align}
		C_{\lambda_*} = (a_\mr{F}r_{a,\mr{F}}+\sqrt{2}b_\mr{F}r_{b,\mr{F}}) + (a_*r_{a,*}+\sqrt{2}b_*r_{b,*})c_\kappa \sqrt{r} +\lambda_*  c_\kappa \sqrt{r},
	\end{align}
	the optimal solution $\hat{B}$ satisfies
	\begin{align}
	\label{ine:det:main6}
		\|\mr{T}_\Sigma(\hat{B} -B^*)\|_\mr{F} & \leq r_0.
	\end{align}
\end{theorem}
In the remaining of Section \ref{sec:key}, we introduce some inequalities to prove  \eqref{ine:det:main1}-\eqref{ine:det:main3}
are satisfied with high probability and with appropriate value of $c_1, r_{a,\mr{F}},r_{a,*},r_{b,\mr{F}}, $ $r_{b,*},a_\mr{F},a_*,b_\mr{F},b_*,c_2,c_2, r_{\mr{F},c},r_c$ under the Assumptions \ref{a:mcs}, \ref{a:lasso}, \ref{a:mc1} or \ref{a:mc2}.

\subsection{Adversarial matrix compressed sensing}
\label{sec:amcs}
In Section \ref{sec:amcs}, suppose that Assumption \ref{a:mcs} holds.
\begin{lemma}
	\label{l:suphx:main}
	For $0<\delta<1/7$ and for any matrix $M \in \mbb{R}^{d_1\times d_2}$, with probability at least $1-\delta$, we have 
	\begin{align}
		\label{ine:main:h1}
	\left|\frac{1}{n} \sum_{i=1}^n  h\left(\frac{\xi_i}{\lambda_o \sqrt{n}}\right) \langle X_i,M\rangle  \right|\leq C L  \left\{\rho\sqrt{\frac{d_1+d_2}{n}}\|M\|_{*}+\sqrt{\frac{\log(1/\delta)}{n}}\|\mr{T}_\Sigma(M)\|_{\mr{F}}\right\}.
	\end{align}
\end{lemma}
\begin{remark}
	\cite{Tho2020Outlier} derived an upper bound of 
	\begin{align}
		\left|\frac{1}{n} \sum_{i=1}^n \xi_i\langle X_i,M\rangle  \right|.
	\end{align}
	The L.H.S. of \eqref{ine:main:h1} is an extended one because 
	$\{\xi_i\}_{i=1}^n$ is `wrapped' by the differential of the Huber loss function.
	Lemma \ref{ine:main:h1} enables us to deal with a heavy-tailed $\{\xi_i\}_{i=1}^n$. 
\end{remark}
\begin{corollary}[Corollary of Proposition 4 of \cite{Tho2020Outlier}]
	\label{ac:|uMv|-cs:main}
	We have for any $o$-sparse vector $\vecu \in \mbb{R}^n$ and any matrix $M \in \mbb{R}^{d_1 \times d_2}$,
	\begin{align}
	\left| \sum_{i=1}^n u_i \frac{1}{\sqrt{n}} \langle X_i, M \rangle\right| & \leq CL\left(\frac{1+\sqrt{\log (1/\delta)}}{\sqrt{n}}\left\| \mr{T}_\Sigma (M)\right\|_\mr{F} + \rho \sqrt{\frac{d_1+d_2}{n}}\|M\|_* +   \sqrt{\frac{o}{n} \log \frac{n}{o}}\| \mr{T}_\Sigma (M)\|_\mr{F}\right) \|\vecu\|_2
	\end{align}
	with probability at least $1-\delta$.
\end{corollary}
\begin{remark}
	Corollary \ref{ac:|uMv|-cs:main} is  derived by combining Proposition 4 of \cite{Tho2020Outlier} and Remark 4 of \cite{DalTho2019Outlier}. Proposition 4 of \cite{Tho2020Outlier} also plays an important role in both \cite{Tho2020Outlier} and our analysis, however the roles in \cite{Tho2020Outlier} and our analysis are different because we analyse \eqref{obj2-intro-m} and \cite{Tho2020Outlier} analyses \eqref{model:obj2-intro-m}.
\end{remark}

\begin{proposition}
	\label{p:sc1:main}
	Let  
	\begin{align}
	\mc{R}_{mcs} = \left\{ \Theta \in \mbb{R}^{d_1 \times d_2}\, |\,  \|\Theta\|_*\leq c_\kappa \|\mr{T}_\Sigma (\Theta) \|_\mr{F},\, \|\mr{T}_\Sigma (\Theta) \|_\mr{F} =r_{mcs} \right\},
	\end{align}
	where $r_{mcs}$ is a positive number such that $r_{mcs} \leq \frac{1}{4\sqrt{3} L^2}$.
	Assume that $\lambda_o  \sqrt{n} \geq 72L^4\sigma$.
	Then, with probability at least $1-\delta$, we have
	\begin{align}
		& \inf_{\Theta \in \mc{R}_{mcs}}\left[\sum_{i=1}^n \lambda_o^2 \left\{-h\left(\frac{\xi_i + \langle  X_i, \Theta \rangle}{\lambda_o\sqrt{n}}\right)+h \left(\frac{\xi_i}{\lambda_o\sqrt{n}}\right) \right\}\frac{\langle   X_i,\Theta\rangle }{\lambda_o \sqrt{n}}\right]\nonumber \\
		&\geq \frac{1}{3}\|\mr{T}_\Sigma(\Theta)\|_{\mr{F}}^2-C\left(L\rho c_\kappa \sqrt{r\frac{d_1+d_2}{n}}+\sqrt{8\frac{\log(1/\delta)}{n}} \right)\|\mr{T}_\Sigma (\Theta)\|_{\mr{F}} -5\frac{\log(1/\delta)}{n}.
\end{align} 
\end{proposition}
\begin{remark}
	Proposition \ref{p:sc1:main} implies restricted strong convexity of the Huber loss.
	Essentially, the same techniques found as in  \cite{CheZho2020Robust} are used to prove Proposition \ref{p:sc1:main}.
\end{remark}
\subsection{Adversarial matrix completion}
\begin{lemma}
	\label{l:mcspec:main}
	Suppose that Assumption \ref{a:mc1} or Assumption \ref{a:mc2} holds.
 For $\delta>0$, with probability at least $1-\delta$, we have 
	\begin{align}
		\left\| \frac{\lambda_o}{\sqrt{n}} \sum_{i=1}^n  h\left(\frac{\xi_i}{\lambda_o\sqrt{n}}\right) X_i\right\|_\mr{op} 
		& \leq C \left(\sigma_{\xi}\sqrt{\frac{d_{mc}(\log d_{mc}+\log(1/\delta))}{n}}+\lambda_o\frac{d_{mc}(\log d_{mc}+\log(1/\delta))}{\sqrt{n}}\right).
	\end{align}
\end{lemma}
\begin{remark}
	\cite{Tho2020Outlier} derived an upper bound of  
	\begin{align}
		\left|\frac{1}{n} \sum_{i=1}^n \xi_i\langle X_i,M\rangle  \right|.
	\end{align}
	The L.H.S. of \eqref{ine:main:h1} is an extended one because 
	$\{\xi_i\}_{i=1}^n$ is `wrapped' by the differential of the Huber loss function.
	Lemma \ref{ine:main:h1} enables us to deal with a heavy-tailed $\{\xi_i\}_{i=1}^n$. 
\end{remark}

\begin{lemma}
	\label{al:|uMv|-mc:main}
	Suppose that Assumption \ref{a:mc1} or Assumption \ref{a:mc2} holds.
	Assume that for any $M\in \mbb{R}^{d_1\times d_2}$,
	\begin{align}
		\label{cond:mccommon}
		\|M\|_\infty \leq c_m \frac{1}{d_{mc}}\|M\|_{\mr{F}}.
	\end{align} 
	for some number $c_m$.
	Then, for any $o$-sparse vector $\vecu \in \mbb{R}^n$ such that $\|\vecu\|_2 \leq 2 \sqrt{o}$,
	we have
	\begin{align}
		\label{mch}
	\left| \sum_{i=1}^n \frac{\lambda_o}{\sqrt{n}}u_i \langle X_i, M\rangle\right|  &\leq c_m \lambda_o\sqrt{n}\frac{o}{n} \|M\|_\mr{F}.
	\end{align}
\end{lemma}
\begin{remark}
	Lemma \ref{al:|uMv|-mc:main} corresponds to Corollary \ref{ac:|uMv|-cs:main} in adversarial matrix compressed sensing
	In the proof of Theorem \ref{t:mc:main1} and  Theorem \ref{t:mc:main2},
	different strategies are used depending on the magnitude of the spikiness of $\hat{B}-B^*$
	like \cite{NegWai2012Restricted} and \cite{FanWanZhu2021Shrinkage}.
	Assumption \eqref{mch} is derived when the spikiness of $\hat{B}-B^*$ is sufficiently small.
\end{remark}

\begin{corollary}
	\label{c:mcsc1:main}
	Suppose that Assumption \ref{a:mc1} holds.
	Let  
	\begin{align}
	\mc{R}_{mc} = \left\{ \Theta\in \mbb{R}^{d_1 \times d_2}\, |\, \|\Theta\|_* \leq C \sqrt{r} \|\Theta\|_{\mr{F}},\, \|\Theta\|_\mr{F} =r_{mc}\right\},
	\end{align}
	where $r_{mc}$ is some number.
	Suppose 
	\begin{align}
		\label{ine:mcs4:main}
		\|\Theta\|_\infty &\leq \frac{1}{12r_{mc}} \frac{1}{d_{mc}}\|\Theta\|_{\mr{F}},\\
		\label{ine:constraint:main}
		\|\Theta\|_\infty &\leq 2\frac{\alpha^*}{d_{mc}},
	\end{align}
	and $\lambda_o  \sqrt{n} \geq 2\sigma_{\xi,\alpha}\min\left\{\left(\frac{n}{o}\right)^\frac{1}{\alpha+1},\left(\frac{n}{rd_{mc}\log  d_{mc}}\right)^\frac{1}{\alpha}\right\}$.
	Then, with probability at least $1-\delta$, we have
	\begin{align}
	&\inf_{\Theta \in \mc{R}_{mc}}\left[\lambda_o^2\sum_{i=1}^n  \left\{-h\left(\frac{\xi_i - \langle  X_i, \Theta\rangle}{\lambda_o\sqrt{n}}\right)+h \left(\frac{\xi_i}{\lambda_o\sqrt{n}}\right) \right\}\frac{\langle   X_i,\Theta\rangle}{\lambda_o\sqrt{n}} \right] \nonumber\\
	&\geq \frac{1}{3}\|\Theta\|_\mr{F}^2 -C \alpha^* \left(\sqrt{r\frac{d_{mc}\{\log d_{mc}+\log(1/\delta)\}}{n}}+\sqrt{r}\frac{d_{mc}\log  d_{mc}}{n}+ \left(\frac{o}{n}\right)^\frac{\alpha}{2(1+\alpha)}	\right)\|\Theta\|_\mr{F}-C\frac{\log(1/\delta)}{n}.
	\end{align} 
\end{corollary}

\begin{corollary}
	\label{c:mcsc2:main}
	Suppose that Assumption \ref{a:mc2} holds.
	Let  
	\begin{align}
	\mc{R}_{mc} = \left\{ L\in \mbb{R}^{d_1 \times d_2}\, |\, \|\Theta\|_* \leq C \sqrt{r} \|\Theta\|_{\mr{F}},\, \|\Theta\|_\mr{F} =r_{mc}\right\}.
	\end{align}
	Suppose 
	\begin{align}
		\label{ine:mcs3:main}
		\|\Theta\|_\infty &\leq \frac{1}{12r_{mc}} \frac{1}{d_{mc}}\|\Theta\|_{\mr{F}},\\
		\label{ine:constraint2:main}
		\|\Theta\|_\infty &\leq 2\frac{\alpha^*}{d_{mc}},
	\end{align}
	and $\lambda_o  \sqrt{n} \geq 2 \sigma_{\xi ,\psi_\alpha} \min\left\{\log^\frac{1}{\alpha} \frac{n}{o}, \log^\frac{1}{\alpha} \frac{n}{rd_{mc}\log d_{mc}}\right\}$.
	Then, with probability at least $1-\delta$, we have
	\begin{align}
	&\inf_{\Theta \in \mc{R}_{mc}}\left[\lambda_o^2\sum_{i=1}^n  \left\{-h\left(\frac{\xi_i - \langle  X_i, \Theta\rangle}{\lambda_o\sqrt{n}}\right)+h \left(\frac{\xi_i}{\lambda_o\sqrt{n}}\right) \right\}\frac{\langle   X_i,\Theta\rangle}{\lambda_o\sqrt{n}} \right] \nonumber\\
	&\geq \frac{2}{3}\|\Theta\|_\mr{F}^2 -C\alpha^{*}\left(\sqrt{r\frac{d_{mc}\{\log d_{mc}+\log(1/\delta)\}}{n}}+\sqrt{r}\frac{d_{mc}\log  d_{mc}}{n}+\sqrt{\frac{o}{n}}\right)\|\Theta\|_\mr{F}-C\frac{\log(1/\delta)}{n}.
	\end{align} 
\end{corollary}
\begin{remark}
	In adversarial matrix completion, $\{\mr{vec}(X_i)\}_{i=1}^n$ is not a sequence of $L$-subGaussian random vector and we need to prove Corollaries \ref{c:mcsc1:main} and \ref{c:mcsc2:main} in a different way from Proposition \ref{p:sc1:main}.
	However, \eqref{ine:mcs4:main} and \eqref{ine:constraint:main} or \eqref{ine:mcs3:main} and \eqref{ine:constraint2:main} fulfill the role of $L$-subGaussianness in the proof of Corollaries \ref{c:mcsc1:main} and \ref{c:mcsc2:main}.
	Like Lemma \ref{al:|uMv|-mc:main}, \eqref{ine:mcs4:main}, \eqref{ine:mcs3:main} are derived when the spikiness of $\hat{B}-B^*$ is sufficiently small.
	Assumptions \eqref{ine:constraint:main} and \eqref{ine:constraint2:main} are derived from the constraint of the optimization problem \eqref{obj-mc}.
\end{remark}

\bibliographystyle{plain}
\bibliography{ARLRME} 

\begin{thebibliography}{10}

\bibitem{BakPra2021Robust}
Ainesh Bakshi and Adarsh Prasad.
\newblock Robust linear regression: Optimal rates in polynomial time.
\newblock In {\em Proceedings of the 53rd Annual ACM SIGACT Symposium on Theory
  of Computing}, pages 102--115, 2021.

\bibitem{BalDuLiSin2017Computationally}
Sivaraman Balakrishnan, Simon~S Du, Jerry Li, and Aarti Singh.
\newblock Computationally efficient robust sparse estimation in high
  dimensions.
\newblock In {\em Conference on Learning Theory}, pages 169--212. PMLR, 2017.

\bibitem{BelLecTsy2018Slope}
Pierre~C Bellec, Guillaume Lecu{\'e}, and Alexandre~B Tsybakov.
\newblock Slope meets lasso: improved oracle bounds and optimality.
\newblock {\em The Annals of Statistics}, 46(6B):3603--3642, 2018.

\bibitem{BodVanSabSuCan2015Slope}
Ma{\l}gorzata Bogdan, Ewout Van Den~Berg, Chiara Sabatti, Weijie Su, and
  Emmanuel~J Cand{\`e}s.
\newblock Slope―adaptive variable selection via convex optimization.
\newblock {\em The annals of applied statistics}, 9(3):1103, 2015.

\bibitem{BouLugMas2013concentration}
St{\'e}phane Boucheron, G{\'a}bor Lugosi, and Pascal Massart.
\newblock {\em Concentration inequalities: A nonasymptotic theory of
  independence}.
\newblock Oxford university press, 2013.

\bibitem{CaiZha2013Sparse}
T~Tony Cai and Anru Zhang.
\newblock Sparse representation of a polytope and recovery of sparse signals
  and low-rank matrices.
\newblock {\em IEEE transactions on information theory}, 60(1):122--132, 2013.

\bibitem{CaiZhou2016Matrix}
T~Tony Cai and Wen-Xin Zhou.
\newblock Matrix completion via max-norm constrained optimization.
\newblock {\em Electronic Journal of Statistics}, 10(1):1493--1525, 2016.

\bibitem{CanRec2009Exact}
Emmanuel~J Cand{\`e}s and Benjamin Recht.
\newblock Exact matrix completion via convex optimization.
\newblock {\em Foundations of Computational mathematics}, 9(6):717, 2009.

\bibitem{CarKloLofNik2018Adaptive}
Alexandra Carpentier, Olga Klopp, Matthias L{\"o}ffler, and Richard Nickl.
\newblock Adaptive confidence sets for matrix completion.
\newblock {\em Bernoulli}, 24(4A):2429--2460, 2018.

\bibitem{Cha2015Matrix}
Sourav Chatterjee.
\newblock Matrix estimation by universal singular value thresholding.
\newblock {\em The Annals of Statistics}, 43(1):177--214, 2015.

\bibitem{CheGaoRen2018robust}
Mengjie Chen, Chao Gao, and Zhao Ren.
\newblock Robust covariance and scatter matrix estimation under huber’s
  contamination model.
\newblock {\em The Annals of Statistics}, 46(5):1932--1960, 2018.

\bibitem{CheZho2020Robust}
Xi~Chen and Wen-Xin Zhou.
\newblock Robust inference via multiplier bootstrap.
\newblock {\em Annals of Statistics}, 48(3):1665--1691, 2020.

\bibitem{CheCarMan2013Robust}
Yudong Chen, Constantine Caramanis, and Shie Mannor.
\newblock Robust sparse regression under adversarial corruption.
\newblock In {\em International Conference on Machine Learning}, pages
  774--782. PMLR, 2013.

\bibitem{CheXuCarSan2015Matrix}
Yudong Chen, Huan Xu, Constantine Caramanis, and Sujay Sanghavi.
\newblock Matrix completion with column manipulation: Near-optimal
  sample-robustness-rank tradeoffs.
\newblock {\em IEEE Transactions on Information Theory}, 62(1):503--526, 2015.

\bibitem{CheDiaGe2019High}
Yu~Cheng, Ilias Diakonikolas, and Rong Ge.
\newblock High-dimensional robust mean estimation in nearly-linear time.
\newblock In {\em Proceedings of the Thirtieth Annual ACM-SIAM Symposium on
  Discrete Algorithms}, pages 2755--2771. SIAM, 2019.

\bibitem{CheDiaGeWoo2019Faster}
Yu~Cheng, Ilias Diakonikolas, Rong Ge, and David~P Woodruff.
\newblock Faster algorithms for high-dimensional robust covariance estimation.
\newblock In {\em Conference on Learning Theory}, pages 727--757. PMLR, 2019.

\bibitem{CheYesFlaBar2019Fast}
Yeshwanth Cherapanamjeri, Nicolas Flammarion, and Peter~L Bartlett.
\newblock Fast mean estimation with sub-gaussian rates.
\newblock In {\em Conference on Learning Theory}, pages 786--806. PMLR, 2019.

\bibitem{CheGupJai2017Nearly}
Yeshwanth Cherapanamjeri, Kartik Gupta, and Prateek Jain.
\newblock Nearly optimal robust matrix completion.
\newblock In {\em International Conference on Machine Learning}, pages
  797--805. PMLR, 2017.

\bibitem{Chi2020Erm}
Geoffrey Chinot.
\newblock Erm and rerm are optimal estimators for regression problems when
  malicious outliers corrupt the labels.
\newblock {\em Electronic Journal of Statistics}, 14(2):3563--3605, 2020.

\bibitem{DalTho2019Outlier}
Arnak Dalalyan and Philip Thompson.
\newblock Outlier-robust estimation of a sparse linear model using
  $\ell_1$-penalized huber's m-estimator.
\newblock In H.~Wallach, H.~Larochelle, A.~Beygelzimer, F.~d'Alch\'{e} Buc,
  E.~Fox, and R.~Garnett, editors, {\em Advances in Neural Information
  Processing Systems 32}, pages 13188--13198. Curran Associates, Inc., 2019.

\bibitem{DalMin2020All}
Arnak~S Dalalyan and Arshak Minasyan.
\newblock All-in-one robust estimator of the gaussian mean.
\newblock {\em arXiv preprint arXiv:2002.01432}, 2020.

\bibitem{DepLec2019Robust}
Jules Depersin and Guillaume Lecu{\'e}.
\newblock Robust subgaussian estimation of a mean vector in nearly linear time.
\newblock {\em arXiv preprint arXiv:1906.03058}, 2019.

\bibitem{DevLerLugOli2016Sub}
Luc Devroye, Matthieu Lerasle, Gabor Lugosi, and Roberto~I Oliveira.
\newblock Sub-gaussian mean estimators.
\newblock {\em The Annals of Statistics}, 44(6):2695--2725, 2016.

\bibitem{DiaKanKanLiMoiSte2019Robust}
Ilias Diakonikolas, Gautam Kamath, Daniel Kane, Jerry Li, Ankur Moitra, and
  Alistair Stewart.
\newblock Robust estimators in high-dimensions without the computational
  intractability.
\newblock {\em SIAM Journal on Computing}, 48(2):742--864, 2019.

\bibitem{DiaKanKanLiMoiSte2017Being}
Ilias Diakonikolas, Gautam Kamath, Daniel~M Kane, Jerry Li, Ankur Moitra, and
  Alistair Stewart.
\newblock Being robust (in high dimensions) can be practical.
\newblock In {\em Proceedings of the 34th International Conference on Machine
  Learning-Volume 70}, pages 999--1008. JMLR. org, 2017.

\bibitem{DiaKamKanLiMoiSte2018Robustly}
Ilias Diakonikolas, Gautam Kamath, Daniel~M Kane, Jerry Li, Ankur Moitra, and
  Alistair Stewart.
\newblock Robustly learning a gaussian: Getting optimal error, efficiently.
\newblock In {\em Proceedings of the Twenty-Ninth Annual ACM-SIAM Symposium on
  Discrete Algorithms}, pages 2683--2702. Society for Industrial and Applied
  Mathematics, 2018.

\bibitem{DiaKanKarPriSte2019Outlier}
Ilias Diakonikolas, Daniel Kane, Sushrut Karmalkar, Eric Price, and Alistair
  Stewart.
\newblock Outlier-robust high-dimensional sparse estimation via iterative
  filtering.
\newblock {\em Advances in Neural Information Processing Systems}, 32, 2019.

\bibitem{DiaKonSte2019Efficient}
Ilias Diakonikolas, Weihao Kong, and Alistair Stewart.
\newblock Efficient algorithms and lower bounds for robust linear regression.
\newblock In {\em Proceedings of the Thirtieth Annual ACM-SIAM Symposium on
  Discrete Algorithms}, pages 2745--2754. SIAM, 2019.

\bibitem{DonHopLi2019Quantum}
Yihe Dong, Samuel Hopkins, and Jerry Li.
\newblock Quantum entropy scoring for fast robust mean estimation and improved
  outlier detection.
\newblock In {\em Advances in Neural Information Processing Systems}, pages
  6067--6077, 2019.

\bibitem{ElsVan2018Robust}
Andreas Elsener and Sara van~de Geer.
\newblock Robust low-rank matrix estimation.
\newblock {\em The Annals of Statistics}, 46(6B):3481--3509, 2018.

\bibitem{FanLiWan2017Estimation}
Jianqing Fan, Quefeng Li, and Yuyan Wang.
\newblock Estimation of high dimensional mean regression in the absence of
  symmetry and light tail assumptions.
\newblock {\em Journal of the Royal Statistical Society. Series B, Statistical
  methodology}, 79(1):247, 2017.

\bibitem{FanLi2001Variable}
Jianqing Fan and Runze Li.
\newblock Variable selection via nonconcave penalized likelihood and its oracle
  properties.
\newblock {\em Journal of the American statistical Association},
  96(456):1348--1360, 2001.

\bibitem{FanLiuSunZha2018Lamm}
Jianqing Fan, Han Liu, Qiang Sun, and Tong Zhang.
\newblock I-lamm for sparse learning: Simultaneous control of algorithmic
  complexity and statistical error.
\newblock {\em Annals of statistics}, 46(2):814, 2018.

\bibitem{FanWanZhu2021Shrinkage}
Jianqing Fan, Weichen Wang, and Ziwei Zhu.
\newblock A shrinkage principle for heavy-tailed data: High-dimensional robust
  low-rank matrix recovery.
\newblock {\em The Annals of Statistics}, 49(3):1239--1266, 2021.

\bibitem{Gao2020Robust}
Chao Gao.
\newblock Robust regression via mutivariate regression depth.
\newblock {\em Bernoulli}, 26(2):1139--1170, 2020.

\bibitem{Hop2020Mean}
Samuel~B Hopkins.
\newblock Mean estimation with sub-gaussian rates in polynomial time.
\newblock {\em The Annals of Statistics}, 48(2):1193--1213, 2020.

\bibitem{JaiNet2015Fast}
Prateek Jain and Praneeth Netrapalli.
\newblock Fast exact matrix completion with finite samples.
\newblock In {\em Conference on Learning Theory}, pages 1007--1034. PMLR, 2015.

\bibitem{KimXin2012Tree}
Seyoung Kim and Eric~P Xing.
\newblock Tree-guided group lasso for multi-response regression with structured
  sparsity, with an application to eqtl mapping.
\newblock {\em The Annals of Applied Statistics}, 6(3):1095--1117, 2012.

\bibitem{Klo2014Noisy}
Olga Klopp.
\newblock Noisy low-rank matrix completion with general sampling distribution.
\newblock {\em Bernoulli}, 20(1):282--303, 2014.

\bibitem{Klo2015matrix}
Olga Klopp.
\newblock Matrix completion by singular value thresholding: sharp bounds.
\newblock {\em Electronic journal of statistics}, 9(2):2348--2369, 2015.

\bibitem{KloLouTsy2017Robust}
Olga Klopp, Karim Lounici, and Alexandre~B Tsybakov.
\newblock Robust matrix completion.
\newblock {\em Probability Theory and Related Fields}, 169(1-2):523--564, 2017.

\bibitem{KolLouTsy2011Nuclear}
Vladimir Koltchinskii, Karim Lounici, and Alexandre~B Tsybakov.
\newblock Nuclear-norm penalization and optimal rates for noisy low-rank matrix
  completion.
\newblock {\em The Annals of Statistics}, 39(5):2302--2329, 2011.

\bibitem{KotSteSte2018Robust}
Pravesh~K Kothari, Jacob Steinhardt, and David Steurer.
\newblock Robust moment estimation and improved clustering via sum of squares.
\newblock In {\em Proceedings of the 50th Annual ACM SIGACT Symposium on Theory
  of Computing}, pages 1035--1046. ACM, 2018.

\bibitem{KucCha2018Moving}
Arun~Kumar Kuchibhotla and Abhishek Chakrabortty.
\newblock Moving beyond sub-gaussianity in high-dimensional statistics:
  Applications in covariance estimation and linear regression.
\newblock {\em arXiv preprint arXiv:1804.02605}, 2018.

\bibitem{LaiRaoVem2016Agnostic}
Kevin~A Lai, Anup~B Rao, and Santosh Vempala.
\newblock Agnostic estimation of mean and covariance.
\newblock In {\em Foundations of Computer Science (FOCS), 2016 IEEE 57th Annual
  Symposium on}, pages 665--674. IEEE, 2016.

\bibitem{LeiLuhVenZha2020Fast}
Zhixian Lei, Kyle Luh, Prayaag Venkat, and Fred Zhang.
\newblock A fast spectral algorithm for mean estimation with sub-gaussian
  rates.
\newblock In {\em Conference on Learning Theory}, pages 2598--2612. PMLR, 2020.

\bibitem{LiuSheLiCar2020High}
Liu Liu, Yanyao Shen, Tianyang Li, and Constantine Caramanis.
\newblock High dimensional robust sparse regression.
\newblock In {\em International Conference on Artificial Intelligence and
  Statistics}, pages 411--421. PMLR, 2020.

\bibitem{Loh2017Statistical}
Po-Ling Loh.
\newblock Statistical consistency and asymptotic normality for high-dimensional
  robust $ m $-estimators.
\newblock {\em The Annals of Statistics}, 45(2):866--896, 2017.

\bibitem{LugMen2019Sub}
G{\'a}bor Lugosi and Shahar Mendelson.
\newblock Sub-gaussian estimators of the mean of a random vector.
\newblock {\em The annals of statistics}, 47(2):783--794, 2019.

\bibitem{LugMen2021Robust}
Gabor Lugosi and Shahar Mendelson.
\newblock Robust multivariate mean estimation: the optimality of trimmed mean.
\newblock {\em The Annals of Statistics}, 49(1):393--410, 2021.

\bibitem{Mas2000Constants}
Pascal Massart.
\newblock About the constants in talagrand's concentration inequalities for
  empirical processes.
\newblock {\em The Annals of Probability}, 28(2):863--884, 2000.

\bibitem{Men2016Upper}
Shahar Mendelson.
\newblock Upper bounds on product and multiplier empirical processes.
\newblock {\em Stochastic Processes and their Applications},
  126(12):3652--3680, 2016.

\bibitem{MenZhi2020Robust}
Shahar Mendelson and Nikita Zhivotovskiy.
\newblock Robust covariance estimation under $ l\_ $\{$4$\}$-l\_ $\{$2$\}$ $
  norm equivalence.
\newblock {\em Annals of Statistics}, 48(3):1648--1664, 2020.

\bibitem{Min2018Sub}
Stanislav Minsker.
\newblock Sub-gaussian estimators of the mean of a random matrix with
  heavy-tailed entries.
\newblock {\em The Annals of Statistics}, 46(6A):2871--2903, 2018.

\bibitem{NegWai2011Estimation}
Sahand Negahban and Martin~J Wainwright.
\newblock Estimation of (near) low-rank matrices with noise and
  high-dimensional scaling.
\newblock {\em The Annals of Statistics}, pages 1069--1097, 2011.

\bibitem{NegWai2012Restricted}
Sahand Negahban and Martin~J Wainwright.
\newblock Restricted strong convexity and weighted matrix completion: Optimal
  bounds with noise.
\newblock {\em The Journal of Machine Learning Research}, 13(1):1665--1697,
  2012.

\bibitem{NguTra2012Robust}
Nam~H Nguyen and Trac~D Tran.
\newblock Robust lasso with missing and grossly corrupted observations.
\newblock {\em IEEE transactions on information theory}, 59(4):2036--2058,
  2012.

\bibitem{PraBalRav2020Robust}
Adarsh Prasad, Sivaraman Balakrishnan, and Pradeep Ravikumar.
\newblock A robust univariate mean estimator is all you need.
\newblock In {\em International Conference on Artificial Intelligence and
  Statistics}, pages 4034--4044. PMLR, 2020.

\bibitem{RasWaiYu2010Restricted}
Garvesh Raskutti, Martin~J Wainwright, and Bin Yu.
\newblock Restricted eigenvalue properties for correlated gaussian designs.
\newblock {\em The Journal of Machine Learning Research}, 11:2241--2259, 2010.

\bibitem{RohTsy2011Estimation}
Angelika Rohde and Alexandre~B Tsybakov.
\newblock Estimation of high-dimensional low-rank matrices.
\newblock {\em The Annals of Statistics}, 39(2):887--930, 2011.

\bibitem{SheChe2017Robust}
Yiyuan She and Kun Chen.
\newblock Robust reduced-rank regression.
\newblock {\em Biometrika}, 104(3):633--647, 2017.

\bibitem{SuCan2016Slope}
Weijie Su and Emmanuel Candes.
\newblock Slope is adaptive to unknown sparsity and asymptotically minimax.
\newblock {\em The Annals of Statistics}, 44(3):1038--1068, 2016.

\bibitem{SunZhoFan2020Adaptive}
Qiang Sun, Wen-Xin Zhou, and Jianqing Fan.
\newblock Adaptive huber regression.
\newblock {\em Journal of the American Statistical Association},
  115(529):254--265, 2020.

\bibitem{Tal2014Upper}
Michel Talagrand.
\newblock {\em Upper and lower bounds for stochastic processes}, volume~60.
\newblock Springer, 2014.

\bibitem{Tho2020Outlier}
Philip Thompson.
\newblock Outlier-robust sparse/low-rank least-squares regression and robust
  matrix completion.
\newblock {\em arXiv preprint arXiv:2012.06750}, 2020.

\bibitem{Tib1996Regression}
Robert Tibshirani.
\newblock Regression shrinkage and selection via the lasso.
\newblock {\em Journal of the Royal Statistical Society: Series B},
  58(1):267--288, 1996.

\bibitem{VelRei2013Multivariate}
Raja Velu and Gregory~C Reinsel.
\newblock {\em Multivariate reduced-rank regression: theory and applications},
  volume 136.
\newblock Springer Science \& Business Media, 2013.

\bibitem{Ver2018High}
Roman Vershynin.
\newblock {\em High-dimensional probability: An introduction with applications
  in data science}, volume~47.
\newblock Cambridge university press, 2018.

\bibitem{VlaGirNguArb2020sub}
Mariia Vladimirova, St{\'e}phane Girard, Hien Nguyen, and Julyan Arbel.
\newblock Sub-weibull distributions: Generalizing sub-gaussian and
  sub-exponential properties to heavier tailed distributions.
\newblock {\em Stat}, 9(1):e318, 2020.

\bibitem{YuaLin2006Model}
Ming Yuan and Yi~Lin.
\newblock Model selection and estimation in regression with grouped variables.
\newblock {\em Journal of the Royal Statistical Society: Series B},
  68(1):49--67, 2006.

\bibitem{Zha2010Nearly}
Cun-Hui Zhang.
\newblock Nearly unbiased variable selection under minimax concave penalty.
\newblock {\em The Annals of statistics}, 38(2):894--942, 2010.

\bibitem{ZouHas2005Regularization}
Hui Zou and Trevor Hastie.
\newblock Regularization and variable selection via the elastic net.
\newblock {\em Journal of the Royal Statistical Society: Series B},
  67(2):301--320, 2005.

\end{thebibliography}

\appendix
\begin{center}\textbf{Appendix} \end{center}
\section{Structure of the appendix}
In this appendix, we give the proofs of Theorems \ref{t:cs:main}-\ref{t:mc:main2}.
In Section \ref{s:deterministic_arguments}, we give deterministic arguments without the randomness of covariates and random noises.
In Sections \ref{sec:inesforcs}, \ref{sec:inesforla} and \ref{sec:inesformc}, we introduce and prove 
some properties about covariates and random noises. In these sections, we assume the randomness of covariates and random noises.
In Section \ref{asec:maincs}, we prove Theorem \ref{t:cs:main} using the results of Sections \ref{s:deterministic_arguments} and \ref{sec:inesforcs}.
In Section \ref{s:laap}, we prove Theorem \ref{t:lasso:main} using the results of Sections \ref{s:deterministic_arguments} and \ref{sec:inesforla}.
In Sections \ref{s:mc1} and \ref{s:mc2}, we prove Theorems \ref{t:mc:main1} and \ref{t:mc:main2} using the results of Sections \ref{s:deterministic_arguments} and \ref{sec:inesformc}.
The value of the numerical constant $C$ shall be allowed to change from line to line. 

\section{Deterministic arguments}
\label{s:deterministic_arguments}
In Section \ref{s:deterministic_arguments}, we give deterministic arguments without the randomness of covariates and random noises.
In Section \ref{sec:main}, we state our main theorem in the deterministic forms.
In Section \ref{r12} and \ref{sec:mainproposition} are prepared for Sections \ref{sec:main}.
First, we introduce a basic lemma related to convexity.
\begin{lemma}
	\label{ap:ine:conv}
	For a differentiable function $f(x)$, we denote its derivative $f'(x)$. 
	For any differentiable and convex function $f(x)$, we have
	\begin{align}
		(f'(a) -f'(b))(a-b) \geq 0.
	\end{align}
\end{lemma}
\begin{proof}
	From the definition of the convexity, we have
	\begin{align}
		f(a)-f(b) \geq f'(b)(a-b)\text{ and }f(b)-f(a) \geq f'(a)(b-a).
	\end{align}
From the inequalities above, we have
\begin{align}
	0 \geq f'(b)(a-b) + f'(a)(b-a) = (f'(b)-f'(a))(a-b)\Rightarrow 0 \leq (f'(a)-f'(b))(a-b).
\end{align}
\end{proof}

\subsection{Matrix restricted eigenvalue condition}
\label{r12}

When the matrix $\Sigma$ satisfies $\mr{MRE}(r,c_0,\kappa)$, the following lemma is obtained.
\begin{lemma} 
	\label{l:rev}
	Suppose that $\Sigma$ satisfies $\mr{MRE}(r,c_0,\kappa)$. Then, we have
	\begin{align}
		\label{i:re-norm-1}
		\|M\|_* &\leq c_\kappa \sqrt{r}\| \mr{T}_\Sigma (M)\|_\mr{F},
	\end{align}
	where $c_{\kappa}:=\frac{c_0+1}{\kappa}$, for any $M \in \mbb{R}^{d_1 \times d_2}$ satisfying \eqref{con:re-v} for any $E \in \mbb{R}^{d_1 \times d_2}$ with $\mr{rank}(E) \leq r$.
\end{lemma}
\begin{proof}
	We have
\begin{align}
	\|M\|_* &= \|\mr{P}_{E} (M)\|_*+\|\mr{P}_{E}^\bot (M)\|_* \nonumber\\
	&\stackrel{(a)}{\leq} (c_0+1)\|\mr{P}_{E} (M)\|_*\nonumber\\
	&\stackrel{(b)}{\leq}(c_0+1) \sqrt{r}\|\mr{P}_E (M)\|_\mr{F} \nonumber\\
	&\stackrel{(c)}{\leq} \frac{c_0+1}{\kappa} \sqrt{r}\|\mr{T}_\Sigma(M)\|_\mr{F}
\end{align}
where (a) follows from~\eqref{con:re-v}, (b) follows from the fact that $\|A\|_* \leq \sqrt{r} \|A\|_\mr{F}$
for a matrix such that $\mr{rank}(A) \leq r$ and (c) follows from~\eqref{con:re-v-pre}.

\end{proof}

\subsection{Relation between $\|\hat{B}-B^*\|_*$ and $\|\mr{T}_\Sigma (\hat{B}-B^*)\|_{\mr{F}}$ under MRE condition}
\label{sec:mainproposition}
In Section~\ref{sec:mainproposition}, we show the main proposition to obtain the main theorem. This proposition enables us to treat the adversarial matrix compressed sensing and adversarial matrix completion in a unified approach. 
Let 
\begin{align}
  h(t) =	\frac{d}{dt} H(t) =   \begin{cases}
  	t\quad &|t| \leq 1\\
  	\mr{sgn}(t)\quad &|t|  >1
  \end{cases}.
\end{align}
For a matrix $M$, let $\|M\|_{\rm op}$ be the operator norm. 
\begin{proposition}
	\label{p:starMRE}
	Let $\Theta_\eta = \eta(\hat{B}-B^*)$.
	Consider the optimization problem \eqref{obj2-intro-m}.
 Assume that $\Sigma$ satisfies $\mr{MRE}(r,c_0,\kappa)$ and assume that for any $o$-sparse vector $u = \left(u_1,\cdots,u_n\right)^\top$ such that $\|\vecu\|_\infty\leq c$, where $c$ is some numerical constant,
	\begin{align}
		\label{ine:det:xis0}
		&\left| \frac{\lambda_o}{\sqrt{n}}\sum_{i=1}^n h\left(\frac{\xi_i}{\lambda_o \sqrt{n}} \right) \langle  X_i, \Theta_\eta \rangle \right| \leq a_{\mr{F}} r_{a,\mr{F}} \left\| \mr{T}_\Sigma (\Theta_\eta)\right\|_\mr{F} + a_*r_{a,*}\|\Theta_\eta\|_*,\\
		\label{ine:det:upper0}
		&\left| \sum_{i=1}^n \frac{\lambda_o}{\sqrt{n}}u_i \langle X_i, \Theta_\eta\rangle\right|
		\leq b_{\mr{F}} r_{b,\mr{F}}\left\| \mr{T}_\Sigma (\Theta_\eta)\right\|_\mr{F} + b_*r_{b,*}\|\Theta_\eta\|_*,
	\end{align}
	where $r_{a,\mr{F}},r_{a,*},r_{b,\mr{F}},r_{b,*},a_\mr{F},a_*,b_\mr{F},b_* \geq 0$ are some numbers. Suppose that $\lambda_*$ satisfy 
	\begin{align}
		\label{ine:det:par}
		\lambda_*-C_s>0,
	\end{align} 
	where
	\begin{align}
		C_s =\frac{a_\mr{F}r_{a,\mr{F}}+\sqrt{2}b_\mr{F}r_{b,\mr{F}}}{c_\kappa \sqrt{r}} + (a_*r_{a,*}+\sqrt{2}b_*r_{b,*}).
	\end{align}
Then, we have
	\begin{align}
		\| \mr{P}_{B^*}^\bot(\Theta_\eta)\|_*  \leq \frac{\lambda_* + C_s}{\lambda_* - C_s  }\| \mr{P}_{B^*}(\Theta_\eta)\|_* .
	\end{align}
\end{proposition}
\begin{proof}
Let  $\hat{\Theta} = \hat{B}-B^*$ and 
\begin{align}
 Q'(\eta) =\frac{\lambda_o}{\sqrt{n}}\sum_{i=1}^n \{-h (r_{B_\eta,i})+h(r_{B^*,i}) \}\langle X_i, \hat{B}-B^*\rangle,
 \end{align}
where
\begin{align}
	r_{M,i} =\frac{y_i-\langle X_i, M\rangle}{\lambda_o \sqrt{n}}.
\end{align}
From the proof of Lemma F.2. of \cite{FanLiuSunZha2018Lamm}, we have $\eta Q'(\eta) \leq \eta Q'(1)$ and this means
\begin{align}
\label{ine:det:1}
\sum_{i=1}^n\frac{\lambda_o}{\sqrt{n}}\left\{-h (r_{B^*+\Theta_\eta,i})+h(r_{B^*,i})\right\}\langle X_i,\Theta_\eta\rangle  &\leq \sum_{i=1}^n\frac{\lambda_o}{\sqrt{n}}\eta \left\{-h (r_{B^*+\Theta_\eta,i}) +h(r_{B^*,i})\right\}\langle X_i,\hat{\Theta}\rangle.
\end{align}
Let $\partial^* M$ be the sub-differential of $\|M\|_*$.
Adding $\eta \lambda_*\left(\|\hat{B}\|_*-\|B^*\|_*\right) $ to both sides of \eqref{ine:det:1}, we have

\begin{align}
\label{ine:det:2}
&\sum_{i=1}^n\frac{\lambda_o}{\sqrt{n}}\left\{-h (r_{B^*+\Theta_\eta,i}) +h(r_{B^*,i})\right\}\langle X_i, \Theta_\eta\rangle +\eta \lambda_*\left(\|\hat{B}\|_*-\|B^*\|_*\right)\nonumber \\
 &\quad\quad\leq \sum_{i=1}^n\frac{\lambda_o}{\sqrt{n}}\eta \left\{-h (r_{B^*+\Theta_\eta,i}) +h(r_{B^*,i})\right\}\langle X_i, \hat{\Theta}\rangle+\eta \lambda_*\left(\|\hat{B}\|_*-\|B^*\|_*\right) \nonumber\\
 & \quad\quad\stackrel{(a)}{\leq}  \sum_{i=1}^n\frac{\lambda_o}{\sqrt{n}}\eta \left\{-h (r_{B^*+\Theta_\eta,i}) +h(r_{B^*,i})\right\}\langle X_i, \hat{\Theta}\rangle+\eta \lambda_*\langle \partial^* \hat{B},\hat{\Theta}\rangle \nonumber \\
 &\quad\quad= \sum_{i=1}^n\frac{\lambda_o}{\sqrt{n}}\eta h(r_{B^*,i})\langle X_i, \hat{\Theta}\rangle+\eta \left\{-\sum_{i=1}^n\frac{\lambda_o}{\sqrt{n}} h (r_{B^*+\Theta_\eta,i}) \langle X_i,\hat{\Theta}\rangle+ \lambda_*\langle \partial^* \hat{B},\hat{\Theta}\rangle\right\}\nonumber\\
  &\quad\quad\stackrel{(b)}{=} \sum_{i=1}^n\frac{\lambda_o}{\sqrt{n}}\eta h(r_{B^*,i})\langle X_i, \hat{\Theta}\rangle = \sum_{i=1}^n\frac{\lambda_o}{\sqrt{n}}  h(r_{B^*,i})\langle X_i,\Theta_\eta\rangle,
\end{align}
where (a) follows from $\|\hat{B}\|_*-\|B^*\|_* \leq \langle \partial^* \hat{B},\hat{\Theta}\rangle$, which is the definition of the sub-differential, and (b) follows from the fact that  $\hat{B}$ is an optimal solution of~\eqref{obj2-intro-m}.

From the convexity of the  Huber loss  and Lemma~\ref{ap:ine:conv}, the first term of the L.H.S. of \eqref{ine:det:2}  satisfies 
\begin{align}
	\label{ine:det:3}
	&\sum_{i=1}^n\frac{\lambda_o}{\sqrt{n}}\left\{-h (r_{B^*+\Theta_\eta,i}) +h(r_{B^*,i})\right\}\langle X_i, \Theta_\eta\rangle= \sum_{i=1}^n \left\{h(r_{B^*,i}) -h (r_{B^*+\Theta_\eta,i})\right\}\left\{h(r_{B^*,i})-h (r_{B_\eta,i})\right\}\geq 0.
\end{align}
The R.H.S. of \eqref{ine:det:2} can be decomposed as
\begin{align}
	\label{ine:det:4}
	\sum_{i=1}^n\frac{\lambda_o}{\sqrt{n}} h(r_{B^*,i})\langle X_i,\Theta_\eta\rangle &= \sum_{i \in I_I}\frac{\lambda_o}{\sqrt{n}} h(r_{B^*,i})\langle X_i,\Theta_\eta\rangle+\sum_{i \in I_o}\frac{\lambda_o}{\sqrt{n}} h(r_{B^*,i})\langle X_i, \Theta_\eta\rangle \nonumber\\
	&\quad  \quad= \sum_{i \in I_I}\frac{\lambda_o}{\sqrt{n}}  h\left( \frac{\xi_i}{\lambda_o \sqrt{n}}\right)\langle X_i, \Theta_\eta\rangle+\sum_{i \in I_o}\frac{\lambda_o}{\sqrt{n}} h(r_{B^*,i})\langle X_i, \Theta_\eta\rangle \nonumber\\
	&\quad  \quad= \sum_{i =1}^n\frac{\lambda_o}{\sqrt{n}}  h\left( \frac{\xi_i}{\lambda_o \sqrt{n}}\right)\langle X_i, \Theta_\eta\rangle+\sum_{i =1}^n\frac{\lambda_o}{\sqrt{n}} 
	w_i\langle X_i,\Theta_\eta\rangle,
\end{align}
where
\begin{align}
	w_i =
	\begin{cases}
		\ 0 & (i \in I_I) \\
		\ {\displaystyle h(r_{B^*,i})-h\left( \frac{\xi_i}{\lambda_o \sqrt{n}}\right)} & (i \in I_O)
	\end{cases}
	.
\end{align}
and $\vecw = (w_1,\cdots,w_n)^\top$.
From~\eqref{ine:det:2}, \eqref{ine:det:3} and \eqref{ine:det:4}, we have
\begin{align}
	\label{ine:det:5}
0\leq \sum_{i=1}^n\frac{\lambda_o}{\sqrt{n}} h\left(\frac{\xi_i}{\lambda_o\sqrt{n}} \right) \langle X_i, \Theta_\eta\rangle +\sum_{i=1}^n\frac{\lambda_o}{\sqrt{n}} w_i \langle X, \Theta_\eta\rangle+	\eta \lambda_*\left(\|B^*\|_*-\|\hat{B}\|_*\right).
\end{align}
Furthermore, we evaluate the right-hand side of~\eqref{ine:det:5}.
First, from~\eqref{ine:det:xis0},
\begin{align}
	\label{ine:det:6}
\sum_{i=1}^n\frac{\lambda_o}{\sqrt{n}} h\left(\frac{\xi_i}{\lambda_o\sqrt{n}}\right) \langle X_i, \Theta_\eta\rangle \leq a_{\mr{F}} r_{a,\mr{F}} \left\| \mr{T}_\Sigma (\Theta_\eta)\right\|_\mr{F} + a_*r_{a,*}\|\Theta_\eta\|_*.
\end{align}
Second, from~\eqref{ine:det:upper0}, we have
\begin{align}
	\label{ine:det:7}
\sum_{i=1}^n\frac{\lambda_o}{\sqrt{n}} w_i \langle X_i, \Theta_\eta\rangle& \leq  \sqrt{2}\left(b_\mr{F}r _{b,\mr{F}} \|\mr{T}_\Sigma (\Theta_\eta)\|_\mr{F} +b_*r_{b,*}\|\Theta_\eta\|_* \right).
\end{align}
From~\eqref{ine:det:5},~\eqref{ine:det:6}, \eqref{ine:det:7}  and the assumption $\|\mr{T}_\Sigma(\Theta_\eta) \|_\mr{F} \leq \frac{1}{c_\kappa\sqrt{r}} \|\Theta_\eta\|_*$, we have
\begin{align}
0 &\leq  a_{\mr{F}} r_{a,\mr{F}} \left\| \mr{T}_\Sigma (\Theta_\eta)\right\|_\mr{F} + a_*r_{a,*}\|\Theta_\eta\|_*+  \sqrt{2}\left(b_\mr{F}r_{b,\mr{F}} \|\mr{T}_\Sigma (\Theta_\eta)\|_\mr{F} +b_*r_{b,*}\|\Theta_\eta\|_* \right)+\eta \lambda_*\left(\|B^*\|_*-\|\hat{B}\|_*\right) \nonumber\\
& \leq C_s \|\Theta_\eta\|_*+\eta \lambda_*\left(\|B^*\|_*-\|\hat{B}\|_*\right).
\end{align}
Furthermore, we see
\begin{align}
0
&\leq C_s\|\Theta_\eta\|_*+\eta \lambda_*\left(\|B^*\|_*-\|\hat{B}\|_*\right) \nonumber \\
& = C_s \left(\|\mr{P}_{B^*} (\Theta_\eta)\|_*+\|\mr{P}_{B^*}^\bot (\Theta_\eta)\|_*\right)  +\eta \lambda_*(\|\mr{P}_{B^*} (B^*)\|_*+\|\mr{P}_{B^*}^\bot (B^*)\|_*-\|\mr{P}_{B^*} (\hat{B})\|_*-\|\mr{P}_{B^*}^\bot (\hat{B})\|_*)\nonumber\\
& =  C_s \left(\|\mr{P}_{B^*} (\Theta_\eta)\|_*+\|\mr{P}_{B^*}^\bot (\Theta_\eta)\|_*\right) + \eta\lambda_*(\|\mr{P}_{B^*} (B^*)\|_*-\|\mr{P}_{B^*} (\hat{B})\|_*-\|\mr{P}_{B^*}^\bot (\hat{B})\|_*)\nonumber\\
& \leq  C_s \left(\|\mr{P}_{B^*} (\Theta_\eta)\|_*+\|\mr{P}_{B^*}^\bot (\Theta_\eta)\|_*\right) + \eta\lambda_*(\|\mr{P}_{B^*} (B^*) - \mr{P}_{B^*} (\hat{B})\|_*-\|\mr{P}_{B^*}^\bot (\hat{B})\|_*)\nonumber\\
& = C_s\left(\|\mr{P}_{B^*} (\Theta_\eta)\|_*+\|\mr{P}_{B^*}^\bot (\Theta_\eta)\|_*\right) +\eta \lambda_*(\|\mr{P}_{B^*} (\Theta_{\hat{B}})\|_*-\|\mr{P}_{B^*}^\bot (\hat{B})\|_*)\nonumber\\
& =  C_s\left(\|\mr{P}_{B^*} (\Theta_\eta)\|_*+\|\mr{P}_{B^*}^\bot (\Theta_\eta)\|_*\right) +\eta \lambda_*(\|\mr{P}_{B^*} (\Theta_{\hat{B}})\|_*-\|\mr{P}_{B^*}^\bot (\hat{B}-B^*)\|_*)\nonumber\\
& =  C_s\left(\|\mr{P}_{B^*} (\Theta_\eta)\|_*+\|\mr{P}_{B^*}^\bot (\Theta_\eta)\|_*\right) + \lambda_*(\|\mr{P}_{B^*} (\Theta_\eta)\|_*-\|\mr{P}_{B^*}^\bot (\Theta_\eta)\|_*)\nonumber\\
& =\left(\lambda_* +  C_s\right)\|\mr{P}_{B^*} (\Theta_\eta)\|_*
+\left(-\lambda_*+  C_s\right)\|\mr{P}_{B^*}^\bot (\Theta_\eta)\|_*
\end{align}
and the proof is complete.

\end{proof}

Combining Lemma~\ref{l:rev} with Proposition~\ref{p:starMRE}, we can easily prove the following proposition, which shows a relation between  $\|\Theta_\eta\|_*$ and $\|\mr{T}_\Sigma(\Theta_\eta)\|_\mr{F}$. 

\begin{proposition}
	\label{p:coe-1-2-norm}
	Suppose the conditions used in Proposition \ref{p:starMRE}.  
	Then,  we obtain
	\begin{align}
		\label{e:l1l2}
		\|\Theta_\eta\|_* \leq c_\kappa \sqrt{r}\|\mr{T}_\Sigma (\Theta_\eta)\|_\mr{F}.
	\end{align}
\end{proposition}

\begin{proof}
When $\|\Theta_\eta\|_* < c_\kappa\sqrt{r}\|\mr{T}_\Sigma (\Theta_\eta)\|_\mr{F}$, we obtain \eqref{e:l1l2} immediately.
When $\|\Theta_\eta\|_* \geq c_\kappa\sqrt{r}\|\Theta_\eta\|_\mr{F}$,
from Proposition~\ref{p:starMRE}, we see that $\Theta_\eta$ satisfies $\|\mr{P}_{B^*} ^\bot (\Theta_\eta) \|_* \le c_0 \| \mr{P}_{B^*} (\Theta_\eta)\|_*$, that is, the condition \eqref{con:re-v}. Hence, because $\Sigma$ satisfies ${\rm MRE}(r,c_0,\kappa)$, we have the property \eqref{i:re-norm-1} with $\Theta_\eta$, so that we see $\|\Theta_\eta\|_*\le c_\kappa\sqrt{r}\| \mr{T}_\Sigma (\Theta_\eta)\|_{\mr{F}}$, and then the property \eqref{e:l1l2} holds. 
\end{proof}

\subsection{Main theorem}
\label{sec:main}

\begin{theorem}
	\label{t:det:main}
	Consider the optimization problem \eqref{obj2-intro-m}. 
	Assume all the conditions used in Proposition \ref{p:starMRE}. 
	Assume that
	\begin{align}
		\label{ine:det:lower}
		c_1 \| \mr{T}_\Sigma (\Theta_{\eta})\|_\mr{F}^2-c_2r_{c,\mr{F}}\|\mr{T}_\Sigma (\Theta_{\eta})\|_{\mr{F}} -c_3 r_{c}\leq \lambda_o^2\sum_{i=1}^n \left\{	-h \left(\frac{\xi_i-\langle X_i,\Theta_{\eta}\rangle}{\lambda_o\sqrt{n}} \right)+h  \left(\frac{\xi_i}{\lambda_o\sqrt{n}}  \right)\right\} \frac{\langle X_i,\Theta_{\eta}\rangle}{\lambda_o \sqrt{n}}
	\end{align}
	where  $c_1>0, c_2,c_3, r_{\mr{F},c},r_c \geq 0$ are some numbers.
Suppose that
		\begin{align}
		\label{ine:det:condc}
		\frac{c_2r_{c,\mr{F}} +C_{\lambda_*} +\sqrt{c_1c_3r_c}}{c_1}  < r_0,
	\end{align}
	 where
	\begin{align}
		C_{\lambda_*} = (a_\mr{F}r_{a,\mr{F}}+\sqrt{2}b_\mr{F}r_{b,\mr{F}}) + (a_*r_{a,*}+\sqrt{2}b_*r_{b,*})c_\kappa \sqrt{r} +\lambda_*  c_\kappa \sqrt{r}.
	\end{align}
	Then, the optimal solution $\hat{B}$ satisfies
	\begin{align}
	\label{ine:det:main-2}
		\|\mr{T}_\Sigma(\hat{B} -B^*)\|_\mr{F} & \leq r_0.
	\end{align}
\end{theorem}
\begin{proof}
We prove Theorem~\ref{t:det:main} in a manner similar to   the proof of Lemma B.7 in \cite{FanLiuSunZha2018Lamm} and the proof of Theorem 2.1 in \cite{CheZho2020Robust}.

For  fixed $r_0>0$, we define
\begin{align}
\mbb{B} :=\left\{ B\, :\, \| \mr{T}_\Sigma(B-B^*)\|_\mr{F} \leq r_0\right\}.
\end{align}
We prove $\hat{B} \in \mbb{B}$ by assuming $\hat{B} \notin \mbb{B}$ and deriving a contradiction. For $\hat{B} \notin \mbb{B}$, we can find some $\eta  \in [0,1]$ such that $\|\mr{T}_\Sigma(\Theta_\eta)\|_\mr{F}=r_0$.

From~\eqref{ine:det:2}, we have
\begin{align}
	\label{ine:det2:1}
	&\sum_{i=1}^n\frac{\lambda_o}{\sqrt{n}}\left\{-h(r_{B^*+\Theta_\eta,i}) +h(r_{B^*,i})\right\}\langle X_i,\Theta_\eta\rangle \leq \sum_{i=1}^n\frac{\lambda_o}{\sqrt{n}} h(r_{B^*,i})\langle X_i,\Theta_\eta\rangle+\eta \lambda_*\left(\|B^*\|_*-\|\hat{B}\|_*\right).
\end{align}

The L.H.S of~\eqref{ine:det2:1} can be decomposed as
\begin{align}
	\label{ine:det2:2}
& \sum_{i=1}^n\frac{\lambda_o}{\sqrt{n}}\left\{-h(r_{B^*+\Theta_\eta,i}) +h(r_{B^*,i})\right\}\langle X_i, \Theta_\eta\rangle \nonumber\\
&= \sum_{i \in I_O}\frac{\lambda_o}{\sqrt{n}}\left\{-h(r_{B^*+\Theta_\eta,i}) +h(r_{B^*,i})\right\}\langle X_i, \Theta_\eta\rangle+\sum_{i \in I_I}\frac{\lambda_o}{\sqrt{n}}\left\{-h(r_{B^*+\Theta_\eta,i}) +h(r_{B^*,i})\right\}\langle X_i, \Theta_\eta \rangle  \nonumber\\
&=  \sum_{i \in I_O}\frac{\lambda_o}{\sqrt{n}}\left\{-h(r_{B^*+\Theta_\eta,i}) +h(r_{B^*,i})\right\}\langle X_i, \Theta_\eta\rangle  +\sum_{i \in I_I}\frac{\lambda_o}{\sqrt{n}}\left\{-h\left(\frac{\xi_i-\langle X_i ,\Theta_\eta\rangle}{\lambda_o \sqrt{n}}\right) +h\left(\frac{\xi_i}{\lambda_o\sqrt{n}}\right)\right\}\langle X_i, \Theta_\eta \rangle  \nonumber\\
& =  \sum_{i \in I_O}\frac{\lambda_o}{\sqrt{n}}    \left\{-h(r_{B^*+\Theta_\eta,i}) +h(r_{B^*,i})+ h\left(\frac{\xi_i-\langle X_i ,\Theta_\eta\rangle}{\lambda_o \sqrt{n}}\right) -h\left(\frac{\xi_i}{\lambda_o\sqrt{n}}\right)\right\}  \langle X_i, \Theta_\eta\rangle  \nonumber\\
&+\sum_{i =1}^n\frac{\lambda_o}{\sqrt{n}}\left\{-h\left(\frac{\xi_i-\langle X_i ,\Theta_\eta\rangle}{\lambda_o \sqrt{n}}\right) +h\left(\frac{\xi_i}{\lambda_o\sqrt{n}}\right)\right\}\langle X_i, \Theta_\eta\rangle.
\end{align}

The first term of the R.H.S. of \eqref{ine:det2:1} can be decomposed as
\begin{align}
	\label{ine:det2:3}
	 \sum_{i=1}^n\frac{\lambda_o}{\sqrt{n}} h(r_{B^*,i})\langle X_i, \Theta_\eta\rangle &=\sum_{i \in I_O}\frac{\lambda_o}{\sqrt{n}} h(r_{B^*,i})\langle X_i, \Theta_\eta\rangle+\sum_{i \in I_u}\frac{\lambda_o}{\sqrt{n}} h(r_{B^*,i})\langle X_i, \Theta_\eta\rangle\nonumber\\
	 &=\sum_{i \in I_O}\frac{\lambda_o}{\sqrt{n}} h(r_{B^*,i})\langle X_i, \Theta_\eta\rangle+\sum_{i \in I_u}\frac{\lambda_o}{\sqrt{n}} h\left(\frac{\xi_i}{\lambda_o \sqrt{n}}\right)\langle X_i, \Theta_\eta\rangle\nonumber\\
		&=\sum_{i \in I_O}\frac{\lambda_o}{\sqrt{n}}  \left\{h(r_{B^*,i}) -h\left(\frac{\xi_i}{\lambda_o \sqrt{n}}\right)\right\}\langle X_i, \Theta_\eta\rangle+\sum_{i =1}^n\frac{\lambda_o}{\sqrt{n}} h\left(\frac{\xi_i}{\lambda_o \sqrt{n}}\right)\langle X_i, \Theta_\eta\rangle.
\end{align}
From \eqref{ine:det2:1}, \eqref{ine:det2:2} and \eqref{ine:det2:3}, we have
\begin{align}
	\label{ine:det2:4}
	&\sum_{i =1}^n\frac{\lambda_o}{\sqrt{n}}\left\{-h\left(\frac{\xi_i-\langle X_i ,\Theta_\eta\rangle}{\lambda_o \sqrt{n}}\right) +h\left(\frac{\xi_i}{\lambda_o\sqrt{n}}\right)\right\}\langle X_i, \Theta_\eta\rangle
	 \nonumber \\
	&\quad\quad\leq \sum_{i =1}^n\frac{\lambda_o}{\sqrt{n}} h\left(\frac{\xi_i}{\lambda_o \sqrt{n}}\right)\langle X_i, \Theta_\eta\rangle + \sum_{i =1}^n\frac{\lambda_o}{\sqrt{n}} w'_i\langle X_i, \Theta_\eta\rangle +\eta \lambda_*\left(\|B^*\|_*-\|\hat{B}\|_*\right),
\end{align}
where
\begin{align}
	w'_i =
	\begin{cases}
		\ 0 & (i \in I_I) \\
		\ {\displaystyle  h(r_{B^*+\Theta_\eta,i}) -h\left(\frac{\xi_i+\langle X_i ,\Theta_\eta\rangle}{\lambda_o \sqrt{n}}\right) } & (i \in I_O)
	\end{cases}
\end{align}
and $\vecw' = (w_1',\cdots,w_n)$.
We evaluate each term of \eqref{ine:det2:4}. From~\eqref{ine:det:lower}, the L.H.S. of~\eqref{ine:det2:4} is evaluated as
\begin{align}
	c_1 \|\mr{T}_\Sigma(\Theta_\eta)\|_\mr{F}^2 -c_2 r_{c,\mr{F}}\|\mr{T}_\Sigma\Theta_\eta\|_{\mr{F}}-c_3 r_{c} \leq \sum_{i=1}^n\frac{\lambda_o}{\sqrt{n}}\left\{-h \left(\frac{\xi_i-\langle X_i,\Theta_\eta\rangle}{\lambda_o\sqrt{n}} \right) +h\left(\frac{\xi_i}{\lambda_o\sqrt{n}} \right)\right\}\langle X, \Theta_\eta\rangle.
\end{align}
From~\eqref{ine:det:xis0} and~\eqref{e:l1l2} and Proposition \ref{p:coe-1-2-norm}, the first term of the R.H.S. of \eqref{ine:det2:4} is evaluated as
\begin{align}
 \sum_{i=1}^n\frac{\lambda_o}{\sqrt{n}} h \left(\frac{\xi_i}{\lambda_o\sqrt{n}} \right) \langle X_i, \Theta_\eta\rangle &\leq a_\mr{F}r_{a,\mr{F}}\|\mr{T}_\Sigma (\Theta_\eta)\|_\mr{F} +a_*r_{a,*}\|\Theta_\eta\|_*\nonumber\\
 &\leq \left(
	a_\mr{F}r_{a,\mr{F}} +a_*c_\kappa \sqrt{r} r_{a,*}\right)\|\mr{T}_\Sigma (\Theta_\eta)\|_\mr{F}.
\end{align}
From~\eqref{ine:det:upper0} and \eqref{e:l1l2} and Proposition \ref{p:coe-1-2-norm}, the second term of the R.H.S. of \eqref{ine:det2:4} is evaluated as
\begin{align}
	\label{ine:branching1}
 \sum_{i=1}^n\frac{\lambda_o}{\sqrt{n}} w'_i \langle X_i, \Theta_\eta\rangle& \leq \left(b_\mr{F}\frac{r_{b,\mr{F}}}{\sqrt{o}}\|\mr{T}_\Sigma (\Theta_\eta)\|_\mr{F} +b_*\frac{r_{b*}}{\sqrt{o}}\|\Theta_\eta\|_* \right)\|\vecw'\|_2\nonumber\\
 &\leq \sqrt{2}\left(
	b_\mr{F} r_{b,\mr{F}}\|\mr{T}_\Sigma (\Theta_\eta)\|_\mr{F}+b_*c_\kappa \sqrt{r}r_{b,*}\|\mr{T}_\Sigma (\Theta_\eta)\|_\mr{F}\right).
\end{align}
From~\eqref{e:l1l2} and Proposition \ref{p:coe-1-2-norm}, the third term of the R.H.S. of \eqref{ine:det2:4} is evaluated as
\begin{align}
\eta \lambda_*\left(\|B^*\|_*-\|\hat{B}\|_*\right) \leq \lambda_* \|\Theta_\eta\|_* \leq \lambda_* c_\kappa \sqrt{r} \|\mr{T}_\Sigma (\Theta_\eta)\|_\mr{F}.
\end{align}
Combining the above four inequalities  with \eqref{ine:det2:4}, we have
\begin{align}
	\label{ine:det:quad}
	c_1 \|\mr{T}_\Sigma(\Theta_\eta)\|_\mr{F}^2 -c_2 r_{c,\mr{F}}\|\mr{T}_\Sigma(\Theta_\eta)\|_{\mr{F}}-c_3 r_c \leq  C_{\lambda_*}\|\mr{T}_\Sigma(\Theta_\eta)\|_\mr{F}.
 \end{align}
From \eqref{ine:det:quad}, $\sqrt{A+B}  \leq \sqrt{A}+\sqrt{B}$ for $A,B>0$, we have
\begin{align}
	\|\mr{T}_\Sigma(\Theta_\eta)\|_\mr{F} &\leq \frac{c_2r_{c,\mr{F}} +C_{\lambda_*} +\sqrt{c_1c_3r_c}}{c_1} < r_0.
\end{align}
This contradicts $\|\mr{T}_\Sigma(\Theta_\eta)\|_\mr{F} = r_0$.
Consequently, we have $\hat{B} \in \mbb{B}$ and $\|\mr{T}_\Sigma(\hat{B}-B^*)\|_\mr{F} < r_0$.

\end{proof}

\section{Tools for proving Theorem  \ref{t:cs:main}}
\label{sec:inesforcs}
In this section, suppose that Assumption \ref{a:mcs} holds.
\subsection{Derivation of \eqref{ine:det:xis0} under the assumptions of Theorem \ref{t:cs:main}}
\label{sec:a:cs:xis0}

\begin{lemma}
	\label{l:suphx}
	For $0<\delta<1/7$, with probability at least $1-\delta$, we have 
	\begin{align}
		\label{ine:suphx}
	\left|\frac{1}{n} \sum_{i=1}^n  h\left(\frac{\xi_i}{\lambda_o \sqrt{n}}\right) \langle X_i,M\rangle  \right|\leq C L  \left\{\rho\sqrt{\frac{d_1+d_2}{n}}\|M\|_{*}+\sqrt{\frac{\log(1/\delta)}{n}}\|\mr{T}_\Sigma(M)\|_{\mr{F}}\right\}.
	\end{align}
\end{lemma}
\begin{proof}
	Let 
	\begin{align}
		V_M = \left\{ M \in \mbb{R}^{d_1 \times d_2}\, |\,  \|\mr{T}_\Sigma \left(M\right)\|_\mr{F} =1,\, \|M\|_* \leq r_{*}\right\}.
	\end{align} 
	For any $M,M'\in \mbb{R}^{d_1\times d_2}$, we have
	\begin{align}
		\label{mcssub}
		\left\|h\left(\frac{\xi_i}{\lambda_o \sqrt{n}}\right) \langle X_i, M\rangle-h\left(\frac{\xi_i}{\lambda_o \sqrt{n}}\right) \langle X_i, M'\rangle\right\|_{\psi_2} \leq \|\langle X_i, M\rangle-\langle X_i, M'\rangle\|_{\psi_2}
	\end{align}
	because $\left|h\left(\frac{\xi_i}{\lambda_o \sqrt{n}}\right)\right|\leq 1$ and we see that $h\left(\frac{\xi_i}{\lambda_o \sqrt{n}}\right) \langle X_i, M\rangle$ is a $L$-subGaussian distribution.
	From \eqref{mcssub} and the fact that $\left\{h\left(\frac{\xi_i}{\lambda_o \sqrt{n}}\right) \langle X_i, M\rangle\right\}_{i=1}^n$ is a sequence of i.i.d. random variables, for any $M, M' \in V_M$, we have
	\begin{align}
		\left\|\frac{1}{n} \sum_{i=1}^n  h\left(\frac{\xi_i}{\lambda_o \sqrt{n}}\right)  \langle Z_i,\mr{T}_\Sigma(M)\rangle-\frac{1}{n} \sum_{i=1}^n  h\left(\frac{\xi_i}{\lambda_o \sqrt{n}}\right)  \langle Z_i,\mr{T}_\Sigma(M')\rangle  \right\|_{\psi_2}\leq L\frac{1}{\sqrt{n}}\left\|\mr{T}_\Sigma(M) -\mr{T}_\Sigma(M') \right\|_2.
	\end{align}
	and from Exercise 8.6.5 of \cite{Ver2018High}, with probability at least $1-\delta$, 
	we have
	\begin{align}
		\label{gc-pre}
		\sup_{M \in V_M}\left|\frac{1}{n} \sum_{i=1}^n   h\left(\frac{\xi_i}{\lambda_o \sqrt{n}}\right) \langle Z_i,\mr{T}_\Sigma(M)\rangle  \right| \leq C\frac{L}{\sqrt{n}}\left\{\mbb{E}\sup_{M \in V_M} \langle Z'_i,\mr{T}_{\Sigma}(M)\rangle+\sqrt{\log(1/\delta)}\sup_{M \in V_M}\sqrt{ \mbb{E}\langle Z'_i, \mr{T}_{\Sigma}(M)\rangle^2}\right\},
	\end{align}
	where $Z_i'$ is the standard normal Gaussian random matrix. Define $X_i' = \mr{T}_\Sigma(Z_i')$ and combining the arguments above, we have,
	with probability at least $1-\delta$,
	\begin{align}
		\sup_{M \in V_M}\left|\frac{1}{n} \sum_{i=1}^n  h\left(\frac{\xi_i}{\lambda_o \sqrt{n}}\right) \langle X_i,M\rangle  \right| &\leq CL\frac{1}{\sqrt{n}}\left\{\mbb{E}\sup_{M \in V_M} \langle Z'_i,\mr{T}_{\Sigma}(M)\rangle+\sqrt{\log(1/\delta)}\sup_{M \in V_M}\sqrt{ \mbb{E}\langle Z'_i, \mr{T}_{\Sigma}(M)\rangle^2}\right\}\nonumber\\
		&\stackrel{(a)}{\leq} CL\frac{1}{\sqrt{n}}\left\{\mbb{E}\|X_i'\|_{\mr{op}}\|M\|_{*} +\sqrt{\log(1/\delta)}\sup_{M \in V_M}\sqrt{ \mbb{E}\langle Z'_i, \mr{T}_{\Sigma}(M)\rangle^2}\right\}\nonumber\\
		& \stackrel{(b)}{\leq} CL\frac{1}{\sqrt{n}}\left\{\rho\sqrt{d_1+d_2}r_{*}+\sqrt{\log(1/\delta)}\right\},
	\end{align}
	where (a) follows from H{\"o}lder's inequality and (b) follows from Lemma H.1 of \cite{NegWai2011Estimation}.
	Lastly, using peeling device (Lemma 5 of \cite{DalTho2019Outlier}), the proof is complete.
\end{proof}

\subsection{Derivation of \eqref{ine:det:upper0} under the assumptions of Theorem \ref{t:cs:main}}
\label{sec:a:cs:upper0}
In Section \ref{sec:a:cs:upper0}, we show Corollary \ref{ac:|uMv|-cs}, which implies that \eqref{ine:det:upper0} is satisfied with high probability under the assumptions of Theorem \ref{t:cs:main}.
Combining Proposition 4 of \cite{Tho2020Outlier} and Remark 4 of \cite{DalTho2019Outlier}, we easily have the following corollary.
\begin{corollary}[Corollary of Proposition 4 of \cite{Tho2020Outlier}]
	\label{ac:|uMv|-cs}
	We have for all $o$-sparse vector $\vecu \in \mbb{R}^n$ and $M \in \mbb{R}^{d_1 \times d_2}$,
	\begin{align}
	\left| \sum_{i=1}^n u_i \frac{1}{\sqrt{n}} \langle X_i, M \rangle\right| & \leq CL\left(\frac{1+\sqrt{\log (1/\delta)}}{\sqrt{n}}\left\| \mr{T}_\Sigma (M)\right\|_\mr{F} + \rho \sqrt{\frac{d_1+d_2}{n}}\|M\|_* +   \sqrt{\frac{o}{n}\log \frac{n}{o}}\| \mr{T}_\Sigma (M)\|_\mr{F}\right) \|\vecu\|_2
	\end{align}
	with probability at least $1-\delta$.
\end{corollary}
\begin{proof}
	From Proposition 4 of \cite{Tho2020Outlier}, for any $\vecu \in\mbb{R}^n$ and $M \in \mbb{R}^{d_1\times d_2}$, we have
	\begin{align}
		&\left| \sum_{i=1}^n u_i \frac{1}{\sqrt{n}} \langle X_i, M \rangle\right| \nonumber \\
		& \leq CL\left(\frac{1+\sqrt{\log (1/\delta)}}{\sqrt{n}}\left\| \mr{T}_\Sigma (M)\right\|_\mr{F}\|\vecu\|_2 + \frac{\|Z'\|_{\mr{op}}}{\sqrt{n}}\|M\|_* +   \| \mr{T}_\Sigma (M)\|_\mr{F} \sqrt{\log \frac{n}{o}} 
		\sup_{\vecb \in \mbb{B}_1^n \cap \mbb{B}_2^n/\|\vecu\|_1} \vecg_d^\top \vecb\right),
	\end{align}
	where $\vecg_d$ is the $d$-dimensional standard Gaussian vector $Z'\in \mbb{R}^{d_1 \times d_2}$ is a random matrix whose entries are standard normal Gaussian.
	From Lemma H.1 of \cite{NegWai2011Estimation}
	\begin{align}
		\label{ine:t20001}
		\left\|Z'\right\|_{\mr{op}} \leq \rho \sqrt{d_1+d_2}
	\end{align}
	When $\vecu$ is $o$-sparse, from Remark 4 of \cite{DalTho2019Outlier}, we have
	\begin{align}
		\label{ine:t20002}
		\sup_{\vecb \in \mbb{B}_1^n \cap \mbb{B}_2^n/\|\vecu\|_1} \vecg_d^\top \vecb \leq C \sqrt{o}\|\vecu\|_2.
	\end{align}
	Combining above arguments, the proof is complete.
\end{proof}

\subsection{Derivation of  \eqref{ine:det:lower} under the assumptions of Theorem \ref{t:cs:main}}
\label{sec:a:cs:lower}
In Section \ref{sec:a:cs:lower}, we show Proposition \ref{p:sc1}, which implies that \eqref{ine:det:lower} is satisfied with high probability under the assumptions of Theorem \ref{t:cs:main}.
 This is partly proved in a similar manner to \cite{Loh2017Statistical}, \cite{FanLiWan2017Estimation}, \cite{CheZho2020Robust} and  \cite{SunZhoFan2020Adaptive}.

\begin{proposition}
	\label{p:sc1}
	Let  
	\begin{align}
	\mc{R}_{mcs} = \left\{ \Theta \in \mbb{R}^{d_1 \times d_2}\, |\,  \|\Theta\|_*\leq c_\kappa \|\mr{T}_\Sigma (\Theta) \|_\mr{F},\, \|\mr{T}_\Sigma (\Theta) \|_\mr{F} =r_{mcs} \right\},
	\end{align}
	where $r_{mcs}$ is some number such that $0\leq  r_{mcs} \leq \frac{1}{4\sqrt{3} L^2}$.
	Assume that $\lambda_o  \sqrt{n} \geq 72L^4\sigma$.
	Then, with probability at least $1-\delta$, we have
	\begin{align}
		& \inf_{\Theta \in \mc{R}_{mcs}}\left[\sum_{i=1}^n \lambda_o^2 \left\{-h\left(\frac{\xi_i + \langle  X_i, \Theta \rangle}{\lambda_o\sqrt{n}}\right)-h \left(\frac{\xi_i}{\lambda_o\sqrt{n}}\right) \right\}\frac{\langle   X_i,\Theta\rangle }{\lambda_o \sqrt{n}}\right]\\
		&\geq \frac{1}{3}\|\mr{T}_\Sigma(\Theta)\|_{\mr{F}}^2-C\left(L\rho c_\kappa \sqrt{r\frac{d_1+d_2}{n}}+\sqrt{\frac{\log(1/\delta)}{n}} \right)\|\mr{T}_\Sigma (\Theta)\|_{\mr{F}} -C\frac{\log(1/\delta)}{n}.
\end{align} 
\end{proposition}

\begin{proof}
	Let
	\begin{align}
		u_i = \frac{\xi_i}{\lambda_o \sqrt{n}}, \quad v_i = \frac{\langle X_i,\Theta \rangle}{\lambda_o \sqrt{n}}.
	\end{align}
 The left-hand side of \eqref{ine:sc1-1} divided by $\lambda_o^2$ can be expressed as
\begin{align}
	\sum_{i=1}^n  \left\{-h\left(\frac{\xi_i - \langle  X_i, \Theta \rangle}{\lambda_o\sqrt{n}}\right)+h \left(\frac{\xi_i}{\lambda_o\sqrt{n}}\right) \right\}\frac{\langle  X_i,\Theta\rangle}{\lambda_o \sqrt{n}} = \sum_{i=1}^n  \left\{-h\left(u_i-v_i\right)+h \left(u_i\right) \right\}v_i
\end{align}
	From the convexity of the Huber loss and Lemma \ref{ap:ine:conv},  we have
	\begin{align}
	&\sum_{i=1}^n  \left\{-h\left(u_i-v_i\right)+h \left(u_i\right) \right\}v_i \geq \sum_{i=1}^n  \left\{-h\left(u_i-v_i\right)+h \left(u_i\right) \right\}v_i\mr{I}_{E_i},
	\end{align}
	where $\mr{I}_{E_i}$ is the indicator function of the event
	\begin{align}
	\mr{I}_{E_i} :=  \left( \left|u_i\right | \leq \frac{1}{2} \right)  \cap \left(   \left |v_i \right |  \leq  \frac{1}{2\lambda_o\sqrt{n}}   \right).
	\end{align}
	Define the functions
	\begin{align}
	\label{def:phipsi}
	\varphi(v) =\begin{cases}
	v^2   &  \mbox{ if }  |v | \leq  \frac{1}{2\lambda_o\sqrt{n}}\\
	(v-1/2)^2   &  \mbox{ if }  \frac{1}{2\lambda_o\sqrt{n}}\leq v  \leq  1/2 \\
		(v+1/2)^2   &  \mbox{ if }  -1/2\leq v  \leq  -\frac{1}{2\lambda_o\sqrt{n}}   \\
	0 & \mbox{ if } |v| >1/2
	\end{cases} ~\mbox{ and }~
	\psi(u) = I_{(|u| \leq 1/2 ) }.
	\end{align}
Let 
	\begin{align}
 \sum_{i=1}^n \varphi \left(v_i\right) \psi \left(u_i\right) = \sum_{i=1}^n f_i(\Theta)
\end{align}
with  $f_i(\Theta) = \varphi(v_i) \psi(u_i)$ and we have
	\begin{align}
		\label{ine:huv-conv-f}
\sum_{i=1}^n  \left\{-h\left(u_i-v_i\right)+h \left(u_i\right) \right\}v_i &\geq \sum_{i=1}^n  \left\{-h\left(u_i-v_i\right)+h \left(u_i\right) \right\}v_i\mr{I}_{E_i}\nonumber\\
&= \sum_{i=1}^n  v_i^2\mr{I}_{E_i}\nonumber\\
& \stackrel{(a)}{\geq} \sum_{i=1}^n  \varphi(v_i) \psi(u_i)=\sum_{i=1}^n f_i(\Theta),
\end{align}
where (a) follows from $\varphi(v) \geq v^2$ for $|v| \leq 1/2$.
	We note that 
\begin{align}
	\label{ine:f-1/4}
	f_i(\Theta) \leq\varphi(v_i) \leq \max\left(\frac{\langle X_i ,\Theta\rangle^2}{\lambda_o^2n},\frac{1}{4\lambda_o^2n}\right).
\end{align}
To bound $\sum_{i=1}^n f_i(\Theta)$ from bellow, for any fixed $ \Theta \in  \mc{R}(r_{mcs})$, we have
\begin{align}
\label{ine:fbelow}
\sum_{i=1}^n f_i(\Theta)&\geq \mbb{E}f(\Theta) -\sup_{\Theta' \in \mc{R}_{mcs}}  \Big|\sum_{i=1}^n f_i(\Theta')-\mbb{E}\sum_{i=1}^n f_i(\Theta')\Big|.
\end{align}
Define the supremum of a random process indexed by $\mc{R}_{mcs}$:
\begin{align}
\label{ap:delta}
\Delta  :=  \sup_{ \Theta' \in \mc{R}_{mcs}} | \sum_{i=1}^n f_i(\Theta') - \mbb{E}\sum_{i=1}^n f_i(\Theta') | .  
\end{align}
From \eqref{ine:huv-conv-f} and \eqref{def:phipsi}, we have
\begin{align}
\label{ine:aplower:tmp}
\mbb{E}\sum_{i=1}^n f_i(L)&\geq \sum_{i=1}^n\mbb{E} \left|\frac{  \langle  X_i,\Theta \rangle}{\lambda_o\sqrt{n}} \right| ^2 - \sum_{i=1}^n\mbb{E}\left|\frac{  \langle  X_i,\Theta \rangle}{\lambda_o\sqrt{n}} \right| ^2 I \bigg( \left|\frac{  \langle  X_i,\Theta \rangle}{\lambda_o\sqrt{n}} \right| \geq \frac{1}{2\lambda_o\sqrt{n}}   \bigg) -  \mbb{E}\left|\frac{  \langle  X_i,\Theta  \rangle}{\lambda_o\sqrt{n}} \right| ^2 I\left( \left|\frac{\xi_i}{\lambda_o\sqrt{n}} \right|\geq \frac{1}{2} \right).
\end{align}
	We evaluate the right-hand side of \eqref{ine:aplower:tmp} at each term.
	First, we have
	\begin{align}
	\label{ap:ine:cov1}
	\sum_{i=1}^n\mbb{E} \left|\frac{  \langle  X_i,\Theta  \rangle}{\lambda_o\sqrt{n}} \right| ^2 I \bigg( \left|\frac{  \langle  X_i,\Theta   \rangle}{\lambda_o\sqrt{n}} \right| \geq \frac{1}{2\lambda_o\sqrt{n}}   \bigg) 
	&\stackrel{(a)}{\leq} 	\sum_{i=1}^n\sqrt{\mbb{E} \left|\frac{  \langle  X_i,\Theta \rangle}{\lambda_o\sqrt{n}} \right| ^4 } \sqrt{\mbb{E}   \ I \bigg( \left|\frac{  \langle  X_i,\Theta \rangle}{\lambda_o\sqrt{n}} \right| \geq \frac{1}{2\lambda_o\sqrt{n}}   \bigg) }\nonumber\\
	&\stackrel{(b)}{=} 	\sum_{i=1}^n\sqrt{\mbb{E}\left|\frac{  \langle  X_i,\Theta \rangle}{\lambda_o\sqrt{n}} \right| ^4 } \sqrt{\mbb{P}  \left(  \left|\langle  X_i,\Theta \rangle \right| \geq \frac{1}{2}     \right) }\nonumber\\
	&\stackrel{(c)}{\leq}	\sum_{i=1}^n4\sqrt{\mbb{E} \left|\frac{  \langle  X_i,\Theta \rangle}{\lambda_o\sqrt{n}} \right| ^4 } \sqrt{ \mbb{E}\langle  X_i,\Theta  \rangle^4}\nonumber\\
	&\stackrel{(d)}{\leq}\frac{4L^2}{\lambda_o^2} r_{mcs}^4\nonumber\\
	&\stackrel{(e)}{\leq} \frac{1}{3\lambda_o^2 }	\|\mr{T}_\Sigma(\Theta)\|_{\mr{F}}^2,
	\end{align}
	where (a) follows from H{\"o}lder's inequality, (b) follows from the relation between indicator function and expectation, (c) follows from Markov's inequality, (d) follows from Assumption \ref{a:mcs} and  (e) follows from $0\leq r_{mcs} \leq \frac{1}{4\sqrt{3}L^2}$.
	Second,  we have
	\begin{align}
	\label{ap:ine:cov2}
	\sum_{i=1}^n\mbb{E} \left|\frac{  \langle  X_i,\Theta  \rangle}{\lambda_o\sqrt{n}} \right| ^2 I\left( \left|\frac{\xi_i}{\lambda_o\sqrt{n}} \right|\geq \frac{1}{2} \right) 
	&\stackrel{(a)}{\leq} \sum_{i=1}^n\sqrt{\mbb{E} \left|\frac{  \langle  X_i,\Theta \rangle}{\lambda_o\sqrt{n}} \right| ^4}  \sqrt{\mbb{E}I\left( \left|\frac{\xi_i}{\lambda_o\sqrt{n}} \right|\geq \frac{1}{2} \right)}\nonumber\\
	&\stackrel{(b)}{\leq}\sum_{i=1}^n\sqrt{\mbb{E}\left|\frac{  \langle  X_i,\Theta \rangle}{\lambda_o\sqrt{n}} \right| ^4}  \sqrt{\mbb{P}  \left( \left|\frac{\xi_i}{\lambda_o\sqrt{n}} \right|\geq \frac{1}{2}  \right)}\nonumber\\	  
	&\stackrel{(c)}{\leq}\sum_{i=1}^n \sqrt{\frac{2}{\lambda_o\sqrt{n}}}\sqrt{\mbb{E} \left|\frac{  \langle  X_i,\Theta \rangle}{\lambda_o\sqrt{n}} \right| ^4}  \sqrt{\mbb{E}|\xi_i|}\nonumber\\	  
	&\stackrel{(d)}{\leq}\sum_{i=1}^n \sqrt{ \frac{2\sigma}{\lambda_o\sqrt{n}} }\frac{2L^2}{\lambda_o^2n}  r_{mcs}^2 \nonumber\\
	&\stackrel{(e)}{\leq} \frac{1}{3\lambda_o^2} 	\|\mr{T}_\Sigma(\Theta)\|_{\mr{F}}^2,
	\end{align}
		where (a) follows from H{\"o}lder's inequality, (b) follows from relation between indicator function and expectation, (c) follows from Markov's inequality, (d) follows from Assumption \ref{a:mcs} and  (e) follows from the definition of $\lambda_o$.
	Consequently, we have
	\begin{align}
	\label{ap:f_bellow}
	\mbb{E}\sum_{i=1}^nf_i(\Theta)&\geq \frac{1}{3\lambda_o^2}	\|\mr{T}_\Sigma(\Theta)\|_{\mr{F}}^2
	\end{align}
	and
	\begin{align}
	\label{ap:h_bellow}
	\sum_{i=1}^n  \left\{-h\left(u_i-v_i\right)  -h \left(u_i\right) \right\}v_i  \geq \sum_{i=1}^nf_i(\Theta) \geq   \frac{1}{3\lambda_o^2}	\|\mr{T}_\Sigma(\Theta)\|_{\mr{F}}^2-\Delta.
	\end{align}
	Next we evaluate the stochastic term $\Delta$ defined in \eqref{ap:delta}. 
	From \eqref{ine:f-1/4} and Theorem 3 of \cite{Mas2000Constants}, with probability at least $1-\delta$, we have
	\begin{align}
	\label{ine:delta}
	\Delta & \leq 2 \mbb{E} \Delta + \sigma_f \sqrt{8\log(1/\delta)} + \frac{18.5}{4} \log(1/\delta)\leq 2 \mbb{E} \Delta + \sigma_f \sqrt{8\log(1/\delta)} + 5\frac{\log(1/\delta)}{\lambda_o^2 n},
	\end{align}
	where $\sigma^2_f= \sup_{ \Theta \in \mc{R}_{mcs}} \sum_{i=1}^n\mbb{E}  \{f_i(\Theta)-\mbb{E}f_i(\Theta)\}^2$.
	About $\sigma_f$, we have
	\begin{align}
		\mbb{E}\{f_i(\Theta)-\mbb{E}f_i(\Theta)\}^2 \leq \mbb{E}f_i^2(\Theta) .
	\end{align}
	From  \eqref{ine:f-1/4} and $0< r \leq \frac{1}{4\sqrt{3} L^2}$, we have
	\begin{align}
	\label{ap:ine:cov3}
	&\mbb{E}f_i^2(\Theta)\leq \mbb{E} \frac{\langle X_i ,\Theta\rangle^4}{\lambda_o^4n^2}= \frac{1}{\lambda_o^4n^2} \|\mr{T}_\Sigma(\Theta)\|_{\mr{F}}^4.
	\end{align}
	Combining this and \eqref{ine:delta}, we have
	\begin{align}
	\label{ap:delta_upper}
	\Delta \leq 2 \mbb{E} \Delta+ \frac{1}{\lambda_o^2}\sqrt{6\frac{\log(1/\delta)}{n}}\|\mr{T}_\Sigma(\Theta)\|_{\mr{F}}^2+ 5\frac{\log(1/\delta)}{\lambda_o^2 n}.
	\end{align}
	From Symmetrization inequality (Lemma 11.4 of \cite{BouLugMas2013concentration}), we have  $\mbb{E}\Delta \leq 2   \,\mbb{E} \sup_{ \Theta \in \mc{R}_{mcs} } |  \mathbb{G}_{\Theta} |  $,
	where 
	\begin{align}
	\mbb{G}_{\Theta} := \sum_{i=1}^n 
	\varrho_i \varphi \left( \frac{\langle X_i ,\Theta \rangle}{\lambda_o\sqrt{n}} \right) \psi \left(\frac{\xi_i}{\lambda_o\sqrt{n}} \right),
	\end{align} 
	and $\{\varrho_i\}_{i=1}^n$ is a sequence of i.i.d. Rademacher random variables which is independent of $\{X_i,\xi_i\}_{i=1}^n$.
	We   denote $\mbb{E}^*$ as a conditional variance of $\left\{\varrho_i\right\}_{i=1}^n$ given $\left\{X_i,\xi_i\right\}_{i=1}^n$. From contraction principal (Theorem 11.5 of \cite{BouLugMas2013concentration}),  we  have
	\begin{align}
		&\mbb{E} ^*\sup_{\Theta\in\mc{R}_{mcs}} \left|     \sum_{i=1}^n \varrho_i \varphi \left( \frac{\langle X_i, \Theta\rangle}{\lambda_o\sqrt{n}} \right) \psi \left(\frac{\xi_i}{\lambda_o\sqrt{n}} \right)    \right| \leq	\mbb{E}^* \sup_{\Theta\in\mc{R}_{mcs}} \left|  \sum_{i=1}^n   \varrho_i \varphi \left( \frac{\langle X_i ,\Theta\rangle}{\lambda_o\sqrt{n}} \right) \right|
	\end{align}	
and from the basic property of the expectation, we have
	\begin{align}
\mbb{E}\sup_{\Theta\in\mc{R}_{mcs}}  \left|     \sum_{i=1}^n \varrho_i \varphi \left( \frac{\langle X_i, \Theta\rangle}{\lambda_o\sqrt{n}} \right) \psi \left(\frac{\xi_i}{\lambda_o\sqrt{n}} \right)    \right| &\leq	
	\mbb{E}\sup_{\Theta\in\mc{R}_{mcs}}  \left|  \sum_{i=1}^n   \varrho_i \varphi\left( \frac{\langle X_i , \Theta\rangle}{\lambda_o\sqrt{n}} \right) \right|.
\end{align}	
Since $\varphi$ is $\frac{1}{2\lambda_o\sqrt{n}}$-Lipschitz and $\varphi(0)=0$,  from contraction principal (Theorem 11.6 in \cite{BouLugMas2013concentration}), we have
	\begin{align}
			\mbb{E} \sup_{\Theta\in\mc{R}_{mcs}} \left|\sum_{i=1}^n \varrho_i \varphi \left( \frac{\langle X_i ,\Theta \rangle}{\lambda_o\sqrt{n}} \right)\right|&\leq	\mbb{E}\sup_{\Theta\in\mc{R}_{mcs}}  \left|  \sum_{i=1}^n   \varrho_i  \frac{\langle X_i , \Theta\rangle}{2\lambda_o^2n}  \right|.
	\end{align}
From Lemma \ref{l:supRad} and the definition of $\mc{R}_{mcs}$, we have
\begin{align}
	\label{ine:hub-stoc-upper}
	2\lambda_o^2\mbb{E} \Delta \leq CL\rho\sqrt{\frac{d_1+d_2}{n}}\|\Theta\|_* \leq CL\rho c_\kappa \sqrt{r\frac{d_1+d_2}{n}}\|\mr{T}_\Sigma (\Theta)\|_{\mr{F}}.
\end{align}
	Combining \eqref{ine:hub-stoc-upper}  with \eqref{ap:delta_upper}  and \eqref{ap:h_bellow}, with probability at least $1-\delta$, we have
	\begin{align}
	&\sum_{i=1}^n \lambda_o^2 \left\{-h\left(\frac{\xi_i + \langle  X_i, \Theta \rangle}{\lambda_o\sqrt{n}}\right)-h \left(\frac{\xi_i}{\lambda_o\sqrt{n}}\right) \right\}\frac{\langle   X_i,\Theta\rangle }{\lambda_o \sqrt{n}}\nonumber\\
	&\geq \frac{1}{3}\|\mr{T}_\Sigma(\Theta)\|_{\mr{F}}^2-C\left(L\rho c_\kappa \sqrt{r\frac{d_1+d_2}{n}}+\sqrt{8\frac{\log(1/\delta)}{n}} \right)\|\mr{T}_\Sigma (\Theta)\|_{\mr{F}} -5\frac{\log(1/\delta)}{n}.
	\end{align} 
\end{proof}

To caluclate the L.H.S. of \eqref{ine:hub-stoc-upper}, we introduce the following Lemma.
\begin{lemma}
	\label{l:supRad}
	Assume that $\mc{R}_{mcs}$ is the same of the one in Proposition \ref{p:sc1}.
	We have 
	\begin{align}
		\mbb{E}\sup_{\Theta\in\mc{R}_{mcs}}  \left|  \sum_{i=1}^n   \varrho_i  \frac{\langle X_i , \Theta\rangle}{n}  \right| &\leq CL\rho\sqrt{r\frac{d_1+d_2}{n}}\|\Theta\|_\mr{F}
\end{align}
\end{lemma}
\begin{proof}
	From H{\"o}lder's inequality	
	\begin{align}
		\mbb{E}\sup_{\Theta\in\mc{R}_{mcs}}  \left|  \sum_{i=1}^n   \varrho_i  \frac{X_i}{n}  \right| \leq \mbb{E}\left\|  \sum_{i=1}^n   \varrho_i  \frac{X_i}{n}  \right\|_{\mr{op}}\sup_{\Theta\in\mc{R}_{mcs}} \|\Theta\|_*\leq \mbb{E}\left\|  \sum_{i=1}^n   \varrho_i  \frac{\langle X_i , \Theta\rangle}{n}  \right\|_{\mr{op}}\sqrt{r}\|\Theta\|_{\mr{F}}.
	\end{align}
	We calculate $\mbb{E}\left\|  \sum_{i=1}^n   \varrho_i  \frac{ X_i}{n}  \right\|_{\mr{op}}$.
	For any $\Theta,\Theta'\in \mbb{R}^{d_1\times d_2}$, we have
	\begin{align}
		\label{mcssub2}
		\left\|\varrho_i \langle X_i, M\rangle-\varrho_i  \langle X_i, M'\rangle\right\|_{\psi_2} \leq \|\langle X_i, M\rangle-\langle X_i, M'\rangle\|_{\psi_2}
	\end{align}
	because $|\varrho_i |= 1$ and we see that $\varrho_i \langle X_i, \Theta\rangle$ is a $L$-subGaussian distribution.
	From \eqref{mcssub} and the face that $\left\{\varrho_i \langle X_i, \Theta\rangle\right\}_{i=1}^n$ is a sequence of i.i.d. random variables, we have
	\begin{align}
		\left\|\frac {1}{n}\sum_{i=1}^n\varrho_i\langle X_i , \Theta'\rangle-\frac {1}{n}\sum_{i=1}^n\varrho_i\langle X_i , \Theta^{''}\rangle\right\|_{\psi_2} \leq L\frac{1}{\sqrt{n}}\|\mr{T}_{\Sigma}(\Theta')- \mr{T}_{\Sigma}(\Theta^{''})\|_2.
	\end{align}
	From Exercise 8.6.4 of \cite{Ver2018High}, Lemma H.1. of \cite{NegWai2011Estimation}, we have
	\begin{align}
		\mbb{E}\left\|  \sum_{i=1}^n   \varrho_i  \frac{X_i}{n}  \right\|_{\mr{op}} \leq \frac{CL}{\sqrt{n}}\|Z_i'\|_{\mr{op}} \leq CL\rho \sqrt{\frac{d_1+d_2}{n}},
\end{align}
where $Z'$ is a random matrix whose entries are standard normal Gaussian.
Combining the arguments above, the proof is complete.
\end{proof}

\section{Tools  for proving Theorem \ref{t:lasso:main}}
\label{sec:inesforla}
In Section \ref{sec:inesforla}, suppose that  Assumption \ref{a:lasso} holds and  we introduce some inequalities to prove Theorem \ref{t:lasso:main}.
The inequalities are the special case of Lemma \ref{l:suphx}, Corollary \ref{ac:|uMv|-cs}, Proposition \ref{p:sc1} and Lemma \ref{l:supRad}
because, like the argument of section 2.2 of \cite{FanWanZhu2021Shrinkage}, linear regression is a special case of trace regression, where $B^*$ and $X_i,\,i=1,\cdots,n$ are diagonal matrices because for some diagonal matrix $M \in \mbb{R}^{d \times d}$, we see that  $\|X_i\|_\mr{op} = \|\mr{diag}(X_i)\|_\infty$ and $\|X_i\|_* = \|\mr{diag}(X_i)\|_1$.

\subsection{Derivation of \eqref{ine:det:xis0} under the assumptions of Theorem \ref{t:lasso:main}}
\label{sec:a:la:xis0}
In Section \ref{sec:a:la:xis0}, we show Lemma \ref{l:HBernstein}, which implies that \eqref{ine:det:xis0} is satisfied with high probability under the assumptions of Theorem \ref{t:lasso:main}.

\begin{lemma}
	\label{l:HBernstein}
	For $0<\delta<1/7$, with probability at least $1-\delta$, we have 
	\begin{align}
	&\left|\frac{1}{n} \sum_{i=1}^n  h\left(\frac{\xi_i}{\lambda_o \sqrt{n}}\right) \langle \vecx_i,\vecv\rangle  \right|\leq C L \rho \left\{\sqrt{\frac{\log (d/s)}{n}}\|\vecv\|_1+\left\{\frac{1+\sqrt{\log (1/\delta)}}{\sqrt{n}}+\sqrt{\frac{s\log (d/s)}{n}}\right\}\|\Sigma^\frac{1}{2}\vecv\|_2\right\}.
	\end{align}
\end{lemma}
\begin{proof}
	Define a set  $V_\vecv$, where 
	\begin{align}
		V_\vecv = \left\{ \vecv \in \mbb{R}^{d}\, |\,  \|\Sigma^\frac{1}{2} \vecv\|_2 =1,\, \|\vecv\|_1 \leq r_1\right\}.
	\end{align} 
	For any $\vecv,\vecv'\in \mbb{R}^{d}$, we have
	\begin{align}
		\label{mcssub2}
		\left\|h\left(\frac{\xi_i}{\lambda_o \sqrt{n}}\right) \langle \vecx_i, \vecv\rangle-h\left(\frac{\xi_i}{\lambda_o \sqrt{n}}\right) \langle \vecx_i, \vecv'\rangle\right\|_{\psi_2} \leq \|\langle \vecx_i, \vecv\rangle-\langle \vecx_i, \vecv'\rangle\|_{\psi_2}
	\end{align}
	because $\left|h\left(\frac{\xi_i}{\lambda_o \sqrt{n}}\right)\right|\leq 1$ and we see that $h\left(\frac{\xi_i}{\lambda_o \sqrt{n}}\right) \langle \vecx_i, \vecv\rangle$ is a $L$-subGaussian distribution.
	From \eqref{mcssub2} and the fact that $\left\{h\left(\frac{\xi_i}{\lambda_o \sqrt{n}}\right) \langle \vecx_i, \vecv\rangle\right\}_{i=1}^n$ is a sequence of i.i.d. random variables, for any $\vecv, \vecv' \in V_\vecv$, we have
	\begin{align}
		\left\|\frac{1}{n} \sum_{i=1}^n  h\left(\frac{\xi_i}{\lambda_o \sqrt{n}}\right)  \langle \vecz_i,\Sigma^\frac{1}{2}\vecv\rangle-\frac{1}{n} \sum_{i=1}^n  h\left(\frac{\xi_i}{\lambda_o \sqrt{n}}\right)  \langle \vecz_i,\Sigma^\frac{1}{2}\vecv'\rangle  \right\|_{\psi_2}\leq L\frac{1}{\sqrt{n}}\left\|\Sigma^\frac{1}{2}\vecv -\Sigma^\frac{1}{2}\vecv'\right\|_2.
	\end{align}
	and from Exercise 8.6.5 of \cite{Ver2018High}, with probability at least $1-\delta$, we have
	\begin{align}
		\sup_{\vecv \in V_M\vecv}\left|\frac{1}{n} \sum_{i=1}^n  h\left(\frac{\xi_i}{\lambda_o \sqrt{n}}\right) \langle \vecx_i,\vecv\rangle  \right| &\leq \frac{CL}{\sqrt{n}}\left\{\mbb{E}\sup_{\vecv \in V_\vecv} \langle \vecz'_i,\Sigma^\frac{1}{2}\vecv\rangle+\sqrt{\log(1/\delta)}\sup_{\vecv \in V_\vecv}\sqrt{\langle \vecz'_i, \Sigma^\frac{1}{2}\vecv\rangle^2}\right\},
	\end{align}
	where $\vecz_i'$ is the $d$-dimensional standard normal Gaussian random vector.
	Combining the arguments above, we have with probability at least $1-\delta$,
	\begin{align}
		\sup_{\vecv \in V_M\vecv}\left|\frac{1}{n} \sum_{i=1}^n  h\left(\frac{\xi_i}{\lambda_o \sqrt{n}}\right) \langle \vecx_i,\vecv\rangle  \right| &\leq \frac{CL}{\sqrt{n}}\left\{\mbb{E}\sup_{\vecv \in V_\vecv} \langle \vecz'_i,\Sigma^\frac{1}{2}\vecv\rangle+\sqrt{\log(1/\delta)}\sup_{\vecv \in V_\vecv}\sqrt{\langle \vecz'_i, \Sigma^\frac{1}{2}\vecv\rangle^2}\right\}\nonumber\\
		&\stackrel{(a)}{\leq} \frac{CL}{\sqrt{n}}\left\{\rho\sqrt{\log (d/s)}r_1+ \rho\sqrt{s\log (d/s)}+\sqrt{\log(1/\delta)}\sup_{\vecv \in V_\vecv}\sqrt{\langle \vecz'_i, \Sigma^\frac{1}{2}\vecv\rangle^2}\right\}\nonumber\\
		& \stackrel{(b)}{\leq} \frac{CL}{\sqrt{n}}\left\{\rho\sqrt{\log (d/s)}r_1+\rho\sqrt{s\log (d/s)}+\sqrt{\log(1/\delta)}\right\},
	\end{align}
	where (a) follows from Proposition E.1 and E.2 of \cite{BelLecTsy2018Slope} and (b) follows from Lemma H.1 of \cite{NegWai2011Estimation}.
	Lastly, using peeling device (Lemma 5 of \cite{DalTho2019Outlier}), the proof is complete.
\end{proof}

\subsection{Derivation of \eqref{ine:det:upper0} under the assumptions of Theorem \ref{t:lasso:main}}
\label{sec:a:la:upper0}
In Section \ref{sec:a:la:upper0}, we show Corollary \ref{ac:|uMv|-la}, which implies that \eqref{ine:det:upper0} is satisfied with high probability under the assumptions of Theorem \ref{t:lasso:main}. The following proposition is easily derived from Proposition 4 of \cite{Tho2020Outlier}, Remark 4 of \cite{DalTho2019Outlier}, Proposition E.1 and E.2 of \cite{BelLecTsy2018Slope} .
\begin{corollary}[Corollary of Proposition 4 of \cite{Tho2020Outlier}]
	\label{ac:|uMv|-la}
	We have for all $o$-sparse vector $\vecu \in \mbb{R}^n$ and $\vecv \in \mbb{R}^{d}$,
	\begin{align}
	&\left| \sum_{i=1}^n u_i \frac{1}{\sqrt{n}} \langle \vecx_i, \vecv \rangle\right|\leq CL\left\{\left(\frac{1+\sqrt{\log (1/\delta)}}{\sqrt{n}}+\sqrt{\frac{s\log (d/s)}{n}}+\sqrt{\frac{o}{n}}\right)\|\Sigma^\frac{1}{2} \vecv\|_2 + \rho \sqrt{\frac{\log (d/s)}{n}}\|\vecv\|_1\right\} \|\vecu\|_2
	\end{align}
	with probability at least $1-\delta$.
\end{corollary}

\begin{proposition}
	\label{p:sc1-lasso}
	Let  
	\begin{align}
	\mc{R}_{lasso} = \left\{ \vectheta \in \mbb{R}^{d}\, |\, \|\vectheta\|_1 \leq c_\kappa\|\Sigma^\frac{1}{2}\vectheta\|_2 ,\, \|\Sigma^\frac{1}{2}\vectheta\|_2 =r\right\},
	\end{align}
	where $r$ is a number such that $0<r\leq \frac{1}{4\sqrt{3} L^2}$.
	Assume that $\lambda_o  \sqrt{n} \geq 72L^4\sigma$.
	Then, with probability at least $1-\delta$, we have
	\begin{align}
	\label{ine:sc1-1}
		&\inf_{\vectheta\in \mc{R}_{lasso}}\sum_{i=1}^n \lambda_o^2 \left\{-h\left(\frac{\xi_i + \langle \vecx_i, \vectheta\rangle}{\lambda_o\sqrt{n}}\right)-h \left(\frac{\xi_i}{\lambda_o\sqrt{n}}\right) \right\}\frac{\langle \vecx_i, \vectheta\rangle}{\lambda_o \sqrt{n}}\nonumber \\
		&\geq \frac{1}{3}\|\Sigma^\frac{1}{2}\vectheta\|_2^2-C\left(L\rho\sqrt{\frac{s\log (d/s)}{n}}+ \sqrt{\frac{\log(1/\delta)}{n}}\right)\|\vectheta\|_2 -C\frac{\log(1/\delta)}{n}.
\end{align} 
\end{proposition}

\section{Tools for proving Theorem \ref{t:mc:main1} and Theorem \ref{t:mc:main2}}
\label{sec:inesformc}
In Section \ref{sec:inesformc}, we state some inequalities used to prove Theorem \ref{t:mc:main1} and Theorem \ref{t:mc:main2}.

\subsection{Derivation of  \eqref{ine:det:xis0} under the assumptions of Theorem \ref{t:mc:main1} and Theorem \ref{t:mc:main2}}
\label{sec:a:mc:xis0}
In Section \ref{sec:a:mc:xis0}, we show Lemma \ref{l:mcspec}, which implies that \eqref{ine:det:xis0} is satisfied with high probability under the assumptions of Theorem \ref{t:mc:main1} or Theorem \ref{t:mc:main2}.

\begin{lemma}
	\label{l:mcspec}
	Suppose that Assumption \ref{a:mc1} or Assumption \ref{a:mc2} holds.
 For $\delta>0$, with probability at least $1-\delta$, we have 
	\begin{align}
		\left\| \frac{\lambda_o}{\sqrt{n}} \sum_{i=1}^n  h\left(\frac{\xi_i}{\lambda_o\sqrt{n}}\right) X_i\right\|_\mr{op} 
		& \leq C\left(\sigma_{\xi}\sqrt{\frac{d_{mc}(\log d_{mc}+\log(1/\delta))}{n}}+\lambda_o\frac{d_{mc}}{\sqrt{n}}(\log d_{mc}+\log(1/\delta))\right).
	\end{align}
\end{lemma}
\begin{proof}
	Let $U_i =\frac{\lambda_o}{\sqrt{n}}  h\left(\frac{\xi_i}{\lambda_o\sqrt{n}}\right) X_i$. We confirm that $U_i$ satisfies the conditions of Lemma 7 of \cite{NegWai2012Restricted} and then we apply Lemma 7 of \cite{NegWai2012Restricted} to $\sum_{i=1}^n U_i$. From Assumption \ref{a:mc1} or \ref{a:mc2}, we have $\mbb{E}\sum_{i=1}^nU_i = 0$. From the definition of $X_i$, we have
	\begin{align}
		\left\|U_i\right\|_\mr{op} = d_{mc}\left\|\frac{\lambda_o}{\sqrt{n}} h\left(\frac{\xi_i}{\lambda_o\sqrt{n}}\right)\varepsilon_i E_i \right\|_\mr{op}
	\end{align}
and from the definition of $h(\cdot)$ and $E_i$, we have	
\begin{align}
	d_{mc}\left\|\frac{\lambda_o}{\sqrt{n}} h\left(\frac{\xi_i}{\lambda_o\sqrt{n}}\right)\varepsilon_i E_i \right\|_\mr{op} \leq  \lambda_o\frac{d_{mc}}{\sqrt{n}}.
\end{align}
On the other hand, from the definition of $h(\cdot)$, we have
\begin{align}
&\max\left\{\left\|\frac{\lambda_o^2}{n}\mbb{E}h\left(\frac{\xi_i}{\lambda_o\sqrt{n}}\right)^2 X_i X_i^\top\right\|_{\mr{op}},  \left\|\frac{\lambda_o^2}{n}\mbb{E}h\left(\frac{\xi_i}{\lambda_o\sqrt{n}}\right)^2 X_i^\top X_i \right\|_{\mr{op}}\right\}\nonumber\\
&\leq \frac{d_{mc}}{n^2} \max\left\{\left\|\mbb{E}\xi_i^2E_iE_i^\top\right\|_{\mr{op}},\left\|\mbb{E}\xi_i^2E_i^\top E_i\right\|_{\mr{op}}\right\}\nonumber\\
&=\sigma^2_{\xi}\frac{d_{mc}}{n^2}
\end{align}
Applying Lemma 7 of \cite{NegWai2012Restricted}, we have
\begin{align}
	\mbb{P} \left[\left\| \frac{\lambda_o}{\sqrt{n}}\sum_{i=1}^n h\left(\frac{\xi_i}{\lambda_o \sqrt{n}}\right)X_i\right\|_\mr{op}  \geq t\right] \leq d_1d_2\max\left\{\exp\left(\frac{-t^2n}{4\sigma^2_{\xi}d_{mc}}\right),\exp  \left(\frac{-t\sqrt{n}}{2\lambda_od_{mc}}\right)\right|
\end{align}
Setting $\delta = d_{mc}^2\max\left\{\exp\left(\frac{-t^2n}{4\sigma^2_{\xi}d_{mc} }\right),\exp  \left(\frac{-t\sqrt{n}}{2\lambda_od_{mc}}\right)\right\}$, the proof is complete.
\end{proof}

\subsection{Derivation of  \eqref{ine:det:upper0} under the assumptions of Theorem \ref{t:mc:main1} and Theorem \ref{t:mc:main2}}
\label{sec:a:mc:upper0}
In Section \ref{sec:a:mc:upper0}, we show Lemma \ref{al:|uMv|-mc}, which implies that \eqref{ine:det:upper0} is satisfied under the assumptions of Theorem \ref{t:mc:main1} and Theorem \ref{t:mc:main2}.

\begin{lemma}
	\label{al:|uMv|-mc}
	Suppose that Assumption \ref{a:mc1} or Assumption \ref{a:mc2} holds.
	Assume that for any $M\in \mbb{R}^{d_1\times d_2}$,
	\begin{align}
		\label{ine:M-2}
		\|M\|_\infty \leq c_m \frac{1}{d_{mc}}\|M\|_{\mr{F}}.
	\end{align} 
	for some number $c_m$.
	Then, for any $o$-sparse vector $\vecu \in \mbb{R}^n$ such that $\|\vecu\|_2 \leq 2 \sqrt{o}$,
	we have
	\begin{align*}
	\left| \sum_{i=1}^n \frac{\lambda_o}{\sqrt{n}}u_i \langle X_i, M\rangle\right|  &\leq c_m2 \lambda_o\sqrt{n}\frac{o}{n} \|M\|_\mr{F}.
	\end{align*}
\end{lemma}
\begin{proof}
	We re-write $\left|\sum_{i=1}^n u_i \langle X_i, M\rangle\right|$ as
	\begin{align}
	\left| \vecu^\top  \tilde{X} \mr{vec}\left(M\right)\right|,\quad \text{where}\quad \tilde{X} = \left(\begin{array}{c}
	\mr{vec}\left(X_1\right)^\top \\
	\vdots\\
	\mr{vec}\left(X_n\right)^\top
	\end{array}\right).
	\end{align}
	From H{\"o}lder's inequality and  the  definition of $X_i$, we have
	\begin{align}
	\left| \vecu^\top  \tilde{X} \mr{vec}\left(M\right)\right| \leq \|\vecu\|_1 \left\|\tilde{X}\mr{vec}(M)\right\|_\infty &\leq \|\vecu\|_1 \left\|\tilde{X}\right\|_\infty \left\|\mr{vec}(M)\right\|_\infty= \|\vecu\|_1 \left\|\tilde{X}\right\|_\infty \left\|M\right\|_\infty.
	\end{align}
From the definition of $X_i$ and \eqref{ine:M-2}, we have
\begin{align}
	\|\vecu\|_1 \left\|\tilde{X}\right\|_\infty \left\|M\right\|_\infty&\leq \|\vecu\|_1 \sqrt{d_1d_2} \left\|M\right\|_\infty \leq c_m\|\vecu\|_1 \|M\|_{\mr{F}}\leq c_m\sqrt{o}\|\vecu\|_2 \|M\|_{\mr{F}}
	\end{align}
	and the proof is complete.
\end{proof}

\subsection{Derivation of \eqref{ine:det:lower} under the assumptions of Theorems \ref{t:mc:main1} and \ref{t:mc:main2}}
\label{sec:a:la:lower}
In Section \ref{sec:a:la:lower}, we show Corollary \ref{c:mcsc1}, which implies that \eqref{ine:det:lower} is satisfied with high probability under the assumptions of  Theorem \ref{t:mc:main1} and, we show Corollary \ref{c:mcsc2}, which implies that \eqref{ine:det:lower} is satisfied with high probability under the assumptions of Theorem \ref{t:mc:main2}.
\begin{corollary}
		\label{c:mcsc1}
		Suppose that Assumption \ref{a:mc1} holds.
		Let  
		\begin{align*}
		\mc{R}_{mc} = \left\{ \Theta\in \mbb{R}^{d_1 \times d_2}\, |\, \|\Theta\|_* \leq C \sqrt{r} \|\Theta\|_{\mr{F}},\, \|\Theta\|_\mr{F} =r_{mc}\right\},
		\end{align*}
		where $r_{mc}$ is some number.
		Suppose 
		\begin{align}
			\label{ine:mcs2}
			\|\Theta\|_\infty &\leq \frac{1}{12r_{mc}} \frac{1}{d_{mc}}\|\Theta\|_{\mr{F}}\\
			\label{ine:constraint}
			\|\Theta\|_\infty &\leq 2\frac{\alpha^*}{d_{mc}}
		\end{align}
		and 
		\begin{align}
			\label{ine:mc1:lambda}
			\lambda_o  \sqrt{n} \geq 2\sigma_{\xi,\alpha}\min\left\{\left(\frac{n}{o}\right)^\frac{1}{\alpha+1},\left(\frac{n}{rd_{mc}\log d_{mc}}\right)^\frac{1}{\alpha}\right\}.
		\end{align}
		Then, with probability at least $1-\delta$, we have
		\begin{align}
		\label{inec:huberconvex-main-mc}
		&\inf_{\Theta \in \mc{R}_{mc}}\left[\lambda_o^2\sum_{i=1}^n  \left\{-h\left(\frac{\xi_i - \langle  X_i, \Theta\rangle}{\lambda_o\sqrt{n}}\right)+h \left(\frac{\xi_i}{\lambda_o\sqrt{n}}\right) \right\}\frac{\langle X_i,\Theta\rangle}{\lambda_o\sqrt{n}} \right] \nonumber\\
		&\geq \frac{2}{3}\|\Theta\|_\mr{F}^2 -C\left( \alpha^* \sqrt{r\frac{d_{mc}\{\log d_{mc}+\log(1/\delta)\}}{n}}+ \sqrt{r}\frac{d_{mc}\log  d_{mc}}{n}+ \alpha^* \left(\frac{o}{n}\right)^\frac{\alpha}{2(1+\alpha)}	\right)\|\Theta\|_\mr{F}-5\frac{\log(1/\delta)}{n}.
		\end{align} 
	\end{corollary}
	
	\begin{proof}
		The proof is almost identical to the one of Proposition \ref{p:sc1}. However, because $\left\{X_i\right\}_{i=1}^n$ for matrix completion is not $L$-subGaussian, we should calculate  \eqref{ap:ine:cov1},  \eqref{ap:ine:cov2} and \eqref{ap:ine:cov3}  by another strategy and \eqref{ine:hub-stoc-upper} also requires another strategy.
		We note that 
		\begin{align}
			\mbb{E}\langle X_i,\Theta\rangle^2 = \|\Theta\|_{\mr{F}}^2 = r_{mc}^2.
		\end{align}
	From the definition of $X_i$ and \eqref{ine:mcs2}, we have
	\begin{align}
		\label{ine:mc:cov3:mc}
		\mbb{E}\langle X_i,\Theta\rangle^4 \leq \mbb{E}\langle X_i,\Theta\rangle^2\|X_i\|_1^2 \|\Theta\|_\infty^2  \leq \frac{1}{144}r_{mc}^2.
	\end{align}
	From the definition of $X_i$ and  \eqref{ine:constraint2}, we have
	\begin{align}
		\label{ine:mc:cov4:mc0}
		\mbb{E}\langle X_i,\Theta\rangle^4 \leq \mbb{E}\langle X_i,\Theta\rangle^2\|X_i\|_1^2 \|\Theta\|_\infty^2  \leq 4\alpha^{*2}\|\Theta\|_{\mr{F}}^2.
	\end{align}
		Instead of \eqref{ap:ine:cov1}, we have
		\begin{align}
			\sum_{i=1}^n\mbb{E}  \left[  \left|\frac{  \langle  X_i,\Theta \rangle}{\lambda_o\sqrt{n}} \right| ^2 I \bigg( \left|\frac{  \langle  X_i,\Theta \rangle}{\lambda_o\sqrt{n}} \right| \geq \frac{1}{2\lambda_o\sqrt{n}}   \bigg)  \right] &\stackrel{(a)}{\leq} 	\sum_{i=1}^n\sqrt{\mbb{E}   \left|\frac{  \langle  X_i,\Theta \rangle}{\lambda_o\sqrt{n}} \right| ^4  } \sqrt{\mbb{E} \ I \bigg( \left| \langle  X_i,\Theta \rangle \right| \geq \frac{1}{2}   \bigg) }\nonumber\\
			&\stackrel{(b)}{=} 	\sum_{i=1}^n\sqrt{\mbb{E}  \left|\frac{  \langle  X_i,\Theta \rangle}{\lambda_o\sqrt{n}} \right| ^4  } \sqrt{\mbb{P}  \left[  \left|\langle  X_i,\Theta \rangle\right| \geq \frac{1}{2} \right] }\nonumber\\
			&\stackrel{(c)}{\leq}	\sum_{i=1}^n4\sqrt{\mbb{E} \left|\frac{  \langle  X_i,\Theta \rangle}{\lambda_o\sqrt{n}} \right| ^4} \sqrt{ \mbb{E} \langle  X_i,\Theta \rangle^4 }\nonumber\\
			&\stackrel{(d)}{\leq}\frac{1}{3\lambda_o^2} \|\Theta\|_{\mr{F}}^2
			\end{align}
			where (a) follows from H{\"o}lder's inequality, 
			(b) follows from the relation between indicator function and expectation,
			 (c) follows from Markov's inequality and  (d) follows from \eqref{ine:mc:cov3:mc} and the definition of $\mc{R}_{mc}$.
		Second,  we have
		\begin{align}
			\sum_{i=1}^n\mbb{E}\left|\frac{  \langle  X_i,\Theta\rangle}{\lambda_o\sqrt{n}} \right| ^2 I\left( \left|\frac{\xi_i}{\lambda_o\sqrt{n}} \right|\geq \frac{1}{2} \right)
			&\stackrel{(a)}{\leq} \sum_{i=1}^n\sqrt{\mbb{E}\left|\frac{  \langle  X_i,\Theta\rangle}{\lambda_o\sqrt{n}} \right| ^4}  \sqrt{\mbb{E}I\left( \left|\frac{\xi_i}{\lambda_o\sqrt{n}} \right|\geq \frac{1}{2} \right)}\nonumber\\
			&\stackrel{(b)}{\leq}\sum_{i=1}^n\sqrt{\mbb{E}\left|\frac{  \langle  X_i,\Theta\rangle}{\lambda_o\sqrt{n}} \right| ^4}  \sqrt{\mbb{P}  \left[ \left|\frac{\xi_i}{\lambda_o\sqrt{n}} \right|\geq \frac{1}{2}  \right]}\nonumber\\	  
			&\stackrel{(c)}{\leq}\sum_{i=1}^n \left(\frac{2}{\lambda_o\sqrt{n}}\right)^\frac{\alpha}{2}\sqrt{\mbb{E}\left|\frac{  \langle  X_i,\Theta\rangle}{\lambda_o\sqrt{n}} \right|^4}  \sqrt{\mbb{E} |\xi_i^\alpha|}\nonumber\\	  
			&\stackrel{(d)}{\leq}\sum_{i=1}^n\left(\frac{2\sigma_{\xi, \alpha}}{\lambda_o\sqrt{n}}\right)^\frac{\alpha}{2} \sqrt{\frac{\mbb{E}\langle  X_i,\Theta\rangle^4}{\lambda_o^4n^2}} \nonumber\\
			&\stackrel{(e)}{\leq} \left(\frac{2\sigma_{\xi, \alpha}}{\lambda_o\sqrt{n}}\right)^\frac{\alpha}{2}\frac{2\alpha^*}{\lambda_o^2}\|\Theta\|_{\mr{F}}\nonumber\\
			&\stackrel{(f)}{\leq} \frac{2 \alpha^*}{\lambda_o^2} \left\{\left(\frac{o}{n}\right)^\frac{\alpha}{2(1+\alpha)}+\sqrt{\frac{rd_{mc}\log  d_{mc}}{n}}\right\}	\|\Theta\|_{\mr{F}},
			\end{align}
			where (a) follows from H{\"o}lder's inequality, (b) follows from relation between indicator function and expectation, (c) follows from Markov's inequality, (d) follows  from Assumption \ref{a:mc1}  and  (e) follows from \eqref{ine:mc:cov4:mc0},  and (f) follows from \eqref{ine:mc1:lambda}.
	
	Instead of \eqref{ap:ine:cov3}, from \eqref{ine:constraint}, we have
		\begin{align}
			\label{ine:wvmc}
			\mbb{E}f_i^2(\Theta) \leq \mbb{E}\frac{\langle X_i ,\Theta\rangle^4}{\lambda_o^4n^2}\leq \mbb{E}\frac{\langle X_i ,\Theta\rangle^2 \|X_i\|_1^2 \|\Theta\|_\infty^2}{\lambda_o^4n^2}\leq \frac{4\alpha^{*2}}{\lambda_o^4n^2}\|\Theta\|_{\mr{F}}^2
		\end{align}
		and remembering $\sigma^2_f= \sup_{ \Theta \in \mc{R}_{mc}} \sum_{i=1}^n\mbb{E}  \{f_i(\Theta)-\mbb{E}f_i(\Theta)\}^2$, we have
		\begin{align}
			\lambda_o^2 \times \sigma_f \sqrt{8\log(1/\delta)} \leq 4\alpha^{*}\sqrt{ \frac{\log(1/\delta)}{n}}\|\Theta\|_{\mr{F}}.
		\end{align}
		To obtain the result \eqref{inec:huberconvex-main-mc} by the strategy of the proof of Proposition \ref{p:sc1}, the remaining term we must calculate is 
		\begin{align}
			\mbb{E}\sup_{\Theta\in\mc{R}_{mc}}  \left|  \sum_{i=1}^n   \varrho_i  \frac{\langle X_i , \Theta\rangle}{2\lambda_o^2n}  \right|.
	\end{align}
	From Lemma 6 of \cite{NegWai2012Restricted}  and combining the definition of $\mc{R}_{mc}$, we have
	\begin{align}
		\mbb{E}\sup_{\Theta \in\mc{R}_{mc}}  \left|  \sum_{i=1}^n   \varrho_i  \frac{\langle X_i , \Theta\rangle}{2\lambda_o^2n}  \right| &\leq \frac{C}{\lambda_o^2}\left\{ \sqrt{\frac{d_{mc}\log  d_{mc}}{n}}+\frac{d_{mc}\log  d_{mc}}{n}\right\}\sup_{\Theta \in\mc{R}_{mc}}  \|\Theta\|_{*}\nonumber\\
		&\leq \frac{C}{\lambda_o^2}\left\{ \sqrt{r\frac{d_{mc}\log  d_{mc}}{n}}+\sqrt{r}\frac{d_{mc}\log d_{mc}}{n}\right\}\|\Theta\|_{\mr{F}}
\end{align}
and combining $\alpha^*\geq 1$, the proof is complete.
\end{proof}

	\begin{corollary}
		\label{c:mcsc2}
		Suppose that Assumption \ref{a:mc2} holds.
		Let  
		\begin{align}
		\mc{R}_{mc} = \left\{ L\in \mbb{R}^{d_1 \times d_2}\, |\, \|\Theta\|_* \leq C \sqrt{r} \|\Theta\|_{\mr{F}},\, \|\Theta\|_\mr{F} =r_{mc}\right\}.
		\end{align}
		Suppose 
		\begin{align}
			\label{ine:mcs3}
			\|\Theta\|_\infty &\leq \frac{1}{12r_{mc}} \frac{1}{d_{mc}}\|\Theta\|_{\mr{F}},\\
			\label{ine:constraint2}
			\|\Theta\|_\infty &\leq 2\frac{\alpha^*}{d_{mc}},
		\end{align}
		and 
		\begin{align}
			\label{ine:mc2:lambda}
			\lambda_o  \sqrt{n} \geq 2 \sigma_{\xi ,\psi_\alpha} \min\left\{\log^\frac{1}{\alpha} \frac{n}{o}, \log^\frac{1}{\alpha} \frac{n}{rd_{mc}\log d_{mc}}\right\}.
		\end{align}
		Then, with probability at least $1-\delta$, we have
		\begin{align}
		&\inf_{\Theta \in \mc{R}_{mc}}\left[\lambda_o^2\sum_{i=1}^n  \left\{-h\left(\frac{\xi_i - \langle  X_i, \Theta\rangle}{\lambda_o\sqrt{n}}\right)+h \left(\frac{\xi_i}{\lambda_o\sqrt{n}}\right) \right\}\frac{\langle   X_i,\Theta\rangle}{\lambda_o\sqrt{n}} \right] \nonumber\\
		&\geq \frac{2}{3}\|\Theta\|_\mr{F}^2 -C\left( \alpha^* \sqrt{r\frac{d_{mc}\{\log d_{mc}+\log(1/\delta)\}}{n}}+\sqrt{r}\frac{d_{mc}\log  d_{mc}}{n}+ \alpha^* \sqrt{\frac{o}{n}}\right)\|\Theta\|_\mr{F}-5\frac{\log(1/\delta)}{n}.
		\end{align} 
	\end{corollary}
	
	\begin{proof}
		The proof is almost identical to the one of Proposition \ref{p:sc1}.
		Like Corollary \ref{c:mcsc1}, because $\left\{X_i\right\}_{i=1}^n$ for matrix completion is not $L$-subGaussian, we should calculate  \eqref{ap:ine:cov1} and  \eqref{ap:ine:cov3} in the same strategy of Corollary \ref{c:mcsc1}.
		However, to calculate \eqref{ap:ine:cov2}, we need another strategy
		than one of Corollary \ref{c:mcsc1} because of the difference of the assumption on random noise.

		From the definition of $X_i$,  \eqref{ine:constraint2}, we have
		\begin{align}
			\label{ine:mc:cov4:mc}
			\mbb{E}\langle X_i,\Theta\rangle^4 \leq \mbb{E}\langle X_i,\Theta\rangle^2\|X_i\|_1^2 \|\Theta\|_\infty^2  \leq 2 \alpha^{*2}\|\Theta\|_{\mr{F}}^2.
		\end{align}

		For  \eqref{ap:ine:cov2}, we have
		\begin{align}
		\sum_{i=1}^n\mbb{E}\left|\frac{  \langle  X_i,\Theta\rangle}{\lambda_o\sqrt{n}} \right| ^2 I\left( \left|\frac{\xi_i}{\lambda_o\sqrt{n}} \right|\geq \frac{1}{2} \right)
		&\stackrel{(a)}{\leq} \sum_{i=1}^n\sqrt{\mbb{E}\left|\frac{  \langle  X_i,\Theta\rangle}{\lambda_o\sqrt{n}} \right| ^4}  \sqrt{\mbb{E}I\left( \left|\frac{\xi_i}{\lambda_o\sqrt{n}} \right|\geq \frac{1}{2} \right)}\nonumber \\
		&\stackrel{(b)}{\leq}\sum_{i=1}^n\sqrt{\mbb{E}\left|\frac{  \langle  X_i,\Theta\rangle}{\lambda_o\sqrt{n}} \right| ^4}  \sqrt{\mbb{P}  \left[ \left|\frac{\xi_i}{\lambda_o\sqrt{n}} \right|\geq \frac{1}{2}  \right]}\nonumber \\	  
		&\stackrel{(c)}{\leq}\sum_{i=1}^n\sqrt{\mbb{E}\left|\frac{  \langle  X_i,\Theta\rangle}{\lambda_o\sqrt{n}} \right| ^4} \sqrt{ \exp\left(\frac{-\lambda^\alpha_o n^\frac{\alpha}{2}}{2^\alpha \sigma_{\xi, \psi_\alpha}^\alpha}\right)}\nonumber \\	  
		&\stackrel{(d)}{\leq}\sum_{i=1}^n \sqrt{ \frac{o}{n} }\frac{1 }{\lambda_o^2n}  \sqrt{\mbb{E}\langle  X_i,\Theta\rangle^4} \nonumber \\
		&\stackrel{(e)}{\leq} \frac{2 \alpha^*}{\lambda_o^2}\left(\sqrt{\frac{o}{n}}+\sqrt{\frac{rd_{mc}\log d_{mc}}{n}} \right)\|\Theta\|_{\mr{F}},
		\end{align}
		where (a) follows from H{\"o}lder's inequality, (b) follows from relation between indicator function and expectation, (c) follows from assumption, (d) follows from \eqref{ine:mc2:lambda} and (e) follows from \eqref{ine:mc:cov4:mc}.
	\end{proof}

\section{Proof of Theorem \ref{t:cs:main}}
\label{asec:maincs}
Suppose that  the assumptions of Theorem~\ref{t:cs:main} hold.
The proof is complete if we confirm the assumptions in Theorem~\ref{t:det:main} with
\begin{align}
		r_{mcs}  =c'_{mcs} \times\lambda_o\sqrt{n}\times L\left(\frac{1+\sqrt{\log (1/\delta)}}{\sqrt{n}}+c_\kappa\rho\sqrt{r\frac{d_1+d_2}{n}}+\frac{o}{n}\sqrt{\log \frac{n}{o}}\right)
\end{align}
are satisfied with probability at least $1-3\delta$.
We divide the proof into four steps:
\begin{itemize}
	\item[I. ] We confirm \eqref{ine:det:xis0} and \eqref{ine:det:upper0} are satisfied with probability at least $1-2\delta$.
	\item[I\hspace{-.1em}I. ]We confirm \eqref{ine:det:par} is satisfied. 
	\item[I\hspace{-.1em}I\hspace{-.1em}I.  ] We confirm \eqref{ine:det:lower} is satisfied with probability at least $1-\delta$. 
	\item[I\hspace{-.1em}V. ] We confirm~\eqref{ine:det:condc}  is satisfied.
\end{itemize}
In this section, we set 
\begin{align}
	&r_{a,\mr{F}} =r_{b,\mr{F}}=\frac{1+\sqrt{\log (1/\delta)}}{\sqrt{n}}+ \frac{o}{n} \sqrt{\log \frac{n}{o}},\quad r_{a,*}=r_{b,*} = \sqrt{\frac{d_1+d_2}{n}}, \\
	&r_{c,\mr{F}}=\rho c_\kappa \sqrt{r\frac{d_1+d_2}{n}}+\sqrt{\frac{\log(1/\delta)}{n}}, \quad r_c = \frac{\log(1/\delta)}{n}
 \end{align}
and $C_{mcs}$ is some sufficiently large numerical constant.

\subsection{Proof of I}
Remember that \eqref{ine:det:xis0} is 
\begin{align}
	&\left| \frac{\lambda_o}{\sqrt{n}}\sum_{i=1}^n h\left(\frac{\xi_i}{\lambda_o \sqrt{n}} \right) \langle  X_i, \Theta_\eta \rangle \right| \leq a_{\mr{F}} r_{a,\mr{F}} \left\| \mr{T}_\Sigma (\Theta_\eta)\right\|_\mr{F} + a_*r_{a,*}\|\Theta_\eta\|_*
\end{align}
and \eqref{ine:det:upper0} is
\begin{align}
	&\left| \sum_{i=1}^n \frac{\lambda_o}{\sqrt{n}}u_i \langle X_i, \Theta_\eta\rangle\right|
	\leq b_{\mr{F}} r_{b,\mr{F}}\left\| \mr{T}_\Sigma (\Theta_\eta)\right\|_\mr{F} + b_*r_{b,*}\|\Theta_\eta\|_*.
\end{align}
First, from Lemma \ref{l:suphx} , we confirm \eqref{ine:det:xis0} with 
\begin{align}
	a_\mr{F}= C_{mcs}\lambda_o\sqrt{n}L,\quad  a_* = C_{mcs}\lambda_o\sqrt{n}L\rho
\end{align}
with probability at least $1-\delta$.
Second, from Corollary \ref{ac:|uMv|-cs} and $\frac{o}{n}\leq 1$ we confirm \eqref{ine:det:upper0} with
\begin{align}
	b_\mr{F} = C_{mcs}\lambda_o \sqrt{n}L,\quad b_* = C_{mcs}\lambda_o \sqrt{n}L \rho,
\end{align}
with probability at least $1-\delta$.
From union bound, we confirm \eqref{ine:det:xis0} and \eqref{ine:det:upper0} are satisfied with probability at least $(1-\delta)^2$

\subsection{Proof of I\hspace{-.1em}I}
Remember we have
\begin{align}
		\lambda_*&= c_{mcs} \times\lambda_o\sqrt{n}\times L\times\left(\frac{1}{c_\kappa \sqrt{r}}\frac{1+\sqrt{\log (1/\delta)}}{\sqrt{n}}+\rho\sqrt{\frac{d_1+d_2}{n}}+\frac{1}{c_\kappa \sqrt{r}}\frac{o}{n}\sqrt{\log \frac{n}{o}}\right)\nonumber\\
		&=c_{mcs} \times\lambda_o\sqrt{n}\times L\times\left(\frac{r_{a,\mr{F}}}{c_\kappa \sqrt{r}}+\rho r_{a,*}\right)
\end{align}
for some sufficiently large numerical constant $c_{mcs}$ and we have
\begin{align}
	C_s &= \frac{a_\mr{F}+\sqrt{2}b_\mr{F}}{c_\kappa \sqrt{r}}r_{a,\mr{F}} + (a_*+\sqrt{2}b_*)r_{a,*} \nonumber\\
	&=(1+\sqrt{2})C_{mcs} \times\lambda_o\sqrt{n}\times L\times\left(\frac{r_{a,\mr{F}}}{c_\kappa \sqrt{r}}+\rho r_{a,*}\right).
\end{align}
For a sufficiently large numerical constant $c_{mcs}$, we see $\lambda_*-C_s>0$.

\subsection{Proof of I\hspace{-.1em}I\hspace{-.1em}I}
Remember that \eqref{ine:det:lower} is
\begin{align}
	c_1 \| \mr{T}_\Sigma (\Theta_{\eta})\|_\mr{F}^2-c_2r_{c,\mr{F}}\|\mr{T}_\Sigma (\Theta_{\eta})\|_{\mr{F}} -c_3 r_{c}\leq \lambda_o^2\sum_{i=1}^n \left\{	-h \left(\frac{\xi_i-\langle X_i,\Theta_{\eta}\rangle}{\lambda_o\sqrt{n}} \right)+h  \left(\frac{\xi_i}{\lambda_o\sqrt{n}}  \right)\right\} \frac{\langle X_i,\Theta_{\eta}\rangle}{\lambda_o \sqrt{n}}
\end{align}
and we can confirm \eqref{ine:det:lower} from Proposition \ref{p:sc1}, $L\geq 1$ and the definition of $\lambda_o\sqrt{n}$
with
\begin{align}
	c_1 = \frac{1}{3},\quad c_2 =C_{mcs}L,\quad c_3 =C_{mcs}^2L^2.
\end{align}

\subsection{Proof of I\hspace{-.1em}V}
From the definition of $\lambda_*,\,\lambda_o\sqrt{n}$ and the values of $a_{\mr{F}}, a_*, b_{\mr{F}}, b_*, c_1,c_2,c_3,r_{a,\mr{F}},r_{a,*},r_{a,\mr{F}},r_{b,*},r_{b,\mr{F}},r_{c,\mr{F}},r_c$, we have
\begin{align}
	&\frac{c_2r_{c,\mr{F}} +C_{\lambda_*} +\sqrt{c_1c_3r_c}}{c_1} \nonumber\\
	&\leq\frac{(a_\mr{F}+\sqrt{2}b_\mr{F})r_{a,\mr{F}} + (a_*+\sqrt{2}b_*)c_\kappa \sqrt{r} r_{a,*}+\lambda_*  c_\kappa \sqrt{r} +c_2r_{c,\mr{F}} +\sqrt{c_1c_3r_c}}{c_1}\nonumber\\
	&\leq C_{mcs}'\times \lambda_o\sqrt{n} \times L \left(r_{a,\mr{F}}+\rho c_\kappa \sqrt{r}r_{a,*}\right),
\end{align}
where $C_{mcs}'$ is some sufficiently large constant.
Remember that
\begin{align}
r_0  &=c'_{mcs} \times\lambda_o\sqrt{n}\times L\left(\frac{1+\sqrt{\log (1/\delta)}}{\sqrt{n}}+c_\kappa\rho\sqrt{r\frac{d_1+d_2}{n}}+\frac{o}{n}\sqrt{\log \frac{n}{o}}\right)\nonumber\\
&=c'_{mcs} \times\lambda_o\sqrt{n}\times L\left(r_{a,\mr{F}}+\rho c_\kappa \sqrt{r}r_{a,*}\right).
\end{align}
For sufficiently large numerical constant $c'_{mcs}$, we confirm \eqref{ine:det:condc}.

\section{Proof of Theorem \ref{t:lasso:main}}
\label{s:laap}
The proof of Theorem \ref{t:lasso:main} is almost identical to that of  Theorem \ref{t:cs:main}.
Therefore, we shall omit it.

\section{Proof of Theorem \ref{t:mc:main1}}
\label{s:mc1}
Suppose the assumptions of Theorem~\ref{t:mc:main1} hold.
The proof is complete if we confirm the assumptions in Theorem~\ref{t:det:main} with 
\begin{align}
	&r_0 = c_{mc1}'\times \nonumber\\
	&\left((\sigma_\xi+\alpha^*) \sqrt{\frac{r d_{mc} (\log d_{mc}+ \log (1/\delta))}{n}}+\lambda_o\frac{\sqrt{r}d_{mc}(\log  d_{mc}+\log(1/\delta))}{\sqrt{n}}+ \sqrt{r}\frac{d_{mc}\log d_{mc}}{n}+ \sqrt{\lambda_o \sqrt{n}\frac{o}{n}} + \alpha^*\left(\frac{o}{n}\right)^\frac{\alpha}{2(1+\alpha)}\right)
\end{align} 
are satisfied with probability at least $1-2\delta$.
We divide the proof into four steps:
\begin{itemize}
	\item[I. ] We confirm \eqref{ine:det:upper0} and \eqref{ine:det:xis0}  are satisfied with probability at least $1-\delta$.
	\item[I\hspace{-.1em}I. ]We confirm \eqref{ine:det:par} is satisfied. 
	\item[I\hspace{-.1em}I\hspace{-.1em}I.  ] We confirm \eqref{ine:det:lower} is satisfied with probability at least $1-\delta$. 
	\item[I\hspace{-.1em}V. ] We confirm~\eqref{ine:det:condc}  is satisfied.
\end{itemize}
In Section \ref{s:mc1}, we set 
\begin{align}
	&r_{a,\mr{F}} = 0,\quad r_{a,*} = \sigma_\xi\sqrt{\frac{d_{mc}(\log d_{mc} + \log(1/\delta))}{n}}+\lambda_o\frac{d_{mc}(\log  d_{mc}+\log(1/\delta))}{\sqrt{n}},\quad r_{b,\mr{F}} = \sqrt{\lambda_o\sqrt{n}\frac{o}{n}},\nonumber\\
	&r_{b,*}=0,\quad r_{c,\mr{F}} = \alpha^*\sqrt{r\frac{d_{mc}\{\log d_{mc}+\log(1/\delta)\}}{n}}+ \sqrt{r}\frac{d_{mc}\log d_{mc}}{n}+ \alpha^*\left(\frac{o}{n}\right)^\frac{\alpha}{2(1+\alpha)}	,\quad r_c = \frac{\log(1/\delta)}{n} 
 \end{align}
and $C_{mc1}$ is some sufficiently large numerical constant.

Before the proceeding each steps above, we confirm that when  
\begin{align}
	\alpha(\Theta_\eta) \geq \frac{1}{12r_0},
\end{align}
we have $\|B^*-\hat{B}\|_{\mr{F}} \leq 24\alpha^*r_0 $.
From $\alpha(\Theta_\eta) \geq \frac{1}{12r_0}$, we have
\begin{align}
	&\frac{1}{r_0}\leq \alpha(\Theta_\eta) = d_{mc}\frac{\|B^*-\hat{B}\|_\infty}{\|B^*-\hat{B}\|_{\mr{F}}}\Rightarrow \|B^*-\hat{B}\|_{\mr{F}} \leq 12\sqrt{d_1d_2} \times r_0\|B^*-\hat{B}\|_\infty.
\end{align}
From the spikiness condition and the constraint of the optimization problem and, we have $\|B^*\|_\infty \leq \frac{\alpha^*}{ d_{mc}}$ and $\|\hat{B}\|_\infty\leq \frac{\alpha^*}{ d_{mc}}$.Combining these inequalities and the triangular inequality, we have
\begin{align}
 \|B^*-\hat{B}\|_{\mr{F}} \leq 24\alpha^*r_0.
\end{align}
Therefore, in the remaining part of Section \ref{s:mc1}, we assume
\begin{align}
	\label{alpha1}
	\alpha(\Theta_\eta) \leq \frac{1}{12r_0},
\end{align}

\subsection{Proof of I}
Remember that \eqref{ine:det:xis0} is 
\begin{align}
	&\left| \frac{\lambda_o}{\sqrt{n}}\sum_{i=1}^n h\left(\frac{\xi_i}{\lambda_o \sqrt{n}} \right) \langle  X_i, \Theta_\eta \rangle \right| \leq a_{\mr{F}} r_{a,\mr{F}} \left\|\Theta_\eta\right\|_\mr{F} + a_*r_{a,*}\|\Theta_\eta\|_*
\end{align}
and \eqref{ine:det:upper0} is
\begin{align}
	&\left| \sum_{i=1}^n \frac{\lambda_o}{\sqrt{n}}u_i \langle X_i, \Theta_\eta\rangle\right|
	\leq b_{\mr{F}} r_{b,\mr{F}}\left\|\Theta_\eta\right\|_\mr{F} + b_*r_{b,*}\|\Theta_\eta\|_*.
\end{align}
First, from Lemma \ref{l:mcspec}, we confirm \eqref{ine:det:xis0} with 
\begin{align}
	a_\mr{F}=0,\quad a_* = C_{mc}
\end{align}
with probability at least $1-\delta$.
Second, from \eqref{alpha1}, we can confirm \eqref{ine:M-2}:
\begin{align}
	\|\Theta_\eta\|_\infty \leq \frac{1}{12r_0}\frac{1}{d_{mc}}\|\Theta_\eta\|_{\mr{F}} \leq \frac{1}{12\times\sqrt{\lambda_o \sqrt{n}\frac{o}{n}}}\frac{1}{d_{mc}}\|\Theta_\eta\|_{\mr{F}}.
\end{align}
From Lemma  \ref{al:|uMv|-mc}, we confirm \eqref{ine:det:upper0} with
\begin{align}
	b_\mr{F} = \frac{1}{6},\quad b_* =0.
\end{align}
Lastly, from union bound, we confirm \eqref{ine:det:xis0} and \eqref{ine:det:upper0} are satisfied with probability at least $1-\delta$.

\subsection{Proof of I\hspace{-.1em}I}
Remember we have
\begin{align}
		\lambda_*&= c_{mc1}\times  \frac{1}{ \sqrt{r}}\times\left(\sigma_\xi \sqrt{\frac{r (d_1+d_2) \{\log (d_1+d_2) + \log (1/\delta)\}}{n}}+\lambda_o\frac{d_{mc}(\log  d_{mc}+\log(1/\delta))}{\sqrt{n}}+\sqrt{\lambda_o \sqrt{n}\frac{o}{n}}\right)\nonumber\\
		&=c_{mc1}\times \left(r_{a,*}+\frac{r_{b,\mr{F}}}{\sqrt{r}}\right)
\end{align}
for a sufficiently large numerical constant $c_{mc1}$ and we have
\begin{align}
C_s =\frac{a_\mr{F}r_{a,\mr{F}}+\sqrt{2}b_\mr{F}r_{b,\mr{F}}}{c_\kappa \sqrt{r}} + (a_*r_{a,*}+\sqrt{2}b_*r_{b,*}) = \frac{\sqrt{2}r_{b,\mr{F}}}{12\sqrt{r}} + C_{mc1}r_{a,*} \leq \left(C_{mc1} + \frac{\sqrt{2}}{12}\right)\left(r_{a,*}+\frac{r_{b,\mr{F}}}{\sqrt{r}}\right),
\end{align}
where we use the fact that we can set $c_0=1, \kappa=1$ and $c_\kappa = 2$.
Consequently, we confirm $\lambda_*-C_s>0$.

\subsection{Proof of I\hspace{-.1em}I\hspace{-.1em}I}
Remember that \eqref{ine:det:lower} is
\begin{align}
	c_1 \| \Theta_{\eta}\|_\mr{F}^2-c_2r_{c,\mr{F}}\|\Theta_{\eta}\|_{\mr{F}} -c_3 r_{c}\leq \lambda_o^2\sum_{i=1}^n \left\{	-h \left(\frac{\xi_i-\langle X_i,\Theta_{\eta}\rangle}{\lambda_o\sqrt{n}} \right)+h  \left(\frac{\xi_i}{\lambda_o\sqrt{n}}  \right)\right\} \frac{\langle X_i,\Theta_{\eta}\rangle}{\lambda_o \sqrt{n}}
\end{align}
for matrix compressed sensing.
From \eqref{alpha1} and the constraint of the optimization problem, we confirm that \eqref{ine:mcs2} and \eqref{ine:constraint} and we have
\begin{align}
&\inf_{\Theta \in \mc{R}_{mc}}\left[\lambda_o^2\sum_{i=1}^n  \left\{-h\left(\frac{\xi_i - \langle  X_i, \Theta\rangle}{\lambda_o\sqrt{n}}\right)+h \left(\frac{\xi_i}{\lambda_o\sqrt{n}}\right) \right\}\frac{\langle X_i,\Theta\rangle}{\lambda_o\sqrt{n}} \right] \nonumber\\
&\geq \frac{2}{3}\|\Theta\|_\mr{F}^2 -C\left( \alpha^* \sqrt{r\frac{d_{mc}\{\log d_{mc}+\log(1/\delta)\}}{n}}+ \sqrt{r}\frac{d_{mc}\log  d_{mc}}{n}+ \alpha^* \left(\frac{o}{n}\right)^\frac{\alpha}{2(1+\alpha)}	\right)\|\Theta\|_\mr{F}-5\frac{\log(1/\delta)}{n}
\end{align} 
with probability at least $1-\delta$.
We can confirm \eqref{ine:det:lower} by Corollary \ref{c:mcsc1}
with
\begin{align}
	c_1 = \frac{2}{3},\quad c_2 =c_3 =C_{mc1}.
\end{align}

\subsection{Proof of I\hspace{-.1em}V}
From the definition of $\lambda_*,\,\lambda_o\sqrt{n}$ and the values of $a_{\mr{F}}, a_*, b_{\mr{F}}, b_*, c_1,c_2,c_3,r_{a,\mr{F}},r_{a,*},r_{a,\mr{F}},r_{b,*},r_{b,\mr{F}},r_{c,\mr{F}},r_c$, we have
\begin{align}
	&\frac{c_2r_{c,\mr{F}} +C_{\lambda_*} +\sqrt{c_1c_3r_c}}{c_1} \nonumber\\
	&\leq\frac{(a_\mr{F}+\sqrt{2}b_\mr{F})r_{a,\mr{F}} + (a_*+\sqrt{2}b_*)c_\kappa \sqrt{r} r_{a,*}+\lambda_*  \sqrt{r} +c_2r_{c,\mr{F}} +\sqrt{c_1c_3r_c}}{c_1}\nonumber \\
		&\leq C_{mc1}'\times(\sqrt{r}r_{a,*}+r_{b,\mr{F}}+r_{c,\mr{F}}+\sqrt{r_c}),
\end{align}
where $C_{mc1}'$ is some sufficiently large constant and we use the fact that $\sqrt{r_c}\leq r_{c,\mr{F}}$.
Remember that $r_0$
\begin{align}
	&r_0 = c_{mc1}'\times \nonumber\\
	&\left((\sigma_\xi+\alpha^*) \sqrt{\frac{r d_{mc} (\log d_{mc}+ \log (1/\delta))}{n}}+\lambda_o\frac{\sqrt{r}d_{mc}(\log  d_{mc}+\log(1/\delta))}{\sqrt{n}}+ \sqrt{r}\frac{d_{mc}\log d_{mc}}{n}+ \sqrt{\lambda_o \sqrt{n}\frac{o}{n}} + \alpha^*\left(\frac{o}{n}\right)^\frac{\alpha}{2(1+\alpha)}\right)\nonumber\\
	&= c_{mc1}'\times (\sqrt{r}r_{a,*}+r_{b,\mr{F}}+r_{c,\mr{F}}+\sqrt{r_c})
\end{align} 
for sufficiently large numerical constant $c'_{mc1}$, and we confirm \eqref{ine:det:condc}.

\section{Proof of Theorem \ref{t:mc:main2}}
\label{s:mc2}
The proof of Theorem \ref{t:mc:main2} is almost identical to that of  Theorem \ref{t:mc:main1}.
Therefore, we shall omit it.

\end{document}